%% file: main.tex
\documentclass{article}

\usepackage{preprint,times}

\iclrfinalcopy % Uncomment for camera-ready version, but NOT for submission

\usepackage[utf8]{inputenc} % allow utf-8 input
\usepackage[T1]{fontenc}    % use 8-bit T1 fonts
\usepackage{paper_style}
\usepackage{tabs/prambles}
\input{notation_macros}

\title{ENCORE: Exact Non-equilibrium COntrol with Replica Exchange for Diffusion Generation}

\author{
Jiahao Yu\textsuperscript{\textnormal{1}}\thanks{Work completed at the University of Cambridge.}\textnormal{,}\ \ 
Saifuddin Syed\textsuperscript{\textnormal{2}}\textnormal{,}\ 
Jos\'{e} Miguel Hern\'{a}ndez-Lobato\textsuperscript{\textnormal{3}}\textnormal{,}\ 
Jiajun He\textsuperscript{\textnormal{3}}\thanks{Correspondence to <\texttt{jh2383@cam.ac.uk}>.}
\\
\textsuperscript{\textnormal{1}}University of Oxford,
\textsuperscript{\textnormal{2}}University of British Columbia,
\textsuperscript{\textnormal{3}}University of Cambridge
}

\begin{document}

\maketitle

\begin{abstract}
Inference-time control steers a pretrained generative model towards a target distribution without retraining. We study tilted targets $\pi_0\propto G_0\,p_0$, where $p_0$ is the sampler output distribution and $G_0$ is an evaluable reweighting function. Existing approaches rely on sequential annealing with sequential Monte Carlo (SMC) or parallel annealing with replica exchange (RE). Sequential control is exact but needs large particle populations, whereas no exact parallel control method exists: existing RE corrections approximate an intractable time reversal and are biased. We propose \textit{\textbf{E}xact \textbf{N}on-equilibrium \textbf{CO}ntrol with \textbf{R}eplica \textbf{E}xchange} (ENCORE), the first exact parallel control method. Each replica stores its generation trajectory, so the upward move is a truncation and the intractable time reversal is never simulated. We prove target invariance and show that the resulting dynamics are those of non-equilibrium replica exchange with the exact time reversal as forward proposal. Under regularity conditions, our diffusion analysis shows that both sequential and parallel control become unstable under refinement of the time discretisation without guidance, whereas guided proposals remain stable and yield diagnostics for tuning the schedule and the computational budget. Across synthetic targets, Boltzmann sampling of biomolecules, and image generation, ENCORE achieves competitive accuracy and diversity, remains robust to sampler perturbations, and applies to distilled samplers where existing RE corrections are unavailable.
\end{abstract}

% ==================================================================================
\input{main/01_introduction}

\input{main/02_background}
\input{main/03_methods}
\input{main/05_experiments}
\input{main/06_conclusion}

% ==================================================================================
% \clearpage

\subsection*{AI use statement}

% (This section is \textbf{required} and does not count toward the page limit.)

In this work, we used generative AI tools for assist in the writing of proofs, design or provide feedback on research methodology or experiments, implement methods, and interpret results.
We have not used generative AI tools for help develop theoretical models or conceptual frameworks, formulate mathematical claims, provide critical ingredients for proving mathematical claims, propose or refine hypotheses, and clean and reformat dataset, and generate synthetic data sets, assist with translation, and support qualitative and thematic data analysis are not applicable to this work.
Additionally, we used generative AI tools for sourcing/searching for information.
We have reviewed all AI-assisted work. We checked all LLM-generated proofs, verified and tested LLM-generated code for correctness.
We take responsibility for the final content of this work,
including text, claims or artifacts produced with the aid of generative AI.

\subsection*{Reproducibility Statement}
We provide all theoretical results, assumptions, and proofs in the appendix or refer to established results where appropriate.
We also report the experimental settings and hyperparameters in \cref{app:experiments}.
To support reproducibility, we plan to release the code upon acceptance.

% See the ICLR 2027 AI Policy for Authors for more details. This statement should
% not be more than 1 page.

\subsubsection*{Acknowledgments}
SS acknowledges the support of a CANSSI CRT grant and an NSERC Discovery Grant. JH acknowledges support from the University of Cambridge Harding Distinguished Postgraduate Scholars Programme. JMHL acknowledges funding from the AI Hub in Generative Models under grant EP/Y028805/1.

Part of this work was performed using resources provided by the Cambridge Service for Data Driven Discovery (CSD3) operated by the University of Cambridge Research Computing Service (\url{www.csd3.cam.ac.uk}), provided by Dell EMC and Intel using Tier-2 funding from the Engineering and Physical Sciences Research Council (capital grant EP/T022159/1), and DiRAC funding from the Science and Technology Facilities Council (\url{www.dirac.ac.uk}).

\bibliographystyle{iclr2027_conference}
\bibliography{references}

% ==================================================================================
\clearpage
\appendix

\crefalias{section}{appendix}
\crefalias{subsection}{appendix}
\crefalias{subsubsection}{appendix}

\input{appendices/theory}
\input{appendices/experiments}

\end{document}

%% file: notation_macros.tex
\ifdefined\ENCORENotationLoaded
  \expandafter 
\fi
\newcommand{\ENCORENotationLoaded}{}

\DeclarePairedDelimiterX{\inner}[2]{\langle}{\rangle}{#1,#2}

\makeatletter
\newcommand{\dirharp@gap}{0ex}
\newcommand{\dirharp@scale}{0.8}
\DeclareRobustCommand{\dirharp}[2]{\mathpalette{\dirharp@{#1}}{#2}}
\newcommand{\dirharp@}[3]{%
  \sbox\z@{$\m@th#2#3$}%
  \ooalign{%
    \hfil$\m@th#2#3$\hfil\cr
    \hfil\raisebox{\dimexpr\ht\z@+\dirharp@gap\relax}{%
      \scalebox{\dirharp@scale}{$\m@th\scriptstyle#1$}}\hfil\cr}}
\makeatother
\newcommand{\Xsp}{\mathbb{X}}                 % state space
\newcommand{\Traj}[1]{\mathbb{X}_{#1}}        % partial-trajectory space to time t_{#1}
\newcommand{\bx}{\mathbf{x}}                  % partial trajectory
\newcommand{\bX}{\mathbf{X}}                  % full replica ladder
\newcommand{\by}{\mathbf{y}}                  % lower replica in an exchange
\newcommand{\bp}{\mathbf{p}}                  % sampler law on trajectories
\newcommand{\bpi}{\boldsymbol{\pi}}           % level law on trajectories
\newcommand{\prodpi}{\bpi_N\otimes\cdots\otimes\bpi_0}
\newcommand{\Lf}{Q'}                         % forward proposal Q'_{n|n-1}
\newcommand{\Kb}{P}                          % sampler kernel P_{n-1|n}
\newcommand{\Kf}{P}                          % exact time reversal P_{n|n-1} (same letter; subscript order gives direction)
\newcommand{\Kfh}{\widehat{P}}               % model's noising kernel
\newcommand{\Lb}{Q}                          % backward proposal Q_{n-1|n}
\newcommand{\gen}{\mathcal{A}}                % space-time generator
\newcommand{\E}{\mathbb{E}}
\newcommand{\Var}{\operatorname{Var}}
\newcommand{\TV}{\operatorname{TV}}

\newcommand{\erf}{\operatorname{erf}}
\newcommand{\erfc}{\operatorname{erfc}}
\newcommand{\dd}{\mathrm{d}}

%% file: main/01_introduction.tex
\section{Introduction}

Diffusion models \citep{ho_denoising_2020,song_scorebased_2021} have achieved strong performance in image \citep{dhariwal_diffusion_2021}, video \citep{ho_video_2022}, and molecular generation \citep{hoogeboom_equivariant_2022}.
% They define a forward process that progressively corrupts data with noise and generate samples by simulating learned reverse-time dynamics from noise to data.
However, downstream tasks often require samples to align with preferences, observations, or constraints.
\textit{Inference-time control} steers generation without retraining the model.
Many such tasks, including reward tilting \citep{uehara_rewardguided_2025} and posterior sampling \citep{wu_practical_2023,chung_diffusion_2023}, can be formulated as endpoint reweighting,
\begin{equation}
    \pi_0(\dd x_0)\propto G_0(x_0)\,p_0(\dd x_0),
    \label{eq:intro-target}
\end{equation}
where $p_0$ is the sampler's output distribution and $G_0$ is a non-negative, evaluable reweighting function.
For instance, it steers image generation towards human preferences \citep{xu_imagereward_2023,singhal_general_2025}, while in biomolecular applications it supports sampling at different temperatures \citep{skreta_feynmankac_2025,he_crepe_2026} and exploring rare conformations for free-energy estimation \citep{xie_enhanced_2026}.

Guidance \citep{dhariwal_diffusion_2021,ho_classifierfree_2022} steers generation by modifying the reverse-time score or drift.
The exact correction is the Doob transform \citep{delmoral_feynmankac_2004}, but it depends on an intractable expectation over the remaining generation.
In practice, guidance methods rely on approximations and therefore need not preserve the intended endpoint distribution \citep{bradley_classifierfree_2024,dandapanthula_are_2026}.

Monte Carlo methods provide a complementary approach by targeting $\pi_0$ directly.
Sequential Monte Carlo (SMC). \citep{delmoral_feynmankac_2004,delmoral_sequential_2006} propagates, reweights, and resamples a particle population through a sequence of intermediate distributions (sequential annealing)
Nevertheless, SMC can suffer from weight and ancestral degeneracy, particularly with limited particle populations.
% More gradual annealing can help control weight degeneracy \citep{syed_optimized_2026}, while larger populations can help mitigate the loss of ancestral diversity caused by resampling \citep{koskela_asymptotic_2020}.

Replica exchange (RE), or parallel tempering (PT) \citep{swendsen_replica_1986,geyer_markov_1991,hukushima_exchange_1996}, instead maintains replicas across a ladder of intermediate distributions and exchanges neighbouring replicas through Metropolis--Hastings moves (parallel annealing).
Whereas SMC propagates particles through the ladder once, RE evolves a persistent set of replicas over successive MCMC iterations \citep{syed_optimized_2026}.
This gives an \textit{online} advantage over SMC: sampling continues without restarting.
% Under compatible target and proposal constructions, both methods use the same incremental weights: SMC reweights particles with them, while RE uses their ratios in the exchange correction.

However, continued sampling recovers the intended target only if the chain preserves it.
TDS \citep{wu_practical_2023} and FKD \citep{singhal_general_2025} compute exact SMC weights for their chosen potentials without evaluating sampler marginals.
In contrast, CREPE \citep{he_crepe_2026} approximates intermediate marginal-density ratios using denoising and noising transition densities.
This substitution relies on an exact time-reversal relation that may be violated by model or discretisation error.
Any resulting bias persists with longer runs. 
CREPE also requires a tractable paired noising kernel, which may be unavailable for, e.g., distilled few-step samplers. No exact parallel control method exists.

We address these limitations with \emph{\textbf{E}xact \textbf{N}on-equilibrium \textbf{CO}ntrol with \textbf{R}eplica \textbf{E}xchange} (ENCORE).
The key idea is to lift each replica from its current state to a partial generation trajectory.
An upward exchange then truncates the stored trajectory rather than simulating a noising transition, allowing common sampler-path terms to cancel exactly in the Metropolis--Hastings ratio.
The resulting replica movement across levels and target-level chains are illustrated in \cref{fig:replica-trajectory}. 
See \cref{fig:gmm-endpoints-fig1} for an illustration of how different inference-time control methods affect the endpoint distribution.

We show that ENCORE preserves the endpoint target without evaluating intermediate marginals or requiring a paired noising kernel.
We show that ENCORE has the rejection rate of non-equilibrium replica exchange with the exact time reversal as forward proposal, which vanishes precisely under the Doob relations.
For diffusion samplers, we show under regularity conditions that refining the time discretisation can cause accumulated rejection to diverge without guidance, whereas guided proposals admit a finite \emph{global communication barrier} $\Lambda$, which measures the deviation of the guidance from the Doob transform and sets the number of levels needed.
Experiments on synthetic reward tilting, molecular sampling, and few-step image generation demonstrate competitive accuracy and diversity, robustness to sampler perturbations, and applicability where existing RE corrections are unavailable.

\begin{figure}[t]
\centering
\includegraphics[width=\textwidth]{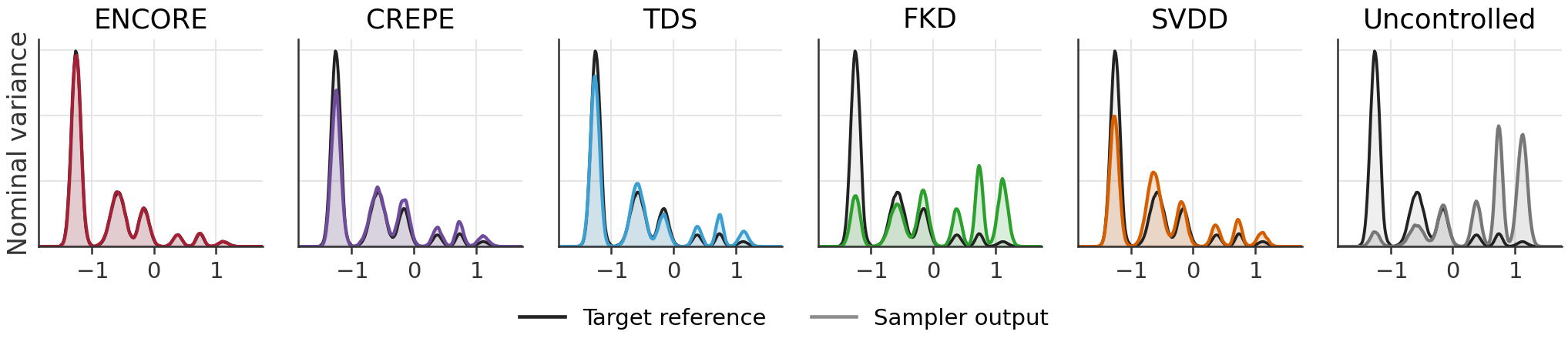}
\caption{One-dimensional output distributions at nominal transition variance with 10,000 samples.}
\label{fig:gmm-endpoints-fig1}
\end{figure}

%% file: main/02_background.tex
\section{Background}
\label{sec:background}

\subsection{Diffusion Samplers as Markov Chains}
\label{sec:bg-diffusion}
 
We consider samplers defined by a tractable reference distribution $\eta$ on a state space $\Xsp$ and a sequence of Markov kernels $P_{n-1|n}$ for $n=N,\ldots,1$.
The sampler draws $x_N\sim\eta$ and then $x_{n-1}\sim P_{n-1|n}(\dd x_{n-1}\mid x_n)$ for $n=N,\ldots,1$, inducing the marginals $p_n$ of $x_n$,
\begin{equation}
p_N=\eta,\qquad p_{n-1}=p_nP_{n-1|n},
\end{equation}
with $p_0$ the output distribution of the sampler.\footnote{For a law $p$, a kernel $K$ and a function $f$, $pK(\dd y):=\int p(\dd x)K(\dd y\mid x)$, $Kf(x):=\int K(\dd y\mid x)f(y)$, and $p\otimes K(\dd x,\dd y):=p(\dd x)K(\dd y\mid x)$ is the joint law of a draw from $p$ propagated by $K$.}
Diffusion models are the motivating example \citep{ho_denoising_2020,song_scorebased_2021}: $P_{n-1|n}$ is one or more steps of a learned denoiser, and $p_0$ approximates the data distribution.
We assume only that $P_{n-1|n}$ can be simulated and that its transition density can be evaluated.
In particular, the marginals $p_n$ for $n<N$ are intractable, since they require integrating over $P_{n-1|n}$, and the reverse kernels $P_{n|n-1}$, for which $p_{n-1}P_{n|n-1}=p_n$, are unavailable.

\begin{figure}[t]
\centering
\includegraphics[width=\linewidth]{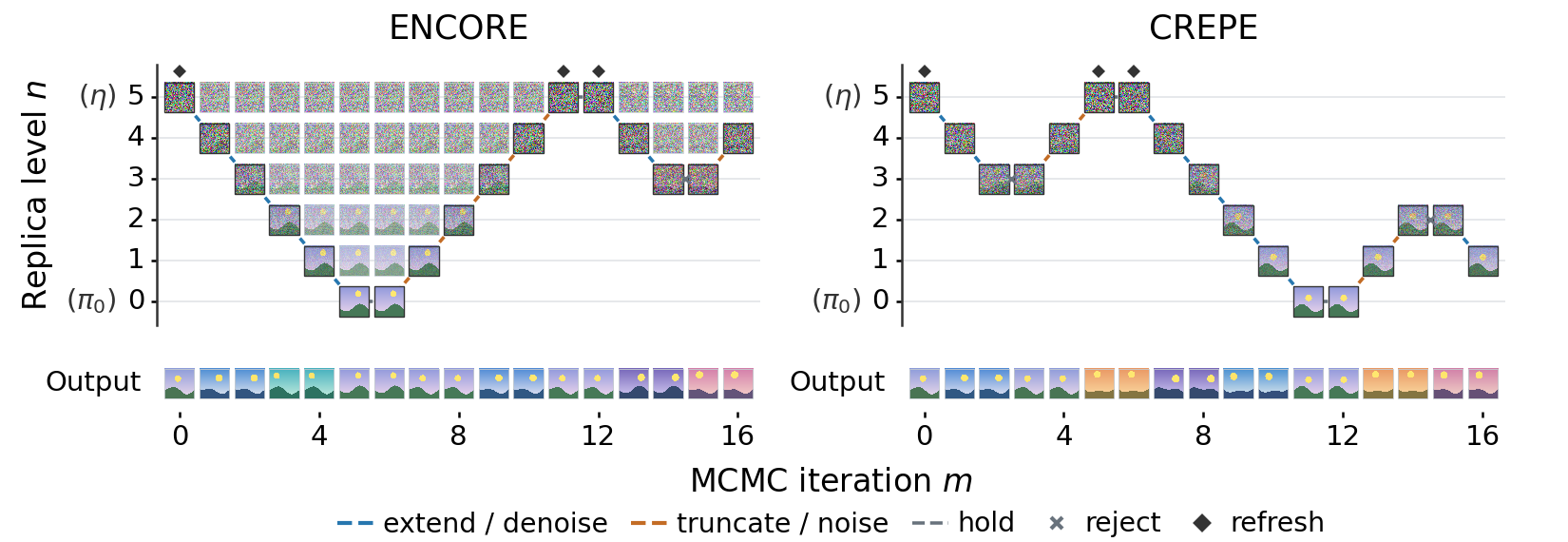}
\caption{Replica movement (top) and output chains (bottom) for ENCORE and CREPE.}
\label{fig:replica-trajectory}
\label{fig:replica-trajectory-encore}
\label{fig:replica-trajectory-crepe}
\end{figure}

\subsection{Inference-Time Control}
\label{sec:bg-control}

Given an evaluable, non-negative function $G_0:\Xsp\to[0,\infty)$ inference-time control asks for samples from the tilted target without retraining the model,
\begin{equation}
    \pi_0(\dd x_0)\propto G_0(x_0)\,p_0(\dd x_0).
    \label{eq:endpoint-target}
\end{equation}
For example, reward tilting uses $G_0(x_0)=\exp\{\kappa r(x_0)\}$ for a reward $r$ and strength $\kappa\ge0$ \citep{singhal_general_2025}, while posterior sampling uses the likelihood $G_0(x_0)=p(y\mid x_0)$ of an observation $y$, with $p_0$ as the prior \citep{wu_practical_2023,chung_diffusion_2023}.

The natural approach keeps the reference $\eta$, replaces each sampler kernel with a proposal $\Lb_{n-1|n}(\dd x_{n-1}\mid x_n)$ that steers the chain towards the target, and returns $x_0$ as an approximate sample from $\pi_0$.
General proposals introduce bias, and an exact construction follows from the Doob $h$-transform of the sampler by $G_0$ \citep{delmoral_feynmankac_2004}.
Define the Doob potentials recursively and the corresponding proposals by
\begin{equation}
    G^*_n:=\Kb_{n-1|n}G^*_{n-1},
    \qquad
    \Lb^*_{n-1|n}(\dd x_{n-1}\mid x_n)
    :=
    \frac{G^*_{n-1}(x_{n-1})}{G^*_n(x_n)}
    \Kb_{n-1|n}(\dd x_{n-1}\mid x_n),
    \label{eq:doob}
\end{equation}
with $G^*_0:=G_0$, so that $G^*_n(x_n)=\E[G_0(X_0)\mid X_n=x_n]$ under the sampler law.
Running $\Lb^*$ from the tilted reference $\eta^*\propto G_N^*\,\eta$ yields $x_0\sim\pi_0$.\footnote{In diffusion models, $x_N$ is typically nearly independent of $x_0$, so $G_N^*$ is approximately constant and $\eta^*\approx\eta$.}
The potentials and the twisted kernels are intractable, and the tractable approximations used in practice need not preserve $\pi_0$ (\cref{sec:related}). The annealed Monte Carlo methods below correct them.

\subsection{Annealed Monte Carlo}
\label{sec:bg-mc}
 
Consider a ladder $\pi_N,\ldots,\pi_0$ on $\Xsp$ from a tractable reference $\pi_N$ to a target $\pi_0$. Annealed Monte Carlo generates, for every level $n$, particles $x^m_n$, $m=1,\ldots,M$, whose empirical law approximates $\pi_n$. The two constructions below differ in the order in which the array $(x^m_n)$ is filled. Sequential annealing propagates $M$ particles level by level, so $m$ indexes particles, and parallel annealing evolves one replica per level over $M$ iterations, so $m$ indexes iterations. Both move a particle from level $n$ to $n-1$ with a backward proposal $\Lb_{n-1|n}(\dd x_{n-1}\mid x_n)$, describe the reverse move with a forward proposal $\Lf_{n|n-1}(\dd x_n\mid x_{n-1})$, and correct it with the incremental weight, the Radon--Nikodym derivative of the two joint laws of $(x_{n-1},x_n)$ the proposals induce,
\begin{equation}
    w_{n-1,n}
    :=\frac{\dd(\pi_{n-1}\otimes\Lf_{n|n-1})}{\dd(\pi_n\otimes\Lb_{n-1|n})}
    \ \propto\ \frac{\pi_{n-1}(x_{n-1})\,\Lf_{n|n-1}(x_n\mid x_{n-1})}{\pi_n(x_n)\,\Lb_{n-1|n}(x_{n-1}\mid x_n)}.
    \label{eq:generic-weight}
\end{equation}
Both produce consistent estimators as $M\to\infty$ and comparable accuracy for comparable $N$ and $M$, so they are \emph{computationally dual} \citep{syed_optimized_2026}. \Cref{app:algorithms} gives the details.
 
\textbf{Parallel annealing} evolves $N$ replicas, one per level, in parallel over $M$ iterations, in $O(Nd)$ memory, and records the level-$0$ replica at every iteration. The general framework is non-equilibrium replica exchange \citep{ballard_replica_2009}, also known as accelerated parallel tempering \citep{zhang_accelerated_2026}, which targets $\pi_N\otimes\cdots\otimes\pi_0$ by MCMC and, in place of the state swaps of parallel tempering \citep{swendsen_replica_1986,geyer_markov_1991,hukushima_exchange_1996}, proposes transitions between levels with non-reversible communication \citep{syed_nonreversible_2022}. Here $x^m_n$ is the state at level $n$ after iteration $m$. At iteration $m$ the reference is redrawn, $x^m_N\sim\eta$, and for every $n\equiv m\pmod 2$ the pair $(x^{m-1}_{n-1},x^{m-1}_n)$ proposes $y^m_{n-1}\sim\Lb_{n-1|n}(\cdot\mid x^{m-1}_n)$, $y^m_n\sim\Lf_{n|n-1}(\cdot\mid x^{m-1}_{n-1})$ and sets $(x^m_{n-1},x^m_n)$ to $(y^m_{n-1},y^m_n)$ if accepted and to $(x^{m-1}_{n-1},x^{m-1}_n)$ if rejected, with acceptance probability $\alpha^m_{n-1,n}:=\alpha_{n-1,n}(x^{m-1}_{n-1},x^{m-1}_n;y^m_{n-1},y^m_n)$ for
\begin{equation}
    \alpha_{n-1,n}(x_{n-1},x_n;y_{n-1},y_n)=\min\bigg\{1,\ \frac{w_{n-1,n}(y_{n-1},x_n)}{w_{n-1,n}(x_{n-1},y_n)}\bigg\}.
    \label{eq:apt}
\end{equation}
The remaining levels keep $x^m_n=x^{m-1}_n$. Optionally, a $\pi_n$-invariant MCMC move is then applied to each $x^m_n$, in parallel over $n$.
 
Particle diversity is measured by the restart rate $\tau$, the rate at which reference draws reach the target \citep{katzgraber_feedbackoptimized_2006}. Under the efficient local exploration assumption of \citet{syed_nonreversible_2022}, that the log-weights of consecutive exchange attempts at a level are independent,
\begin{equation}
    \tau=\bigg(2+2\sum_{n=1}^{N}\frac{r_{n-1,n}}{1-r_{n-1,n}}\bigg)^{-1}\ \le\ \frac12,
    \label{eq:roundtrip}
\end{equation}
where $r_{n-1,n}:=1-\E[\alpha_{n-1,n}]$ is the stationary rejection rate of the exchange at $(n-1,n)$. By \citet[Thm. 1]{zhang_accelerated_2026},
\begin{equation}
    r_{n-1,n}=r\big(\pi_{n-1}\otimes\Lf_{n|n-1},\ \pi_n\otimes\Lb_{n-1|n}\big),\qquad r(p,q):=\TV(p\otimes q,\,q\otimes p).
    \label{eq:apt-rejection}
\end{equation}
When the $\pi_n$ discretise a continuum of annealing distributions and $N\to\infty$ with the rejection rates vanishing, $\sum_nr_{n-1,n}\to\Lambda$, the global communication barrier, and $\tau\to(2+2\Lambda)^{-1}$. An equi-rejection schedule attains stable communication with $N\approx2\Lambda$ levels, beyond which additional levels do not increase $\tau$ and $M$ should be as large as affordable \citep{syed_nonreversible_2022}.
 
\textbf{Sequential annealing} evolves $M$ particles in parallel over the $N$ levels, in $O(Md)$ memory, and terminates with a weighted batch. The general framework is sequential Monte Carlo samplers \citep{delmoral_sequential_2006}. Particle $m$ at level $n$ carries a weight $w^m_n$, the product of the incremental weights along its path. Initialise $x^m_N\sim\pi_N$ with $w^m_N=1$. At each level move each particle down, $x^m_{n-1}\sim\Lb_{n-1|n}(\dd x_{n-1}\mid x^m_n)$, update its weight, $w^m_{n-1}=w^m_n\,w_{n-1,n}(x^m_{n-1},x^m_n)$, and resample the population when the weights degenerate, optionally with a $\pi_{n-1}$-invariant mutation, so that at level $0$ the weighted population targets $\pi_0$. Particle diversity is measured by the effective sample size of the weights, and the analogous barrier governs the number of levels needed for stability \citep{syed_optimized_2026}. SMC never simulates $\Lf$ and only evaluates \cref{eq:generic-weight}. Sequential annealing is preferable when wall-clock time is the binding constraint and parallel annealing when memory is.

%% file: main/03_methods.tex
% Requires the macros in notation_macros.tex.
\providecommand{\trunc}{\mathcal T}

\section{Methods}
\label{sec:method}

\subsection{Sequential versus Parallel Control}
\label{sec:exact}

\begin{table}[t]
\centering
\caption{Monte Carlo corrections for inference-time control with $M$ samples and $N$ levels. $^*$RNE is exact when taking the same construction as TDS \citep[App.~C.5]{he_rne_2026}.}
\label{tab:landscape}
\footnotesize
\setlength{\tabcolsep}{6pt}
\begin{tabular}{@{}lcccccc@{}}
\toprule
Method & Annealing & Exact & Parallel over & Moves/iter. & \# Iter. & Memory \\
\midrule
FKD \citep{singhal_general_2025} & sequential & yes & $M$ particles & $M$ & $N$ & $O(Md)$ \\
TDS \citep{wu_practical_2023} & sequential & yes & $M$ particles & $M$ & $N$ & $O(Md)$ \\
RNE \citep{he_rne_2026} & sequential & no / yes$^*$ & $M$ particles & $M$ & $N$ & $O(Md)$ \\
\midrule
CREPE \citep{he_crepe_2026} & parallel & no & $N$ replicas & $N$ & $M$ & $O(Nd)$ \\
\textbf{ENCORE} (ours) & parallel & yes & $N$ replicas & $N/2$ & $M$ & $O(N^2d)$ \\
\bottomrule
\end{tabular}
\end{table}

Our goal is consistent estimation under the target $\pi_0$, given the denoising kernels $\Kb_{n-1|n}$ and potentials $G_0,\ldots,G_N$ with $G_N\equiv1$, where $G_1,\ldots,G_{N-1}$ are supplied by a guidance method. We run annealed Monte Carlo on the ladder $\pi_n(\dd x_n)\propto G_n(x_n)\,p_n(\dd x_n)$, $n=N,\ldots,0$, which interpolates from the reference $\pi_N=\eta$ to the target $\pi_0$. Both samplers of \cref{sec:bg-mc} require backward and forward kernels $\Lb_{n-1|n}$, $\Lf_{n|n-1}$ for which the incremental weight \cref{eq:generic-weight} is evaluable. The ratio $\dd\pi_{n-1}/\dd\pi_n$ contains the marginal ratio $\dd p_{n-1}/\dd p_n$, which is unavailable. We write it through the time reversal of $P_{n-1|n}$, the kernel $\Kf_{n|n-1}$ defined by $p_{n-1}\otimes\Kf_{n|n-1}=p_n\otimes\Kb_{n-1|n}$, formally
\begin{equation}
    \Kf_{n|n-1}(\dd x_n\mid x_{n-1}):=\frac{p_n(\dd x_n)}{p_{n-1}(\dd x_{n-1})}\,\Kb_{n-1|n}(\dd x_{n-1}\mid x_n).
    \label{eq:time-reversal}
\end{equation}
The weight then splits into a backward and a forward factor, $w_{n-1,n}=w_{n-1|n}\,w_{n|n-1}$ with
\begin{equation}
    w_{n-1|n}:=\frac{G_{n-1}(x_{n-1})}{G_n(x_n)}\,\frac{\Kb_{n-1|n}(x_{n-1}\mid x_n)}{\Lb_{n-1|n}(x_{n-1}\mid x_n)},
    \qquad
    w_{n|n-1}:=\frac{\Lf_{n|n-1}(x_n\mid x_{n-1})}{\Kf_{n|n-1}(x_n\mid x_{n-1})},
    \label{eq:weight-factor}
\end{equation}
both functions of $(x_{n-1},x_n)$. We assume throughout that $\Kb_{n-1|n}(\cdot\mid x_n)\ll\Lb_{n-1|n}(\cdot\mid x_n)$, so that the backward weight $w_{n-1|n}$ exists and is evaluable. The forward weight $w_{n|n-1}$ involves $\Kf_{n|n-1}$, which depends on $p_n$ and can be neither evaluated nor simulated. We call a correction \emph{exact} if it evaluates $w_{n-1,n}$ without asymptotic bias, and the proposal \emph{unguided} if $\Lb_{n-1|n}=\Kb_{n-1|n}$, in which case $w_{n-1|n}=G_{n-1}/G_n$ involves no model density, and \emph{guided} otherwise.

\paragraph{Exact control.}

In practice, we want to preserve the intended target induced by the sampler.
This is straightforward for \textit{sequential control}: SMC never simulates \(\Lf_{n|n-1}\) or evaluates the corresponding forward weight, so we may take \(\Lf_{n|n-1}=\Kf_{n|n-1}\), giving \(w_{n|n-1}=1\) and \(w_{n-1,n}=w_{n-1|n}\).
This gives the exact weights used by TDS \citep{wu_practical_2023} with a guided proposal and FKD \citep{singhal_general_2025} with an unguided one.

However, for \textit{parallel control}, exactness remains challenging.
Replica exchange must simulate \(\Lf_{n|n-1}\), while \(\Kf_{n|n-1}\) is generally unavailable for simulation.
CREPE \citep{he_crepe_2026} substitutes the model's noising kernel, \(\Lf_{n|n-1}=\Kfh_{n|n-1}\), and sets \(w_{n|n-1}\) to one, as in the Radon--Nikodym estimator of \citet{he_rne_2026}.
This is exact only if the sampler reverses the noising process exactly. Otherwise, the error accumulates over the model steps spanned by an exchange (\cref{prop:crepe}).
It is also undefined for samplers without a noising kernel.
The bias can be reduced by increasing \(N\), at the expense of the computational efficiency of parallel annealing.
No exact parallel control method currently exists.
We obtain one by lifting the replicas to path space: rather than approximating \(\Kf_{n|n-1}\), each replica stores its denoising trajectory, so the forward move need not be simulated.

\subsection{Replica Exchange on the Trajectory Ladder}
\label{sec:ladder}
The marginals $p_n$ are intractable, but the law of the denoising path is not: the partial trajectory $\bx_n=(x_N,\ldots,x_n)\in\Traj{n}:=\Xsp^{N-n+1}$ generated down to level $n$ has the product law
\begin{equation}
    \bp_N=p_N,\qquad
    \bp_{n-1}=\bp_n\otimes\Kb_{n-1|n},
    \label{eq:path-law}
\end{equation}
whose $x_n$-marginal is $p_n$. If the replica at level $n$ stores the trajectory $\bx_n$ that produced its state, the forward move is a truncation rather than a simulation: under the level law the truncated state is a draw of $\Kf_{n|n-1}$ given the lower state, so $w_{n|n-1}=1$ with nothing simulated. Bold symbols denote trajectories and objects defined on them, and the subscript of a trajectory is the level of its last state. Let $\trunc_{n|n-1}:\Traj{n-1}\to\Traj{n}$ drop the last state, so that $\bx_{n-1}=(\trunc_{n|n-1}(\bx_{n-1}),x_{n-1})$. Level $n$ holds $\bx_n\in\Traj{n}$ and the level distributions are
\begin{equation}
    \bpi_n(\dd\bx_n)\propto G_n(x_n)\,\bp_n(\dd\bx_n),
    \label{eq:trajectory-target}
\end{equation}
so the replicas jointly target $\prodpi$. The $x_n$-marginal of $\bpi_n$ is $\pi_n$: it is the ladder of \cref{sec:exact} lifted to trajectories, with $\bpi_N=\eta$ and the target $\pi_0$ as the $x_0$-marginal of level $0$. ENCORE (\cref{alg:encore}) is the parallel annealing sampler of \cref{sec:bg-mc} on this ladder with the following path exchange.

\paragraph{Path exchange.} We mirror \cref{eq:apt}. Let $\bx_n$ and $\bx_{n-1}=(\by_n,x_{n-1})$ be the replicas at levels $n$ and $n-1$, where $\by_n:=\trunc_{n|n-1}(\bx_{n-1})$ is the lower replica truncated to level $n$. Propose $y_{n-1}\sim\Lb_{n-1|n}(\dd y_{n-1}\mid x_n)$, set $\by_{n-1}:=(\bx_n,y_{n-1})$, and exchange $(\bx_{n-1},\bx_n)\leftarrow(\by_{n-1},\by_n)$: level $n-1$ receives the upper replica extended by one proposed state, and level $n$ receives the truncation $\by_n$ of the lower replica, with nothing simulated. On the common space $\Traj{n-1}$ the proposal $\by_{n-1}$ has law $\bpi_n\otimes\Lb_{n-1|n}$ and the lower replica has law $\bpi_{n-1}$, and by \cref{eq:path-law} the ratio between $\bpi_{n-1}$ and $\bpi_n\otimes\Lb_{n-1|n}$ is $w_{n-1|n}(y_{n-1},x_n)$ up to a constant: every factor of $\bp_n$ cancels, no marginal appears, and the weight depends on the trajectory only through its last state. The exchange is therefore \cref{eq:apt} with the forward proposal replaced by memory, and it has the same acceptance probability with the evaluable factor $w_{n-1|n}$ alone,
\begin{equation}
    \boldsymbol{\alpha}_{n-1,n}(\bx_{n-1},\bx_n;\,\by_{n-1},\by_n)=\min\bigg\{1,\ \frac{w_{n-1|n}(y_{n-1},x_n)}{w_{n-1|n}(x_{n-1},y_n)}\bigg\}.
    \label{eq:guided-acceptance}
\end{equation}
This is a reversible-jump exchange \citep{green_reversible_1995} whose reverse re-extends $\by_n$ by $x_{n-1}\sim\Lb_{n-1|n}(\dd x_{n-1}\mid y_n)$. \Cref{eq:guided-acceptance} is the acceptance probability \cref{eq:apt} with the exact time reversal \cref{eq:time-reversal} as forward proposal, $\Lf_{n|n-1}=\Kf_{n|n-1}$. \Cref{thm:exact} shows that ENCORE is exact for any potentials and proposals, that the Doob potentials and proposals are the optimal choice, and that departures from them cost rejections rather than bias.

State-level replica exchange simulates both $Q_{n-1|n}$ and $Q'_{n|n-1}$ at every iteration, whereas ENCORE simulates only $Q_{n-1|n}$, since moving up is a truncation of the stored trajectory. Each iteration therefore requires only $\lfloor N/2 \rfloor$ simulations from $Q_{n-1|n}$, performed in parallel, one for each active pair (\cref{alg:encore}). The trade-off is that we must store the entire trajectory of $N-n+1$ states for the replica at level $n$, which requires $O(N^2 d)$ memory.

\SetKwInput{KwIn}{Input}
\SetKw{Record}{Record}
\SetKwFor{ForAll}{for all}{do}{end}
\newcommand{\algcommentstyle}[1]{{\normalfont\itshape\color{teal}#1}}
\SetCommentSty{algcommentstyle}

\begin{figure}[t]
\centering

\begin{minipage}[t]{0.3\linewidth}
\vspace{0pt}
\centering
\includegraphics[width=0.84\linewidth]{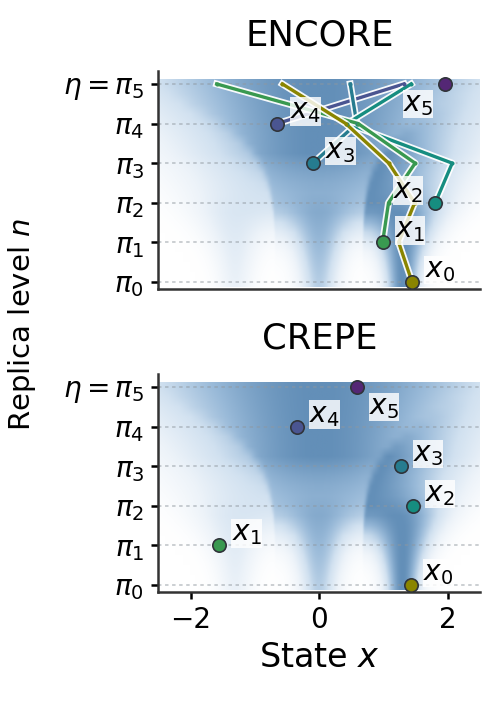}
\caption{Stored paths (ENCORE) vs. states (CREPE). Blue shading: level densities.}
\label{fig:annealing-path}
\label{fig:annealing-path-encore}
\label{fig:annealing-path-crepe}
\end{minipage}
\hfill
\begin{minipage}[t]{0.66\linewidth}
\vspace{0pt}
\SetAlFnt{\scriptsize\fontsize{9}{\baselineskip}\selectfont}
\begin{algorithm}[H]
\SetAlgoNlRelativeSize{-1}
\caption{ENCORE sampler}
\label{alg:encore}
\KwIn{reference $\eta$; sampler $(\Kb_{n-1|n})$; potentials $(G_n)$; proposals $(\Lb_{n-1|n})$; initial replicas $\bx^0_n\in\Traj{n}$; iterations $M$.}
\For{$m=1,\ldots,M$}{
    Draw $\bx^m_N=x^m_N\sim\eta$
    \tcp*[r]{refresh reference replica}
    \ForAll{$n\in\{1,\ldots,N\}$ with $n\equiv m\pmod 2$, in parallel}{
        Let $\bx_n=\bx^{m-1}_n$, $\bx_{n-1}=\bx^{m-1}_{n-1}$ and $\by_n=\trunc_{n|n-1}(\bx_{n-1})$
        \tcp*[r]{truncate lower replica}
        Draw $y_{n-1}\sim\Lb_{n-1|n}(\dd y_{n-1}\mid x_n)$ and set
        $\by_{n-1}\gets(\bx_n,y_{n-1})$
        \tcp*[r]{extend upper replica}
        With probability
        $\boldsymbol{\alpha}^m_{n-1,n}:=\boldsymbol{\alpha}_{n-1,n}(\bx_{n-1},\bx_n;\by_{n-1},\by_n)$
        of \cref{eq:guided-acceptance}, set
        $(\bx^m_{n-1},\bx^m_n)\gets(\by_{n-1},\by_n)$
        \tcp*[r]{exchange}
    }
    Set $\bx^m_n\gets\bx^{m-1}_n$ for every level not exchanged\;
        Optionally apply a $\pi_n$-invariant MCMC move to each $x^m_n$, in parallel over $n$
    \tcp*[r]{local move}
    \Record the last state $x^m_0$ of $\bx^m_0$
    \tcp*[r]{target sample}
}
\end{algorithm}
\end{minipage}
\end{figure}

\begin{theorem}
\label{thm:exact}
\begin{enumerate}
    \item[(a)] The ENCORE transition leaves $\prodpi$ invariant and the recorded state has stationary law $\pi_0$. The stationary rejection rate for the $n$-th exchange is
    \begin{equation}
        r_{n-1,n}:=1-\E[\boldsymbol{\alpha}_{n-1,n}]=r\big(\pi_{n-1}\otimes\Kf_{n|n-1},\ \pi_n\otimes\Lb_{n-1|n}\big).
        \label{eq:rejection-rate}
    \end{equation}
    \item[(b)] If $w_{n-1|n}$ is uniformly bounded above and below for each $n$, the chain is uniformly ergodic.
    \item[(c)] The restart rate \cref{eq:roundtrip} attains $\tau=1/2$ if and only if $G_n$ and $\Lb_{n-1|n}$ satisfy the Doob relations \cref{eq:doob} for every $n$.
\end{enumerate}
\end{theorem}
See \cref{app:derivations} for proofs. Where the potential is constant and $\Lb_{n-1|n}=\Kb_{n-1|n}$, every exchange is accepted whatever the spacing of the levels, so a coarse ladder is tolerated in the high-noise region.

\subsection{Diffusion Models}
\label{sec:diffusion}
\label{sec:design}

We now specialise the sampler and proposal kernels to a diffusion and study the limit of increasing parallelism, $N\to\infty$. Let $1=t_N>\cdots>t_0=0$ with steps $\Delta t_n:=t_n-t_{n-1}$, and let the potentials discretise a smooth path $t\mapsto G_t$ with $G_1\equiv1$ and $G_0$ the reweighting function, so that $\pi_t\propto G_tp_t$, $G_n=G_{t_n}$ and $\pi_n=\pi_{t_n}$. The proposal $\Lb_{n-1|n}$ takes one Euler--Maruyama step from $X_t$ to $t'=t-\Delta t$,
\begin{equation}
    Y_{t'}=X_t+\big(b_t(X_t)+\mathbf 1_{\mathrm{guided}}\,\sigma_t\sigma_t^{\top}\nabla\log G_t(X_t)\big)\,\Delta t+\sigma_t\sqrt{\Delta t}\,\xi,
    \qquad \xi\sim N(0,I),
    \label{eq:euler-step}
\end{equation}
which for $\mathbf 1_{\mathrm{guided}}=0$ is the denoiser step $\Kb_{n-1|n}$, the Euler--Maruyama step of the generative SDE $\dd X_t=b_t(X_t)\,\dd t+\sigma_t\,\dd B_t$ run from $t=1$ to $t=0$, and for $\mathbf 1_{\mathrm{guided}}=1$ the guided denoiser. Composing several model steps per level changes only the remainders below.

Consider the exchange between $t=t_n$ and $t'=t_{n-1}$. Given $X_t$ and $\xi$, the step \cref{eq:euler-step} determines $Y_{t'}$, and so the weight $w_{n-1|n}(Y_{t'},X_t)$ is a function $w_{t'|t}(X_t,\xi)$ of the state and the noise,
\begin{equation}
    \log w_{t'|t}(X_t,\xi)
    :=\underbrace{\log\frac{G_{t'}(Y_{t'})}{G_t(X_t)}}_{\text{potential}}
    \;+\;
    \mathbf 1_{\mathrm{guided}}\underbrace{\log\frac{\Kb_{t'|t}(Y_{t'}\mid X_t)}{\Lb_{t'|t}(Y_{t'}\mid X_t)}}_{\text{proposal correction}},
    \label{eq:logweight}
\end{equation}
and the rejection rate \cref{eq:rejection-rate} is approximately $\tfrac12\E|\ell_n-\ell'_n|$ for two independent copies $\ell_n,\ell'_n$ of this log-weight with $X_t\sim\pi_t$. Expanding along the step,
\begin{equation}
    \log w_{t'|t}(X_t,\xi)=
    \begin{cases}
        \sqrt{\Delta t}\,g_t^{\top}\xi+O(\Delta t) & \text{unguided},\\[2pt]
        \Delta t\,\Psi_t(X_t,\xi)+O(\Delta t^{3/2}) & \text{guided},
    \end{cases}
    \label{eq:guided-logweight}
\end{equation}
where $g_t:=\sigma_t^{\top}\nabla\log G_t$ and $\Psi_t(X_t,\xi)=-\partial_{t'}\log w_{t'|t}(X_t,\xi)\big|_{t'= t}$ is the pathwise derivative of the guided log-weight with $\xi$ held fixed. Explicitly, $\Psi_t$ is the rate of change of $\log G_t$ along the sampler plus a mean-zero curvature fluctuation,
\begin{equation}
    \Psi_t(x,\xi)=\frac{\gen_tG_t(x)}{G_t(x)}+\tfrac12\big(\xi^{\top}M_t(x)\xi-\operatorname{tr}M_t(x)\big),
    \label{eq:Psi-main}
\end{equation}
where $\gen_t:=-\partial_t+b_t\cdot\nabla+\tfrac12\operatorname{tr}(\sigma_t\sigma_t^{\top}\nabla^2)$ is the generator of the generative SDE and $M_t:=\sigma_t^{\top}\nabla^2\log G_t\,\sigma_t$. In continuous time the Doob relations \cref{eq:doob} read $\gen_tG^*_t=0$ with $G^*_0=G_0$, solved by $G^*_t(x)=\E[G_0(X_0)\mid X_t=x]$, together with the guided step for $G^*_t$. For $X_t,X'_t$ i.i.d.\ from $\pi_t$ and $\xi,\xi'\sim N(0,I)$ independent, define the local barriers
\begin{equation}
    \lambda^{\mathrm u}(t):=\tfrac{1}{\sqrt{2\pi}}\,\E\sqrt{\|g_t(X_t)\|^2+\|g_t(X'_t)\|^2},
    \qquad
    \lambda(t):=\tfrac12\,\E\big|\Psi_t(X_t,\xi)-\Psi_t(X'_t,\xi')\big|.
    \label{eq:local-barriers}
\end{equation}

\begin{theorem}
\label{thm:barrier}
Assume $\sigma_t$ is invertible and continuous, and $b_t$ and $\log G_t$ are smooth with bounded derivatives. Then $\lambda^{\mathrm u}$ and $\lambda$ are continuous on $[0,1]$, and if $N\to\infty$ with $\max_n\Delta t_n\to0$, the following hold uniformly in $n$.
\begin{enumerate}
    \item[(a)] If $\mathbf 1_{\mathrm{guided}}=0$ in \cref{eq:euler-step}, $r_{n-1,n}=\lambda^{\mathrm u}(t_n)\sqrt{\Delta t_n}+O(\Delta t_n)$, so $\sum_nr_{n-1,n}\to \infty$.
    \item[(b)] If $\mathbf 1_{\mathrm{guided}}=1$ in \cref{eq:euler-step}, $r_{n-1,n}=\int_{t_{n-1}}^{t_n}\lambda(t)\,\dd t+o(\Delta t_n)$, so $\sum_nr_{n-1,n}\to\Lambda:=\int_0^1\lambda(t)\,\dd t<\infty$, with $\Lambda=0$ if and only if $G_t\propto G^*_t$ and $\log G_t$ is affine in $x$ for every $t$.
\end{enumerate}
\end{theorem}

See \cref{app:barriers} for the proof. Unguided, the log-weight is a mean-zero fluctuation of order $\sqrt{\Delta t}$, which a mean absolute difference does not average out. Guided, the proposal correction is the log-ratio of two Gaussians with equal covariance and cancels the $\sqrt{\Delta t}$ term exactly. $\Lambda=0$ exactly when the guided proposal is the Doob proposal, so $\Lambda$ measures the deviation of the guided ladder from the Doob $h$-transform, and by \cref{eq:roundtrip} the restart rate tends to zero for the unguided ladder and to $(2+2\Lambda)^{-1}$ for the guided one.

Guided and unguided ladders therefore differ qualitatively, and the gap widens with parallelism: the unguided rejection is set by the diffusivity of the sampler, not by the potentials, so additional levels only add rejections, whereas the guided rejection is set by the Doob residual and its total is fixed. $\Lambda$ also sets the schedule. A guided ladder with $N\approx2\Lambda$ equi-rejection levels attains stable communication, after which a fresh reference draw reaches the target every $2+2\Lambda$ iterations whatever $N$, so a longer run rather than a finer ladder increases the sample size. An unguided ladder should have levels about $1.2$ apart in log-weight standard deviation, merged wherever $\nabla\log G_t$ is small, and where $G_t\equiv1$ a single level suffices, the coarse-to-fine ladder of \cref{sec:experiments}.

\subsection{Related Work}
\label{sec:related}

\paragraph{Parallel control.} CREPE \citep{he_crepe_2026} is the only other parallel control method (\cref{tab:landscape}). Exactness costs memory, $O(N^2d)$ against $O(Nd)$, but the gap is smaller in practice: CREPE's reversal error accumulates over the model steps an exchange spans and forces a fine ladder, whereas ENCORE tolerates a coarse one (\cref{sec:experiments}). ENCORE also returns the law of the full trajectory rather than its marginals, and applies without a noising kernel and to deterministic samplers.

\paragraph{Sequential control.} FKD \citep{singhal_general_2025} and TDS \citep{wu_practical_2023} are the sequential counterparts of unguided and guided ENCORE. Against their $O(Md)$, ENCORE uses less memory once $M>N^2$, i.e.\ once $M\ge4\Lambda^2$ at the $N\approx2\Lambda$ levels of \cref{sec:diffusion}. The expansions behind \cref{thm:barrier} govern their weights as well: the unguided log-weight variance does not vanish under refinement while the guided one does, recovering the annealed-SMC barrier of \citet{syed_optimized_2026} (\cref{app:smc-barrier}). RNE \citep{he_rne_2026} estimates the marginal ratio by the transition ratio $\Kb/\Kfh$ and is not exact, and FKC \citep{skreta_feynmankac_2025} derives weights in continuous time for marginal-based targets and reduces to TDS on the ladder of \cref{sec:exact}.

\paragraph{Guidance and search.} Guidance \citep{dhariwal_diffusion_2021,ho_classifierfree_2022,chung_diffusion_2023,song_lossguided_2023,bansal_universal_2024} constructs $\Lb_{n-1|n}$ from an estimate of $G^*_n$, and search methods such as SVDD \citep{li_derivativefree_2025} and LiDAR \citep{kim_lookahead_2026} select among candidate continuations. Neither targets $\pi_0$ \citep{bradley_classifierfree_2024,dandapanthula_are_2026}, and both supply the proposals and potentials that ENCORE corrects.

%% file: main/05_experiments.tex
\section{Experiments}
\label{sec:experiments}

In this section, we evaluate the empirical performance of ENCORE, with experimental details and additional results provided in \cref{app:experiments}.
We first test robustness to perturbed transition kernels in synthetic settings, then evaluate generalisation to molecular and image-generation tasks.

\begin{figure}[t]
\centering
\includegraphics[width=\linewidth]{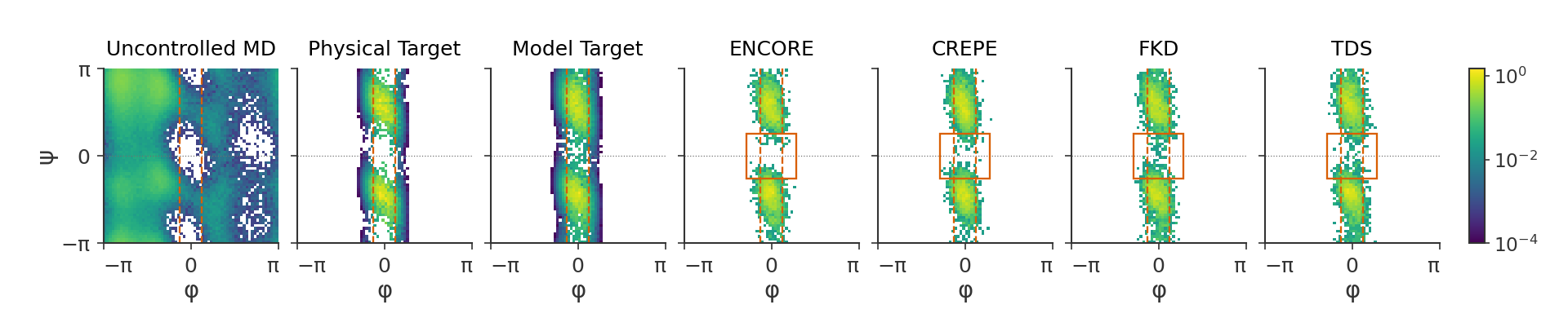}
\caption{Alanine transition-region sampling at \(800\,\mathrm K\); each method panel shows 5,000 samples.}
\label{fig:alanine-transition-region}
\end{figure}

\subsection{Synthetic Reward Tilting with Sampler Perturbations}
\label{sec:synthetic}

\begin{wrapfigure}{r}{0.35\textwidth}
\vspace{-1.5\baselineskip}
\centering
\includegraphics[width=\linewidth]{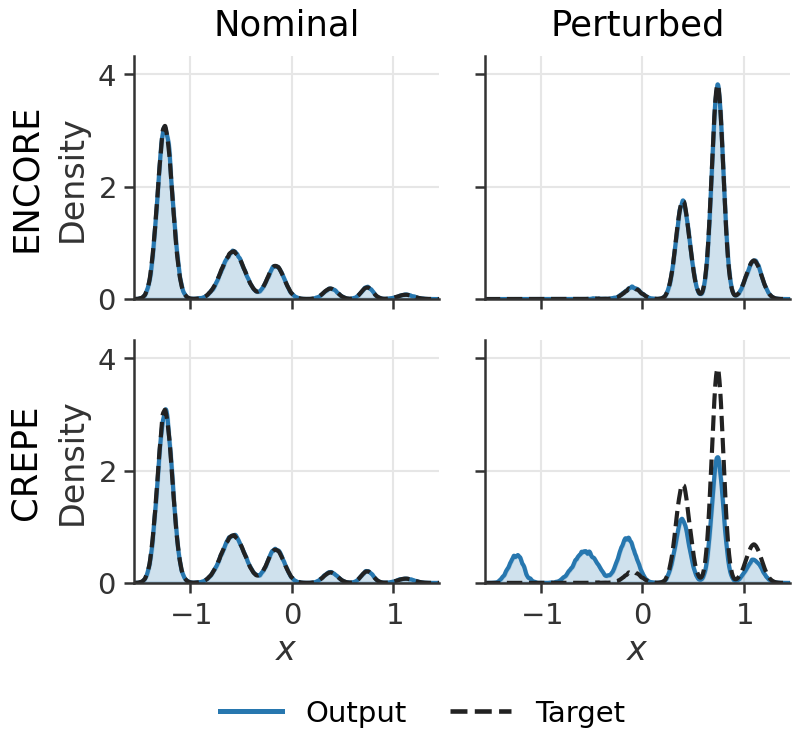}
\caption{1D densities and sampler-specific targets.}
\label{fig:gmm-endpoints}
\vspace{-1.5\baselineskip}
\end{wrapfigure}

We evaluate reward tilting on 6-mode 1D and 40-mode 2D Gaussian mixtures with exact scores under several reward settings.
To assess robustness to sampler misspecification, we compare nominal transitions with perturbed kernels obtained by altering the transition variance while keeping the reward fixed.
Because each perturbation changes the sampler output law \(p_0\), every method is evaluated against the corresponding target \(\pi_0\propto G_0p_0\).
ENCORE remains competitive across all evaluated settings and is notably robust to transition perturbations.
In the 1D example in \cref{fig:gmm-endpoints}, ENCORE and CREPE both closely match the nominal target; after perturbation, ENCORE remains accurate (\(\mathrm W_2=0.015\)), whereas CREPE substantially distorts the target mode weights (\(\mathrm W_2=0.650\)).
See \cref{app:gmm} for detailed 1D and 2D results.

\subsection{Sampling Molecular Transition Regions}
\label{sec:alanine-transition}

In molecular dynamics, transitions between metastable states pass through low-probability reactive regions that are rarely sampled by unbiased dynamics \citep{du_rare_2026,omidi_metalica_2026}.
Efficiently targeting these regions is therefore important but challenging.
We study alanine dipeptide, whose slow conformational changes are well described by the backbone torsion angles \(\phi\) and \(\psi\).
To bias sampling towards a transition region, we apply a harmonic restraint of width \(0.2\) rad to \(\phi\) in an \(800\,\mathrm K\) sampler while leaving \(\psi\) unconstrained.
We construct numerical references for the tilted model distribution and the corresponding physical ensemble using importance sampling over model outputs and molecular-dynamics (MD) frames, respectively.
In \cref{fig:alanine-transition-region}, the red boxes shows residual density connecting the two populated regions for CREPE, FKD, and TDS, whereas ENCORE separates them more clearly.
Across three seeds, ENCORE achieves the lowest mean regional TVD against the model reference (\(0.012\)), compared with \(0.027\) for CREPE and \(0.037\) for both FKD and TDS.
% See \cref{app:alanine-transition} for full results and experimental details.

\subsection{Boltzmann Tempering of Molecular Ensembles}
\label{sec:alanine}

Following \citet{skreta_feynmankac_2025,he_crepe_2026}, we temper \(800\,\mathrm K\) samplers to \(300\,\mathrm K\) for alanine dipeptide (ALDP), \(500\,\mathrm K\) for tetrapeptide (A4), and \(600\,\mathrm K\) for hexapeptide (A6).
For ENCORE and FKD, we formulate the temperature change as endpoint Boltzmann reweighting, \(G_0(x)=\exp\{-(\beta_{T_t}-\beta_{T_s})U(x)\}\), with potential energy \(U\) evaluated using OpenMM \citep{eastman_openmm_2023}; CREPE, RNE, and FKC use the model-density tempering constructions.
Across all evaluations, ENCORE achieves the lowest mean error in most cases and the second lowest in all others; see \cref{tab:encore-a6-results,tab:encore-boltzmann-results}.
The A6 TICA densities in \cref{fig:alanine-a6-tica} show broader conformational coverage for ENCORE and CREPE than for FKC-Noise and RNE, illustrating the diversity retained by RE in this setting.
% See \cref{app:alanine} for full results and experimental details.

\paragraph{Ladder spacing and exchange acceptance.}
\label{sec:alanine-acceptance}

We further compare dense and coarse-to-fine PT ladders, where the latter use a large initial jump followed by finer spacing; see \cref{fig:alanine-acceptance-diversity}.
Across the coarse high-noise interval, ENCORE accepts every exchange in our experiments, whereas CREPE rejects every proposal.
CREPE's RMSD error increases substantially under the coarse-to-fine ladder, while ENCORE remains robust to the wider spacing.
ENCORE's unit acceptance follows from \(\Lb=\Kb\) and \(G_n=G_{n-1}\equiv1\).
For all three peptides, ENCORE achieves lower mean energy TVD under both ladders and lower RMSD \(\mathrm W_1\) on the coarse-to-fine ladder.
See \cref{app:acceptance-diversity} for details.

\subsection{Text-to-Image Reward Tilting}
\label{sec:t2i}

We apply ImageReward tilting \citep{xu_imagereward_2023} to four-step SDXL-Turbo \citep{sauer_adversarial_2024} and DMD2-UNet \citep{yin_improved_2024}, using 32 prompts and 32 outputs per prompt.
ENCORE applies to both without a paired noising kernel.
CREPE uses an approximate pairing on SDXL-Turbo and is not evaluated on DMD2, whose native transitions lack the required pairing.
At matched chain counts, ENCORE and CREPE achieve similar mean ImageReward on SDXL-Turbo.
Compared with FKD, 16-chain ENCORE achieves comparable reward and prompt alignment with higher diversity scores on both samplers; on DMD2, their mean DINO-Vendi scores are \(3.276\) and \(1.473\), respectively (\cref{tab:encore-image-turbo,tab:encore-image-dmd2}).
\Cref{fig:t2i-qualitative} provides a qualitative comparison.
Since exact target samples are unavailable, these comparisons assess the reward--diversity trade-off rather than distributional accuracy.

\input{tabs/alanine_a6}

\begin{figure}[t]
\centering
\begin{minipage}[t]{0.272\textwidth}
\vspace{0pt}
\centering
\includegraphics[height=1.65in]{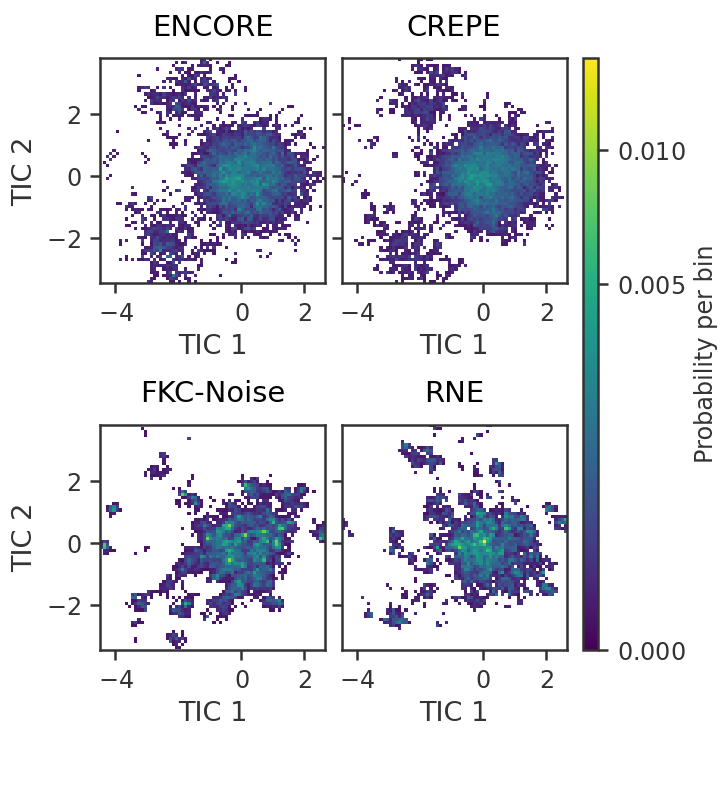}
\captionof{figure}{A6 target-fitted TICA densities; 49,000 samples per method.}
\label{fig:alanine-a6-tica}
\end{minipage}%
\hfill%
\begin{minipage}[t]{0.355\textwidth}
\vspace{0pt}
\centering
\includegraphics[height=1.65in]{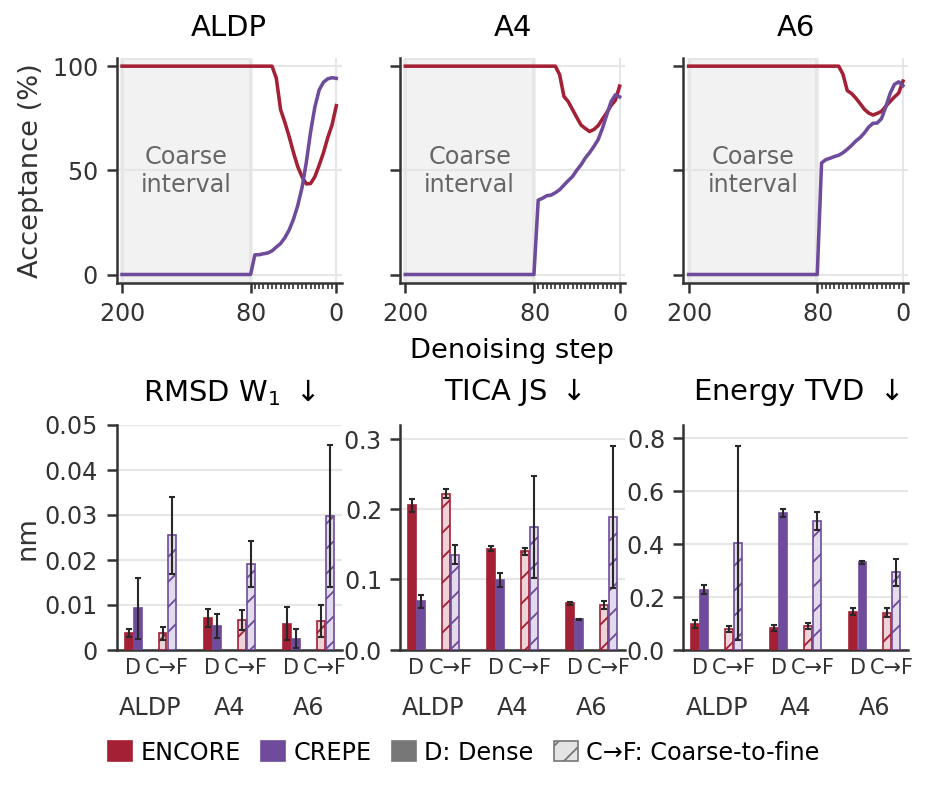}
\captionof{figure}{Coarse-to-fine acceptance (top) and sampling errors under both ladders (bottom).}
\label{fig:alanine-acceptance-diversity}
\label{fig:acceptance-alanine}
\label{fig:alanine-ladder-errors}
\end{minipage}%
\hfill%
\begin{minipage}[t]{0.352\textwidth}
\vspace{0pt}
\centering
\includegraphics[height=1.65in]{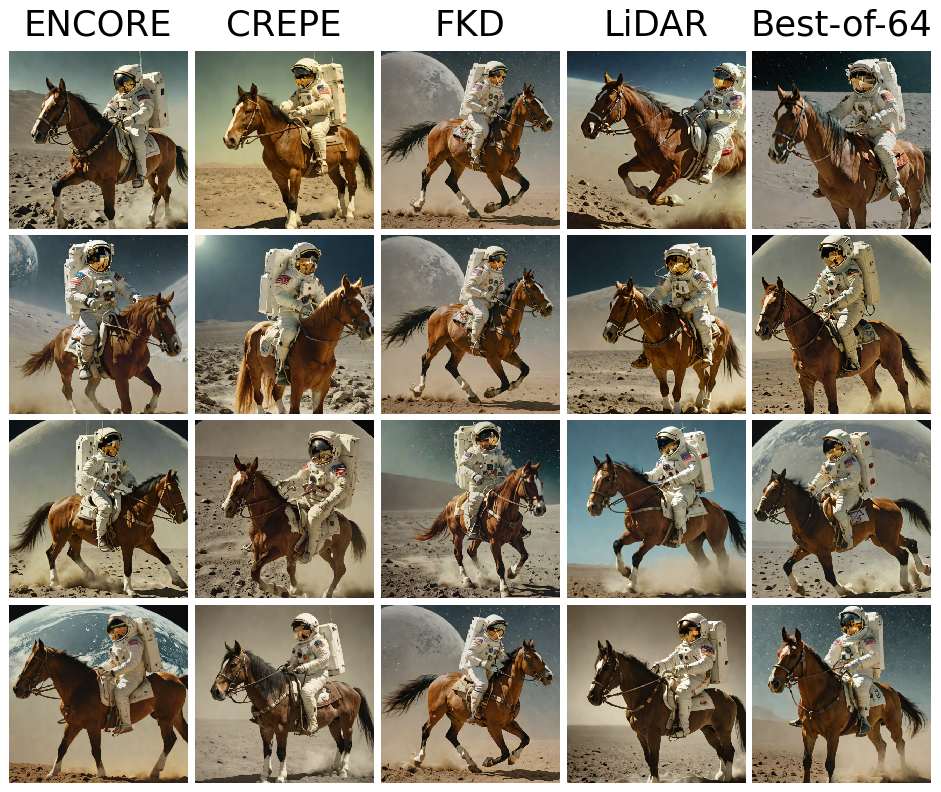}
\captionof{figure}{SDXL-Turbo samples for \emph{a photograph of an astronaut riding a horse}.}
\label{fig:t2i-qualitative}
\end{minipage}
\end{figure}

%% file: tabs/alanine_a6.tex
\begin{table}[t]
  \centering
  \caption{A6 Boltzmann-tempering error using 5,000 samples per evaluation
 (mean \(\pm\) standard deviation over three subsampling evaluations from one sampler run).}
  \label{tab:encore-a6-results}
  % \vspace{1.8ex}
  \begingroup
  \enctablesetup
  \resizebox{\textwidth}{!}{%
  \begin{tabular}{@{}clcccc@{}}
    \toprule
    System & Method & Energy TVD\(\downarrow\) & Distance TVD\(\downarrow\) & Aligned \(\mathrm{W}_2\downarrow\) & Torsion MMD\(^2\)\(\downarrow\) \\
    \midrule
    \multirow[c]{7}{50mm}{\encbtsystem{Alanine Hexapeptide}{A6}{800}{600}{a6}}
         & ENCORE    & \encbtbest{0.098200}{0.001705} & \encbtbest{0.005450}{0.000307} & \encbtsecond{1.253078}{0.003031} & \encbtsecond{0.001642}{0.000218} \\
         & CREPE     & \encbtvalue{0.336600}{0.007045} & \encbtvalue{0.013980}{0.000705} & \encbtvalue{1.287915}{0.001899} & \encbtvalue{0.004168}{0.000196} \\
         & RNE       & \encbtvalue{0.160000}{0.004198} & \encbtvalue{0.022591}{0.000815} & \encbtvalue{1.591211}{0.002857} & \encbtvalue{0.009120}{0.000598} \\
         & FKC-Score & \encbtvalue{0.168667}{0.004739} & \encbtvalue{0.012311}{0.000477} & \encbtvalue{1.596248}{0.005398} & \encbtvalue{0.011241}{0.000576} \\
         & FKC-Noise & \encbtsecond{0.134600}{0.004685} & \encbtvalue{0.009663}{0.000895} & \encbtvalue{1.584179}{0.003716} & \encbtvalue{0.008250}{0.000308} \\
         & FKD       & \encbtvalue{0.196267}{0.006130} & \encbtsecond{0.005946}{0.000267} & \encbtbest{1.192835}{0.002978} & \encbtbest{0.001129}{0.000043} \\
    \cmidrule(lr){2-6}
         & Uncontrolled & \encbtvalue{0.889667}{0.001746} & \encbtvalue{0.020097}{0.001293} & \encbtvalue{1.262478}{0.002014} & \encbtvalue{0.009424}{0.000226} \\
    \bottomrule
  \end{tabular}}
  \endgroup
\end{table}

%% file: main/06_conclusion.tex
\section{Limitations and Future Work}

ENCORE achieves exact replica exchange by augmenting each replica with a partial generation trajectory instead of retaining only its current state.
This incurs additional memory cost, and its practical efficiency depends on the choice of replica ladder, intermediate potentials, and proposal kernels.
Moreover, the controlled target is defined relative to the sampler output distribution \(p_0\), so ENCORE does not correct misspecification of the underlying generative model.
Future work could develop more memory-efficient trajectory representations and adaptive designs for the ladder, potentials, and proposals, potentially drawing on ideas from SMC-based diffusion control.

%% file: appendices/theory.tex
% Requires the macros in notation_macros.tex.
\providecommand{\trunc}{\mathcal T}

\section{Background}
\label{app:background}
\label{app:algorithms}
\label{app:notation}

This appendix expands \cref{sec:background}: the diffusion sampler and its time reversal, the Doob $h$-transform in discrete and continuous time, and the two annealed Monte Carlo samplers of \cref{sec:bg-mc} in full. Results quoted from the literature are stated without proof.

\subsection{Diffusion Samplers}
\label{app:bg-diffusion}

The sampler of \cref{sec:bg-diffusion} draws $x_N\sim\eta$ and then, for $n=N,\ldots,1$, $x_{n-1}\sim\Kb_{n-1|n}(\dd x_{n-1}\mid x_n)$. Its marginals and its path law are
\begin{align}
    p_N&=\eta, & p_{n-1}(\dd x_{n-1})&=\int p_n(\dd x_n)\,\Kb_{n-1|n}(\dd x_{n-1}\mid x_n),\\
    \bp_N&=\eta, & \bp_{n-1}&=\bp_n\otimes\Kb_{n-1|n},
\end{align}
the second line being \cref{eq:path-law}, where $\bx_n=(x_N,\ldots,x_n)$ is the partial trajectory generated down to level $n$. The kernels can be simulated, and for a stochastic sampler they have densities. The marginals $p_n$ with $n<N$ are integrals over the learned transitions and are not available in closed form. This is the obstacle that every correction in this paper works around.

\paragraph{Time reversal.} The kernel $\Kf_{n|n-1}$ of \cref{eq:time-reversal} is the conditional law of $x_n$ given $x_{n-1}$ under the sampler, defined by the identity of laws of the pair $(x_{n-1},x_n)$,
\begin{align}
    p_{n-1}\otimes\Kf_{n|n-1}&=p_n\otimes\Kb_{n-1|n},\quad\text{i.e.}\nonumber\\
    p_{n-1}(\dd x_{n-1})\,\Kf_{n|n-1}(\dd x_n\mid x_{n-1})&=p_n(\dd x_n)\,\Kb_{n-1|n}(\dd x_{n-1}\mid x_n).
    \label{eq:app-reversal}
\end{align}
It exists whenever the joint law on the right admits a regular conditional distribution, which is the case on the Euclidean spaces used here. When densities exist,
\begin{equation}
    \Kf_{n|n-1}(x_n\mid x_{n-1})=\frac{p_n(x_n)}{p_{n-1}(x_{n-1})}\,\Kb_{n-1|n}(x_{n-1}\mid x_n),
\end{equation}
which involves the ratio of two marginals. It can therefore be neither evaluated nor simulated.

\paragraph{The noising kernel.} A diffusion model is trained on a noising process, and the transition $\Kfh_{n|n-1}$ of that process from $t_{n-1}$ to $t_n$ is tractable, typically Gaussian. If the learned sampler were the exact time reversal of the noising process at the chosen discretisation, then $\Kfh_{n|n-1}$ would equal $\Kf_{n|n-1}$. Model error and discretisation error make the two differ, and the local reversal error
\begin{equation}
    \delta_n(x_n,x_{n-1}):=\log\frac{\Kf_{n|n-1}(x_n\mid x_{n-1})}{\Kfh_{n|n-1}(x_n\mid x_{n-1})}
    \label{eq:app-reversal-error}
\end{equation}
is what \cref{app:crepe-weights} quantifies. Distilled few-step samplers are not trained as reversals of a noising process and have no $\Kfh_{n|n-1}$ at all. A deterministic sampler has kernels without densities. The unguided ENCORE weight $w_{n-1|n}=G_{n-1}/G_n$ involves no kernel density and applies to it (\cref{app:general-kernels}).

\subsection{The Doob \texorpdfstring{$h$}{h}-Transform}
\label{app:bg-doob}

\paragraph{Discrete time.} For the target $\pi_0\propto G_0\,p_0$ of \cref{eq:endpoint-target}, the Doob potentials and proposals of \cref{eq:doob} are
\begin{align}
    G^*_0&:=G_0,\qquad G^*_n:=\Kb_{n-1|n}G^*_{n-1},\nonumber\\
    \Lb^*_{n-1|n}(\dd x_{n-1}\mid x_n)&:=\frac{G^*_{n-1}(x_{n-1})}{G^*_n(x_n)}\,\Kb_{n-1|n}(\dd x_{n-1}\mid x_n).
    \label{eq:app-doob}
\end{align}
Unrolling the recursion gives
\begin{equation}
    G^*_n(x_n)=\int\Kb_{n-1|n}(\dd x_{n-1}\mid x_n)\cdots\Kb_{0|1}(\dd x_0\mid x_1)\,G_0(x_0)=\E\big[G_0(X_0)\mid X_n=x_n\big]
\end{equation}
under the sampler, so $G^*_n(X_n)$ is a martingale in the generation direction. $\Lb^*_{n-1|n}$ is a probability kernel because $\int G^*_{n-1}(x_{n-1})\Kb_{n-1|n}(\dd x_{n-1}\mid x_n)=G^*_n(x_n)$. The chain that starts from $\eta^*\propto G^*_N\eta$ and follows $\Lb^*$ has path law
\begin{align}
    \eta^*(\dd x_N)\prod_{n=N}^{1}\Lb^*_{n-1|n}(\dd x_{n-1}\mid x_n)
    &=\frac{G^*_N(x_N)\,\eta(\dd x_N)}{\E_\eta[G^*_N]}\prod_{n=N}^{1}\frac{G^*_{n-1}(x_{n-1})}{G^*_n(x_n)}\,\Kb_{n-1|n}(\dd x_{n-1}\mid x_n)\nonumber\\
    &=\frac{G^*_0(x_0)}{\E_\eta[G^*_N]}\;\eta(\dd x_N)\prod_{n=N}^{1}\Kb_{n-1|n}(\dd x_{n-1}\mid x_n),
    \label{eq:app-doob-path}
\end{align}
since the potentials telescope, and $\E_\eta[G^*_N]=\E_{p_0}[G_0]$ by the martingale property. The right side is the sampler's path law reweighted by $G_0(x_0)$, so its $x_0$-marginal is $G_0\,p_0/\E_{p_0}[G_0]=\pi_0$ \citep{delmoral_feynmankac_2004}. This is the exact guidance of \cref{sec:bg-control}. It is intractable because each $G^*_n$ is a conditional expectation over the remaining generation.

\paragraph{The Doob relations on the ladder.} \Cref{thm:exact}(c) refers to the two relations of \cref{eq:doob} for a general ladder $\pi_n\propto G_np_n$ with proposals $\Lb_{n-1|n}$:
\begin{align}
    \text{potential relation:}&\qquad G_n\propto\Kb_{n-1|n}G_{n-1},\nonumber\\
    \text{proposal relation:}&\qquad \Lb_{n-1|n}(\dd x_{n-1}\mid x_n)\propto G_{n-1}(x_{n-1})\,\Kb_{n-1|n}(\dd x_{n-1}\mid x_n).
    \label{eq:app-doob-relations}
\end{align}
Iterating the potential relation from $G_0$ gives $G_n\propto G^*_n$ for every $n$, the constants compounding along the way and cancelling in $\pi_n$. The proposal relation then reads $\Lb_{n-1|n}=\Lb^*_{n-1|n}$. At the top level the potential relation reads $1=G_N\propto G^*_N$, which holds only if $G^*_N$ is constant, i.e.\ if $x_N$ is independent of $x_0$ under the sampler. For a diffusion this holds up to the small dependence of $x_0$ on the initial noise, so $\eta^*\approx\eta$ and the top exchange carries a correspondingly small residual rejection.

\paragraph{Continuous time.} Let the sampler discretise the generative SDE
\begin{equation}
    \dd X_t=b_t(X_t)\,\dd t+\sigma_t\,\dd B_t,\qquad\text{run from }t=1\text{ to }t=0,
    \label{eq:app-sde}
\end{equation}
whose generator, with the sign of $\partial_t$ reflecting that $t$ decreases along generation, is
\begin{equation}
    \gen_t:=-\partial_t+b_t\cdot\nabla+\tfrac12\operatorname{tr}\big(\sigma_t\sigma_t^{\top}\nabla^2\big).
    \label{eq:app-generator}
\end{equation}
The continuous-time Doob potential
\begin{equation}
    G^*_t(x):=\E\big[G_0(X_0)\mid X_t=x\big]
\end{equation}
solves the backward Kolmogorov equation $\gen_tG^*_t=0$ on $(0,1]$ with terminal condition $G^*_0=G_0$, and is its unique solution with bounded logarithm under the assumptions of \cref{app:barriers} \citep[Ch.~1]{delmoral_feynmankac_2004}. The $h$-transformed process
\begin{equation}
    \dd X_t=\big(b_t(X_t)+\sigma_t\sigma_t^{\top}\nabla\log G^*_t(X_t)\big)\,\dd t+\sigma_t\,\dd B_t,
    \qquad X_1\sim\eta^*\propto G^*_1\eta,
    \label{eq:app-htransform}
\end{equation}
has terminal law $\pi_0$: this is the continuous-time proposal relation, and the guided step of \cref{eq:euler-step} is its Euler--Maruyama discretisation with a general $G_t$ in place of $G^*_t$.

Two facts about this discretisation are used in \cref{app:barriers}. The first is that for a potential whose logarithm is affine in $x$, the guided step is the Doob proposal of that potential exactly. Indeed, the sampler step is Gaussian, $\Kb_{t'|t}(\cdot\mid x)=N(\mu(x),\Sigma)$ with $\mu(x)=x+b_t(x)\Delta t$ and $\Sigma=\sigma_t\sigma_t^{\top}\Delta t$, and if $\log G_{t'}(y)=c^{\top}y+d$ then completing the square gives
\begin{align}
    G_{t'}(y)\,\Kb_{t'|t}(\dd y\mid x)
    &\propto\exp\Big(c^{\top}y-\tfrac12(y-\mu(x))^{\top}\Sigma^{-1}(y-\mu(x))\Big)\,\dd y\nonumber\\
    &\propto\exp\Big(-\tfrac12\big(y-\mu(x)-\Sigma c\big)^{\top}\Sigma^{-1}\big(y-\mu(x)-\Sigma c\big)\Big)\,\dd y,
    \label{eq:app-complete-square}
\end{align}
so the twisted kernel is $N(\mu(x)+\Sigma c,\Sigma)$: the sampler step with its mean shifted by $\Sigma\nabla\log G_{t'}=\sigma_t\sigma_t^{\top}\nabla\log G_{t'}\,\Delta t$, which is the guided step. The second fact is that the rate of change of $\log G_t$ along the sampler is $\gen_tG_t/G_t$, and that it is constant in $x$ if and only if $G_t\propto G^*_t$ (\cref{lem:doob-residual}). Equivalently, $G_t$ is an eigenfunction of $\gen_t$, $\gen_tG_t=c(t)G_t$. The Doob potential is the eigenfunction with eigenvalue zero, the space-time harmonic function of the diffusion.

\subsection{Parallel Annealing: Non-Equilibrium Replica Exchange}
\label{app:bg-parallel}

\paragraph{Replicas.} A \emph{replica} is a state of $\Xsp$ carried by a level of the ladder. The sampler holds one replica per level, $x_0,\ldots,x_N$, and targets $\pi_N\otimes\cdots\otimes\pi_0$ by a Markov chain on $\Xsp^{N+1}$. A replica keeps its identity when it is exchanged between levels, so one can follow it up and down the ladder over the iterations. A fresh replica enters at the reference level $N$ when that level is redrawn from $\pi_N$, and a replica leaves the ladder when the level it occupies is redrawn.

\paragraph{The sampler.} Non-equilibrium replica exchange \citep{ballard_replica_2009,zhang_accelerated_2026} is \cref{alg:ne-re}. It proposes, for a pair of adjacent replicas $(x_{n-1},x_n)$, a downward transition $y_{n-1}\sim\Lb_{n-1|n}(\cdot\mid x_n)$ from the upper replica and an upward transition $y_n\sim\Lf_{n|n-1}(\cdot\mid x_{n-1})$ from the lower one, and replaces the pair by $(y_{n-1},y_n)$ with the probability \cref{eq:apt}. Parallel tempering \citep{swendsen_replica_1986,geyer_markov_1991,hukushima_exchange_1996} is the special case $\Lb_{n-1|n}=\Lf_{n|n-1}=\mathrm{Id}$, for which the incremental weight is $w_{n-1,n}(x_{n-1},x_n)=\pi_{n-1}(x_{n-1})/\pi_n(x_n)$ and \cref{eq:apt} reduces to the swap acceptance
\begin{equation}
    \alpha_{n-1,n}=\min\Big\{1,\ \frac{\pi_{n-1}(x_n)\,\pi_n(x_{n-1})}{\pi_{n-1}(x_{n-1})\,\pi_n(x_n)}\Big\}.
\end{equation}
The even--odd schedule attempts the pairs $(n-1,n)$ with $n\equiv m\pmod 2$ at iteration $m$. The pairs attempted in one iteration are disjoint, so the exchanges run in parallel, and consecutive iterations attempt complementary sets of pairs, which is what makes the communication non-reversible and lets a replica move several levels in one direction without turning back \citep{syed_nonreversible_2022}.

\begin{algorithm}[H]
\caption{Non-equilibrium replica exchange}
\label{alg:ne-re}
\KwIn{iterations $M$; ladder $(\pi_n)_{n=0}^N$; kernels $(\Lb_{n-1|n})$, $(\Lf_{n|n-1})$.}
Initialise one replica $x^0_n$ per level $n=0,\ldots,N$\;
\For{$m=1,\ldots,M$}{
    Draw $x^m_N\sim\pi_N$
    \tcp*[r]{refresh reference replica}
    \ForEach{$n\in\{1,\ldots,N\}$ with $n\equiv m\pmod 2$, in parallel}{
        Draw $y^m_{n-1}\sim\Lb_{n-1|n}(\dd y_{n-1}\mid x^{m-1}_n)$ and $y^m_n\sim\Lf_{n|n-1}(\dd y_n\mid x^{m-1}_{n-1})$\;
        With probability $\alpha^m_{n-1,n}:=\alpha_{n-1,n}(x^{m-1}_{n-1},x^{m-1}_n;y^m_{n-1},y^m_n)$ of \cref{eq:apt}, set $(x^m_{n-1},x^m_n)\gets(y^m_{n-1},y^m_n)$
        \tcp*[r]{exchange}
    }
    Set $x^m_n\gets x^{m-1}_n$ for every level not exchanged\;
    Optionally apply a $\pi_n$-invariant MCMC move to each $x^m_n$, in parallel over $n$\;
    Record $x^m_0$\;
}
\end{algorithm}

\paragraph{Restarts and round trips.} Following a replica through the iterations, a \emph{restart} occurs when a replica that entered at the reference level $N$ reaches level $0$. The recorded sample is then a fresh draw from the reference, transported to the target through accepted exchanges, and it is independent of the samples recorded before the previous restart. A \emph{round trip} is a restart followed by the return of the same replica to level $N$. The \emph{restart rate} $\tau$ is the long-run number of restarts per iteration \citep{katzgraber_feedbackoptimized_2006,syed_nonreversible_2022}. It is the natural measure of particle diversity for parallel annealing: between restarts the recorded samples are correlated through the replica that produced them, and the number of effectively independent samples after $M$ iterations is of order $\tau M$.

\paragraph{Estimators.} In the notation of \cref{alg:ne-re}, $x^m_0$ is the level-$0$ state recorded at iteration $m$ and $(y^m_{n-1},x^{m-1}_n)$ the downward pair proposed at iteration $m$ at level $n$. The recorded states estimate expectations under the target by the Monte Carlo average, and the incremental weights of the downward proposals estimate the ratio of normalising constants $Z:=\mathcal Z_0/\mathcal Z_N$ level by level:
\begin{equation}
    \widehat\pi_0^{\mathrm{RE}}[f]:=\frac{1}{M}\sum_{m=1}^{M}f(x^m_0),
    \qquad
    \widehat Z^{\mathrm{RE}}:=\prod_{n=1}^{N}\Big(\frac{1}{M}\sum_{m=1}^{M}w_{n-1,n}(y^m_{n-1},x^{m-1}_n)\Big).
    \label{eq:app-pt-estimators}
\end{equation}
The second is the stepping-stone estimator \citep{xie_improving_2011}: at stationarity the downward pair has law $\pi_n\otimes\Lb_{n-1|n}$, and the incremental weight \cref{eq:generic-weight} has expectation $\mathcal Z_{n-1}/\mathcal Z_n$ under that law, so each factor is consistent for one ratio of adjacent normalising constants and the product is consistent for $Z$ as $M\to\infty$.

\paragraph{The rejection rate.} At stationarity the pair $(x^{m-1}_{n-1},y^m_n)$ proposed upwards has law $\pi_{n-1}\otimes\Lf_{n|n-1}$ and the pair $(y^m_{n-1},x^{m-1}_n)$ proposed downwards has law $\pi_n\otimes\Lb_{n-1|n}$, and \citet[Theorem~1]{zhang_accelerated_2026} show that
\begin{equation}
    r_{n-1,n}=r\big(\pi_{n-1}\otimes\Lf_{n|n-1},\ \pi_n\otimes\Lb_{n-1|n}\big),
    \qquad r(p,q):=\TV(p\otimes q,\,q\otimes p),
\end{equation}
which is \cref{eq:apt-rejection}. Here $r(p,q)$ is the rejection rate of a swap between two independent replicas with laws $p$ and $q$ on a common space. It is symmetric in $(p,q)$, vanishes if and only if $p=q$, and satisfies $\TV(p,q)\le r(p,q)\le2\,\TV(p,q)$. \cref{lem:swap-tv} below records these facts, since they are used in the proof of \cref{thm:exact}.

\paragraph{Schedule and number of levels.} How $\tau$ depends on the schedule and on $N$, the global communication barrier $\Lambda$ that the summed rejection rates converge to as the schedule is refined, the equi-rejection schedule, and the recommendation $N\approx2\Lambda$ quoted in \cref{sec:bg-mc} are developed in \citet{syed_nonreversible_2022}. \cref{app:barriers} computes $\Lambda$ for the ladder of \cref{sec:exact} on a diffusion.

\subsection{Sequential Annealing: Sequential Monte Carlo Samplers}
\label{app:bg-sequential}

Sequential Monte Carlo samplers \citep{delmoral_sequential_2006} are \cref{alg:smc}. Particle $m$ at level $n$ is a state $x^m_n$ with weight $w^m_n$, and the weighted population approximates $\pi_n$ in the sense that
\begin{equation}
    \frac{\sum_{m=1}^Mw^m_n\,f(x^m_n)}{\sum_{m=1}^Mw^m_n}\ \xrightarrow[M\to\infty]{}\ \pi_n(f)
    \qquad\text{for bounded measurable }f.
\end{equation}
Without resampling, $w^m_0$ is the product of the incremental weights \cref{eq:generic-weight} along the particle's path, and $\tfrac1M\sum_mw^m_0$ is an unbiased estimator of the ratio of normalising constants of $\pi_0$ and $\pi_N$. Resampling replaces the population by a multinomial draw from the normalised weights when they degenerate, and a $\pi_{n-1}$-invariant mutation may follow to restore diversity. Forward proposals appear only inside \cref{eq:generic-weight}, which is evaluated and never simulated.

\begin{algorithm}[H]
\caption{Sequential Monte Carlo sampler}
\label{alg:smc}
\KwIn{number of particles $M$; ladder $(\pi_n)_{n=0}^N$; kernels $(\Lb_{n-1|n})$, $(\Lf_{n|n-1})$; resampling rule.}
Draw $x^{m}_N\sim\pi_N$ and set $w^{m}_N\gets1$ for $m=1,\ldots,M$\;
\For{$n=N,\ldots,1$}{
    \ForEach{$m=1,\ldots,M$ in parallel}{
        Draw $x^{m}_{n-1}\sim\Lb_{n-1|n}(\dd x_{n-1}\mid x^{m}_n)$
        \tcp*[r]{propagate}
        $w^{m}_{n-1}\gets w^{m}_n\,w_{n-1,n}(x^{m}_{n-1},x^{m}_n)$
        \tcp*[r]{reweight}
    }
    \If{the resampling rule fires}{
        Resample $x^{1:M}_{n-1}$ from the normalised weights and set $w^{1:M}_{n-1}\gets1$, optionally followed by a $\pi_{n-1}$-invariant mutation\;
    }
}
\Return the weighted population $(x^{1:M}_0,w^{1:M}_0)$, which targets $\pi_0$\;
\end{algorithm}

\paragraph{Estimators.} The weighted population at level $0$ estimates expectations under the target, and the average weights accumulated between resampling events estimate $Z=\mathcal Z_0/\mathcal Z_N$:
\begin{equation}
    \widehat\pi_0^{\mathrm{SMC}}[f]:=\frac{\sum_{m=1}^Mw^m_0\,f(x^m_0)}{\sum_{m=1}^Mw^m_0},
    \qquad
    \widehat Z^{\mathrm{SMC}}:=\prod_{k=1}^{R}\Big(\frac1M\sum_{m=1}^{M}w^{m}_{(k)}\Big),
    \label{eq:app-smc-estimators}
\end{equation}
where $R$ is the number of resampling events and $w^m_{(k)}$ is the weight of particle $m$ accumulated since the previous resampling event, so that with resampling at every level $\widehat Z^{\mathrm{SMC}}=\prod_{n=1}^N\big(\tfrac1M\sum_mw_{n-1,n}(x^m_{n-1},x^m_n)\big)$. The expectation estimator is consistent, and $\widehat Z^{\mathrm{SMC}}$ is unbiased for $Z$ for any $M$ and any adaptive resampling rule, and consistent as $M\to\infty$ \citep{delmoral_feynmankac_2004,syed_optimized_2026}.

\paragraph{Relative variance and the discrepancy.} Since $\widehat Z^{\mathrm{SMC}}$ is unbiased, its relative variance $\Var[\widehat Z^{\mathrm{SMC}}/Z]$ is the natural measure of particle diversity for sequential annealing, and it is what an annealing schedule should minimise \citep{syed_optimized_2026}. Under the performance model of that paper (independent incremental weights across levels and particles, and efficient local exploration in the sense that the weight at each level is distributed as under $\pi_n\otimes\Lb_{n-1|n}$), the relative variance depends on the schedule only through the incremental discrepancies
\begin{equation}
    \mathcal D_{n-1,n}:=\log\Big(1+\chi^2\big(\pi_{n-1}\otimes\Lf_{n|n-1}\,\big\|\,\pi_n\otimes\Lb_{n-1|n}\big)\Big)
    =\log\frac{\E[w_{n-1,n}^2]}{\E[w_{n-1,n}]^2},
    \label{eq:app-discrepancy}
\end{equation}
the R\'enyi $2$-divergence between the backward target and the forward proposal at level $n$, the expectations under $\pi_n\otimes\Lb_{n-1|n}$. Note that these are the same two joint laws whose swap rejection rate is \cref{eq:apt-rejection}. With $R$ resampling events accumulating discrepancies $\mathcal D_{(1)},\ldots,\mathcal D_{(R)}$ between them \citep[Theorem~1]{syed_optimized_2026},
\begin{equation}
    \Var\Big[\frac{\widehat Z^{\mathrm{SMC}}}{Z}\Big]=\prod_{k=1}^{R}\Big(1+\frac{e^{\mathcal D_{(k)}}-1}{M}\Big)-1,
    \qquad
    \mathcal D_{0,N}:=\sum_{n=1}^{N}\mathcal D_{n-1,n}=\sum_{k=1}^R\mathcal D_{(k)}.
    \label{eq:app-smc-variance}
\end{equation}
For a fixed number of resampling events the right side is an increasing function of the total discrepancy $\mathcal D_{0,N}$, minimised over its distribution when the $\mathcal D_{(k)}$ are equal, and for $\mathcal D_{0,N}$ small it is $\approx\mathcal D_{0,N}/M$. Optimising the annealing schedule for sequential annealing therefore means minimising $\mathcal D_{0,N}$, the sum of the R\'enyi $2$-divergences between adjacent levels. In the small-weight regime $\mathcal D_{n-1,n}=\Var(\ell_n)+O(\E|\ell_n|^3)$ with $\ell_n:=\log w_{n-1,n}$, so $\mathcal D_{0,N}$ is the summed log-weight variance to leading order. \cref{app:smc-barrier} evaluates it for the sequential control weights of FKD and TDS.

\paragraph{Computational duality.} Both samplers fill the same array $(x^m_n)$ of \cref{sec:bg-mc}, sequential annealing one level at a time in parallel over $m$ and parallel annealing one iteration at a time in parallel over $n$. The estimators \cref{eq:app-pt-estimators} and \cref{eq:app-smc-estimators} are the same functions of this array with the roles of $m$ and $n$ interchanged, and \citet[Section~6.3]{syed_optimized_2026} show that under the two performance models $\widehat Z^{\mathrm{RE}}$ and $\widehat Z^{\mathrm{SMC}}$ with resampling at every level have the same law. Consequently the schedule that minimises the variance of $\widehat Z^{\mathrm{SMC}}$ also minimises the variance of $\widehat Z^{\mathrm{RE}}$, and the same divergences between $\pi_{n-1}\otimes\Lf_{n|n-1}$ and $\pi_n\otimes\Lb_{n-1|n}$ govern both: the rejection rate \cref{eq:apt-rejection} through $r(p,q)\le\sqrt{\chi^2(q\|p)/2}$ (\cref{lem:swap-tv}), the discrepancy through $\log(1+\chi^2)$. The restart-rate barrier of parallel annealing and the variance barrier of sequential annealing scale with each other, and the schedules that optimise the two criteria are comparable \citep[Section~E.1]{syed_optimized_2026}.

\section{CREPE and the Reversal Error}
\label{app:crepe-weights}

CREPE \citep{he_crepe_2026} is \cref{alg:ne-re} on the ladder of \cref{sec:exact} with the noising kernel as forward proposal, $\Lf_{n|n-1}=\Kfh_{n|n-1}$, and with the backward weight $w_{n-1|n}$ of \cref{eq:weight-factor} used in place of the full incremental weight in the acceptance probability \cref{eq:apt}. In the notation of \cref{eq:weight-factor}, the forward weight it drops is
\begin{equation}
    w_{n|n-1}(x_{n-1},x_n)=\frac{\Kfh_{n|n-1}(x_n\mid x_{n-1})}{\Kf_{n|n-1}(x_n\mid x_{n-1})}=e^{-\delta_n(x_n,x_{n-1})},
\end{equation}
with $\delta_n$ the local reversal error \cref{eq:app-reversal-error}, and CREPE sets it to one.

\begin{proposition}
\label{prop:crepe}
The CREPE exchange at $(n-1,n)$ is $\pi_n\otimes\pi_{n-1}$-reversible if and only if $\delta_n\equiv0$, i.e.\ $\Kfh_{n|n-1}=\Kf_{n|n-1}$. Otherwise, with $x_{n-1},x_n$ the current replicas and $y_{n-1}\sim\Lb_{n-1|n}(\cdot\mid x_n)$, $y_n\sim\Kfh_{n|n-1}(\cdot\mid x_{n-1})$ the proposals, its log acceptance ratio differs from the correct one by
\begin{equation}
    \epsilon_n=\delta_n(x_n,y_{n-1})-\delta_n(y_n,x_{n-1}),
    \label{eq:crepe-error}
\end{equation}
one reversal error for each of the two pairs it exchanges. When the exchange spans several model steps, the errors of those steps add. The ENCORE ratio has $\epsilon_n\equiv0$ for every sampler.
\end{proposition}

\begin{proof}
Write $w^{\widehat P}_{n-1,n}$ for the incremental weight \cref{eq:generic-weight} with forward proposal $\Kfh_{n|n-1}$, and $A^{\star}_{n-1,n}$ for the Metropolis--Hastings ratio of the move for the target $\pi_n\otimes\pi_{n-1}$.

\emph{The correct ratio.} By \cref{eq:apt},
\begin{equation}
    A^{\star}_{n-1,n}=\frac{w^{\widehat P}_{n-1,n}(y_{n-1},x_n)}{w^{\widehat P}_{n-1,n}(x_{n-1},y_n)}.
\end{equation}

\emph{The ratio CREPE uses.} By \cref{eq:weight-factor} with $\Lf=\Kfh$,
\begin{align}
    w^{\widehat P}_{n-1,n}(x_{n-1},x_n)
    &=w_{n-1|n}(x_{n-1},x_n)\,\frac{\Kfh_{n|n-1}(x_n\mid x_{n-1})}{\Kf_{n|n-1}(x_n\mid x_{n-1})}\\
    &=w_{n-1|n}(x_{n-1},x_n)\,e^{-\delta_n(x_n,x_{n-1})}.
\end{align}
CREPE replaces each $w^{\widehat P}_{n-1,n}$ by $w_{n-1|n}$, so
\begin{align}
    A^{\mathrm{CREPE}}_{n-1,n}
    &=\frac{w_{n-1|n}(y_{n-1},x_n)}{w_{n-1|n}(x_{n-1},y_n)}\\
    &=\frac{w^{\widehat P}_{n-1,n}(y_{n-1},x_n)\,e^{\delta_n(x_n,y_{n-1})}}{w^{\widehat P}_{n-1,n}(x_{n-1},y_n)\,e^{\delta_n(y_n,x_{n-1})}}\\
    &=A^{\star}_{n-1,n}\,\exp\big\{\delta_n(x_n,y_{n-1})-\delta_n(y_n,x_{n-1})\big\}\\
    &=A^{\star}_{n-1,n}\,e^{\epsilon_n}.
\end{align}

\emph{Reversibility.} If $\delta_n\equiv0$ then $A^{\mathrm{CREPE}}=A^{\star}$ and the exchange is a Metropolis--Hastings kernel for $\pi_n\otimes\pi_{n-1}$, hence reversible. Conversely, if $\epsilon_n=0$ for almost every configuration then $\delta_n$ is constant, $\Kf_{n|n-1}=e^{c}\Kfh_{n|n-1}$, and $c=0$ because both are probability kernels. If $\epsilon_n\not\equiv0$, the accepted moves are those of a Metropolis--Hastings kernel with a tilted ratio, which does not preserve $\pi_n\otimes\pi_{n-1}$.

\emph{Compounding.} When the exchange spans $K$ model steps and the substitution is made step by step along the simulated intermediate path, as CREPE does, the reversal error of the composite move is the sum of the per-step errors along that path,
\begin{equation}
    \delta_n=\sum_{k=1}^K\delta^{(k)}.
\end{equation}
These do not cancel in general, so $|\epsilon_n|$ grows with $K$, and $\min\{1,e^{\epsilon_n}A^{\star}\}$ is driven towards zero across a coarse interval.

\emph{ENCORE.} There is no forward proposal: by \cref{prop:ratio} the ENCORE ratio is exactly the ratio of two values of $w_{n-1|n}$, and it is the correct Metropolis--Hastings ratio for $\bpi_n\otimes\bpi_{n-1}$. Hence $\epsilon_n\equiv0$.
\end{proof}

\section{ENCORE: Derivations and Proofs}
\label{app:derivations}

Throughout, in an exchange between levels $n$ and $n-1$, $\bx_n$ is the upper replica, $\bx_{n-1}=(\by_n,x_{n-1})$ the lower one with $\by_n=\trunc_{n|n-1}(\bx_{n-1})$ its truncation, $y_{n-1}\sim\Lb_{n-1|n}(\cdot\mid x_n)$ the proposed extension and $\by_{n-1}=(\bx_n,y_{n-1})$, as in \cref{sec:ladder}. $\trunc_{n|n-1}:\Traj{n-1}\to\Traj{n}$ drops the last state. The backward proposal $\Lb_{n-1|n}$ is fixed, and $w_{n-1|n}$ is the backward weight of \cref{eq:weight-factor},
\begin{equation}
    w_{n-1|n}(x_{n-1},x_n)=\frac{G_{n-1}(x_{n-1})}{G_n(x_n)}\,\frac{\Kb_{n-1|n}(x_{n-1}\mid x_n)}{\Lb_{n-1|n}(x_{n-1}\mid x_n)}.
\end{equation}
For a law $\bpi_n$ on $\Traj{n}$ and a kernel $\Lb_{n-1|n}$, the product $\bpi_n\otimes\Lb_{n-1|n}$ is the law on $\Traj{n-1}$ of ``draw $\bx_n\sim\bpi_n$, then $x_{n-1}\sim\Lb_{n-1|n}(\dd x_{n-1}\mid x_n)$''. Primes denote independent copies. Where densities are assumed to exist they are denoted by the same symbol as the measure. \cref{app:general-kernels} removes the assumption.

\subsection{The Path Exchange}
\label{app:path-exchange}

The exchange compares two laws on $\Traj{n-1}$: the extended upper replica $\by_{n-1}$, which has law $\bpi_n\otimes\Lb_{n-1|n}$, and the lower replica $\bx_{n-1}$, which has law $\bpi_{n-1}$. The first lemma is the statement in \cref{sec:ladder} that their ratio is the backward weight.

\begin{lemma}[Size-biasing identity]
\label{lem:sb}
Let $\mathcal Z_n:=\E_{\bp_n}[G_n(X_n)]$ be the normalising constant of $\bpi_n$. Then
\begin{equation}
    \bpi_{n-1}(\dd\bx_{n-1})
    =\frac{\mathcal Z_n}{\mathcal Z_{n-1}}\;w_{n-1|n}(x_{n-1},x_n)\;(\bpi_n\otimes\Lb_{n-1|n})(\dd\bx_{n-1}).
    \label{eq:app-sb}
\end{equation}
Consequently, writing $W_{n-1|n}:=w_{n-1|n}(x_{n-1},x_n)$ for the weight of a draw from $\bpi_n\otimes\Lb_{n-1|n}$,
\begin{equation}
    \E[W_{n-1|n}]=\frac{\mathcal Z_{n-1}}{\mathcal Z_n}
    \qquad\text{and}\qquad
    \frac{\dd\bpi_{n-1}}{\dd(\bpi_n\otimes\Lb_{n-1|n})}(\bx_{n-1})=\frac{w_{n-1|n}(x_{n-1},x_n)}{\E[W_{n-1|n}]}.
    \label{eq:app-sb-rn}
\end{equation}
\end{lemma}

\begin{proof}
By the definition \cref{eq:trajectory-target} of $\bpi_{n-1}$ and the path law \cref{eq:path-law},
\begin{align*}
    \bpi_{n-1}(\dd\bx_{n-1})
    &=\frac{G_{n-1}(x_{n-1})}{\mathcal Z_{n-1}}\;\bp_{n-1}(\dd\bx_{n-1})\\
    &=\frac{G_{n-1}(x_{n-1})}{\mathcal Z_{n-1}}\;\bp_n(\dd\bx_n)\,\Kb_{n-1|n}(\dd x_{n-1}\mid x_n).
\end{align*}
Substitute $\bp_n=\mathcal Z_n\bpi_n/G_n$, which is \cref{eq:trajectory-target} at level $n$, and write $\Kb_{n-1|n}(\dd x_{n-1}\mid x_n)$ as its density with respect to $\Lb_{n-1|n}(\dd x_{n-1}\mid x_n)$, which exists because $\Kb_{n-1|n}(\cdot\mid x_n)\ll\Lb_{n-1|n}(\cdot\mid x_n)$:
\begin{align*}
    \bpi_{n-1}(\dd\bx_{n-1})
    &=\frac{G_{n-1}(x_{n-1})}{\mathcal Z_{n-1}}\;\frac{\mathcal Z_n\,\bpi_n(\dd\bx_n)}{G_n(x_n)}\;
      \frac{\Kb_{n-1|n}(x_{n-1}\mid x_n)}{\Lb_{n-1|n}(x_{n-1}\mid x_n)}\;\Lb_{n-1|n}(\dd x_{n-1}\mid x_n)\\
    &=\frac{\mathcal Z_n}{\mathcal Z_{n-1}}\;w_{n-1|n}(x_{n-1},x_n)\;\bpi_n(\dd\bx_n)\,\Lb_{n-1|n}(\dd x_{n-1}\mid x_n).
\end{align*}
Every factor of $\bp_n$ has cancelled, which is why no marginal appears in the weight. Integrating both sides over $\Traj{n-1}$ gives $1=(\mathcal Z_n/\mathcal Z_{n-1})\E[W_{n-1|n}]$, which is the first statement of \cref{eq:app-sb-rn}, and dividing \cref{eq:app-sb} by $\bpi_n\otimes\Lb_{n-1|n}$ gives the second.
\end{proof}

\begin{proposition}[Acceptance ratio]
\label{prop:ratio}
The Metropolis--Hastings ratio of the path exchange $(\bx_{n-1},\bx_n)\mapsto(\by_{n-1},\by_n)$ of \cref{sec:ladder} for the target $\bpi_n\otimes\bpi_{n-1}$ is
\begin{equation}
    A_{n-1,n}=\frac{w_{n-1|n}(y_{n-1},x_n)}{w_{n-1|n}(x_{n-1},y_n)},
    \label{eq:app-ratio}
\end{equation}
so that $\boldsymbol\alpha_{n-1,n}=\min\{1,A_{n-1,n}\}$ is \cref{eq:guided-acceptance}.
\end{proposition}

\begin{proof}
The move is a reversible-jump exchange \citep{green_reversible_1995}. The forward proposal draws $y_{n-1}\sim\Lb_{n-1|n}(\dd y_{n-1}\mid x_n)$ and maps
\begin{equation*}
    \big(\bx_n,\ \bx_{n-1}\big)=\big(\bx_n,\ (\by_n,x_{n-1})\big)\ \longmapsto\ \big(\by_n,\ (\bx_n,y_{n-1})\big)=\big(\by_n,\ \by_{n-1}\big);
\end{equation*}
the reverse proposal draws $x_{n-1}\sim\Lb_{n-1|n}(\dd x_{n-1}\mid y_n)$ and maps back. On the augmented space $\Traj{n}\times\Traj{n-1}\times\Xsp$ the dimension-matching map is a permutation of coordinates, so its Jacobian is one, and the ratio is the target density at the proposed pair times the reverse proposal density, over the target density at the current pair times the forward proposal density:
\begin{align}
    A_{n-1,n}
    &=\frac{\bpi_n(\by_n)\;\bpi_{n-1}(\bx_n,y_{n-1})\;\Lb_{n-1|n}(x_{n-1}\mid y_n)}
           {\bpi_n(\bx_n)\;\bpi_{n-1}(\by_n,x_{n-1})\;\Lb_{n-1|n}(y_{n-1}\mid x_n)}
    \label{eq:app-mh1}\\[0.4em]
    &=\frac{\bpi_{n-1}(\bx_n,y_{n-1})\;\big/\;\big[\bpi_n(\bx_n)\,\Lb_{n-1|n}(y_{n-1}\mid x_n)\big]}
           {\bpi_{n-1}(\by_n,x_{n-1})\;\big/\;\big[\bpi_n(\by_n)\,\Lb_{n-1|n}(x_{n-1}\mid y_n)\big]}
    \label{eq:app-mh2}\\[0.4em]
    &=\frac{(\mathcal Z_n/\mathcal Z_{n-1})\;w_{n-1|n}(y_{n-1},x_n)}{(\mathcal Z_n/\mathcal Z_{n-1})\;w_{n-1|n}(x_{n-1},y_n)}\\
    &=\frac{w_{n-1|n}(y_{n-1},x_n)}{w_{n-1|n}(x_{n-1},y_n)}.
\end{align}
The second line regroups the six factors into two densities of $\bpi_{n-1}$ with respect to $\bpi_n\otimes\Lb_{n-1|n}$, one at the proposed pair $(\bx_n,y_{n-1})$ and one at the current pair $(\by_n,x_{n-1})$. The third applies \cref{lem:sb} to each, and the constant cancels.
\end{proof}

\paragraph{Unguided exchange.}
\label{app:unguided-comm}
For $\Lb_{n-1|n}=\Kb_{n-1|n}$ the proposal correction in $w_{n-1|n}$ is one and the ratio reduces to
\begin{equation}
    A_{n-1,n}=\frac{G_n(y_n)\,G_{n-1}(y_{n-1})}{G_n(x_n)\,G_{n-1}(x_{n-1})},
\end{equation}
the potentials of the two proposed states over the potentials of the two current ones, with no model density.

\paragraph{Kernels without densities.}
\label{app:general-kernels}
The proof used the densities of $\Kb$ and $\Lb$ only through the Radon--Nikodym derivative $\dd\Kb_{n-1|n}(\cdot\mid x_n)/\dd\Lb_{n-1|n}(\cdot\mid x_n)$, which exists under the assumption $\Kb_{n-1|n}(\cdot\mid x_n)\ll\Lb_{n-1|n}(\cdot\mid x_n)$ of \cref{sec:exact}, so $w_{n-1|n}$, \cref{lem:sb} and \cref{prop:ratio} are defined for kernels without densities. For $\Lb_{n-1|n}=\Kb_{n-1|n}$ the derivative equals one almost surely, so the unguided ratio holds even when $\Kb_{n-1|n}$ is deterministic, as for a distilled or ODE sampler.

\paragraph{Path-dependent potentials and proposals.}
\label{app:smc-connection}
Nothing above requires the potentials or proposals to depend on the last state only. With $G_n:\Traj{n}\to[0,\infty)$ and $\Lb_{n-1|n}(\dd z\mid\bx_n)$, the weight becomes
\begin{equation}
    w_{n-1|n}(\bx_n,y_{n-1})=\frac{G_{n-1}(\bx_n,y_{n-1})}{G_n(\bx_n)}\,\frac{\Kb_{n-1|n}(y_{n-1}\mid x_n)}{\Lb_{n-1|n}(y_{n-1}\mid\bx_n)},
\end{equation}
and \cref{lem:sb}, \cref{prop:ratio} and \cref{thm:exact}(a),(b) hold verbatim with $G_n(\bx_n)$ in place of $G_n(x_n)$. This covers rewards on whole trajectories, at the cost that the exchange then reads the stored history rather than its last state. The same weight is the incremental weight of an SMC sampler on the trajectory ladder: extending a particle $\bx_n$ with $\Lb_{n-1|n}$ and weighting by $\dd\bpi_{n-1}/\dd(\bpi_n\otimes\Lb_{n-1|n})$ gives $w_{n-1|n}(y_{n-1},x_n)$ up to a constant by \cref{lem:sb}. SMC uses $w_{n-1|n}$ to reweight a population. ENCORE uses the ratio of two evaluations of it.

\subsection{Proof of \texorpdfstring{\cref{thm:exact}}{Theorem 3.1}}
\label{app:thm-exact}

The three parts are proved in turn. Part (a) is invariance and the rejection-rate identity, part (b) is uniform ergodicity under bounded weights, and part (c) is the characterisation of the maximal restart rate.

\subsubsection{Part (a): invariance and the rejection rate}
\label{app:product-invariance}
\label{app:rejection}

\begin{proof}[Proof of invariance]
The exchange at $(n-1,n)$ is a Metropolis--Hastings kernel with the ratio of \cref{prop:ratio}, hence reversible with respect to $\bpi_n\otimes\bpi_{n-1}$. Extended by the identity on the other levels, it leaves $\prodpi$ invariant. Since $G_N\equiv1$, $\bpi_N=\eta$, so redrawing $x_N\sim\eta$ leaves the level-$N$ marginal invariant, and any $\bpi_n$-invariant local move leaves $\prodpi$ invariant. Invariance is preserved under composition, so one iteration of \cref{alg:encore} leaves $\prodpi$ invariant. The $\bx_0$-marginal of $\prodpi$ is $\bpi_0\propto G_0(x_0)\,\bp_0$, whose $x_0$-marginal is
\begin{equation*}
    \int_{\bx_0:\ \text{last state }x_0}\bpi_0(\dd\bx_0)\ \propto\ G_0(x_0)\int_{\bx_0:\ \text{last state }x_0}\bp_0(\dd\bx_0)=G_0(x_0)\,p_0(\dd x_0),
\end{equation*}
which is $\pi_0$. Hence the recorded state has stationary law $\pi_0$.
\end{proof}

The rejection rate is computed in three steps. The path exchange is a swap between two laws on $\Traj{n-1}$. The rejection rate of a swap is a total variation. And the total variation on trajectory space equals the total variation of the projections onto the last two states, which are the two laws in \cref{eq:rejection-rate}. Fix $n$ and abbreviate
\begin{equation}
    p:=\bpi_n\otimes\Lb_{n-1|n}\quad\text{(law of the extended upper replica)},
    \qquad
    q:=\bpi_{n-1}\quad\text{(law of the lower replica)},
    \label{eq:app-pq}
\end{equation}
both on $\Traj{n-1}$. Let $W_{n-1|n}:=w_{n-1|n}(x_{n-1},x_n)$ under $p$, let $W'_{n-1|n}$ be an independent copy, and let $\bar W_{n-1|n}:=W_{n-1|n}/\E[W_{n-1|n}]$, so that $\dd q/\dd p=\bar W_{n-1|n}$ by \cref{lem:sb}.

\begin{lemma}[Rejection rate of a swap]
\label{lem:swap-tv}
Let $p,q$ be laws on a common space, $u\sim p$ and $v\sim q$ independent, and consider the swap $(u,v)\mapsto(v,u)$ accepted with the Metropolis--Hastings probability $\min\{1,(\dd q/\dd p)(u)/(\dd q/\dd p)(v)\}$ for the target $p\otimes q$. Its rejection rate is
\begin{equation}
    r(p,q)=\TV(p\otimes q,\,q\otimes p)=\tfrac12\,\E\big|\bar W(u)-\bar W(u')\big|,
    \qquad \bar W:=\frac{\dd q}{\dd p},\quad u,u'\sim p\ \text{i.i.d.},
    \label{eq:app-swap}
\end{equation}
and it satisfies
\begin{equation}
    \TV(p,q)\ \le\ r(p,q)\ \le\ \min\Big\{2\,\TV(p,q),\ \sqrt{\Var(\bar W)/2}\Big\}.
    \label{eq:app-swap-bounds}
\end{equation}
\end{lemma}

\begin{proof}
\emph{The rejection rate is a total variation.} The acceptance probability is an expectation of a minimum, which becomes an integral of a minimum of two products of densities:
\begin{align*}
    1-r(p,q)
    &=\E\Big[\min\Big\{1,\frac{(\dd q/\dd p)(u)}{(\dd q/\dd p)(v)}\Big\}\Big]\\
    &=\int\!\!\int p(u)\,q(v)\,\min\Big\{1,\frac{q(u)\,p(v)}{p(u)\,q(v)}\Big\}\,\dd u\,\dd v\\
    &=\int\!\!\int\min\big\{p(u)q(v),\ q(u)p(v)\big\}\,\dd u\,\dd v .
\end{align*}
For probability densities $f,g$ one has $\int\min\{f,g\}=1-\TV(f,g)$. Taking $f=p\otimes q$ and $g=q\otimes p$ gives $r(p,q)=\TV(p\otimes q,q\otimes p)$.

\emph{The weight form.} Substitute $q=\bar Wp$ in the last integral and rewrite it as an integral over two independent draws from $p$:
\begin{align*}
    \TV(p\otimes q,\,q\otimes p)
    &=\tfrac12\int\!\!\int\big|p(u)q(u')-q(u)p(u')\big|\,\dd u\,\dd u'\\
    &=\tfrac12\int\!\!\int\big|\bar W(u')-\bar W(u)\big|\,p(u)\,p(u')\,\dd u\,\dd u'
    =\tfrac12\,\E\big|\bar W(u)-\bar W(u')\big|.
\end{align*}

\emph{Lower bound.} Since $\E\bar W=1$, $\TV(p,q)=\tfrac12\E|\bar W(u)-1|$. Conditioning on $\bar W(u)$ and applying Jensen's inequality to the convex map $s\mapsto|\bar W(u)-s|$,
\begin{equation*}
    \E\big[\,|\bar W(u)-\bar W(u')|\ \big|\ \bar W(u)\big]\ \ge\ \big|\bar W(u)-\E\bar W(u')\big|=|\bar W(u)-1|,
\end{equation*}
and taking expectations gives $r(p,q)\ge\TV(p,q)$.

\emph{Upper bounds.} The triangle inequality $|\bar W(u)-\bar W(u')|\le|\bar W(u)-1|+|1-\bar W(u')|$ gives $r(p,q)\le2\,\TV(p,q)$. By Cauchy--Schwarz and independence,
\begin{equation*}
    \E|\bar W(u)-\bar W(u')|\ \le\ \sqrt{\E\big[(\bar W(u)-\bar W(u'))^2\big]}=\sqrt{2\Var(\bar W)},
\end{equation*}
hence $r(p,q)\le\sqrt{\Var(\bar W)/2}$.
\end{proof}

\begin{proof}[Proof of the rejection rate \cref{eq:rejection-rate}]
\emph{The path exchange is a swap between $p$ and $q$.} Under $\prodpi$ the replicas at levels $n$ and $n-1$ are independent, so the proposal $u:=\by_{n-1}=(\bx_n,y_{n-1})$ has law $p$ and the lower replica $v:=\bx_{n-1}$ has law $q$, independently. By \cref{prop:ratio} and \cref{lem:sb} the acceptance ratio is
\begin{equation*}
    A_{n-1,n}=\frac{w_{n-1|n}(u)}{w_{n-1|n}(v)}=\frac{\bar W_{n-1|n}(u)}{\bar W_{n-1|n}(v)}=\frac{(\dd q/\dd p)(u)}{(\dd q/\dd p)(v)},
\end{equation*}
which is the acceptance ratio of the swap in \cref{lem:swap-tv}. Hence
\begin{equation}
    r_{n-1,n}=r(p,q)=\tfrac12\,\E|\bar W_{n-1|n}-\bar W'_{n-1|n}|,
    \label{eq:app-rej-traj}
\end{equation}
with $\bar W_{n-1|n}$ the normalised backward weight, and the bounds \cref{eq:app-swap-bounds} apply. $\Var(\bar W_{n-1|n})$ is the squared coefficient of variation of the sequential control weight $w_{n-1|n}$ under $p$, so a level at which SMC would lose particle diversity is a level at which ENCORE rejects.

\emph{Projection to the state space.} Let $\mathrm{pr}$ map a trajectory to its last two states, $\mathrm{pr}(\bx_{n-1})=(x_n,x_{n-1})$. By \cref{lem:sb}, $\dd q/\dd p=\bar W_{n-1|n}$ depends on the trajectory only through $\mathrm{pr}$, say $\bar W_{n-1|n}=h\circ\mathrm{pr}$. Then the projected laws $p_{\mathrm{pr}}:=p\circ\mathrm{pr}^{-1}$ and $q_{\mathrm{pr}}:=q\circ\mathrm{pr}^{-1}$ have the same density: for every measurable $B$,
\begin{equation*}
    q_{\mathrm{pr}}(B)=q(\mathrm{pr}^{-1}B)=\int_{\mathrm{pr}^{-1}B}h\circ\mathrm{pr}\,\dd p=\int_Bh\,\dd p_{\mathrm{pr}},
\end{equation*}
so $\dd q_{\mathrm{pr}}/\dd p_{\mathrm{pr}}=h$. Since $r(p,q)=\tfrac12\E|\bar W_{n-1|n}(u)-\bar W_{n-1|n}(u')|$ depends only on the law of $\bar W_{n-1|n}$ under $p$, which is the law of $h$ under $p_{\mathrm{pr}}$, we have $r(p,q)=r(p_{\mathrm{pr}},q_{\mathrm{pr}})$.

\emph{The projections.} Under $p$, the last state $x_n$ of $\bx_n\sim\bpi_n$ has law $\pi_n$ and $x_{n-1}\sim\Lb_{n-1|n}(\cdot\mid x_n)$, so $p_{\mathrm{pr}}=\pi_n\otimes\Lb_{n-1|n}$. Under $q$, by \cref{eq:trajectory-target}, \cref{eq:path-law} and the time reversal \cref{eq:app-reversal},
\begin{align*}
    q_{\mathrm{pr}}(\dd x_n,\dd x_{n-1})
    &\propto G_{n-1}(x_{n-1})\int_{\bx_n:\ \text{last state }x_n}\bp_n(\dd\bx_n)\,\Kb_{n-1|n}(\dd x_{n-1}\mid x_n)\\
    &=G_{n-1}(x_{n-1})\,p_n(\dd x_n)\,\Kb_{n-1|n}(\dd x_{n-1}\mid x_n)\\
    &=G_{n-1}(x_{n-1})\,p_{n-1}(\dd x_{n-1})\,\Kf_{n|n-1}(\dd x_n\mid x_{n-1}),
\end{align*}
so $q_{\mathrm{pr}}=\pi_{n-1}\otimes\Kf_{n|n-1}$. Therefore
\begin{equation*}
    r_{n-1,n}=r\big(\pi_{n-1}\otimes\Kf_{n|n-1},\ \pi_n\otimes\Lb_{n-1|n}\big),
\end{equation*}
which is \cref{eq:rejection-rate}, and by \cref{eq:apt-rejection} this is the rejection rate of \cref{alg:ne-re} with $\Lf_{n|n-1}=\Kf_{n|n-1}$.
\end{proof}

\subsubsection{Part (b): uniform ergodicity}
\label{app:ergodicity}

Assume $0<\underline w_{n-1|n}\le w_{n-1|n}\le\overline w_{n-1|n}$ for every $n$ and set
\begin{equation}
    c_{n-1,n}:=\frac{\underline w_{n-1|n}}{\overline w_{n-1|n}}\in(0,1].
\end{equation}
Since $A_{n-1,n}$ is a ratio of two values of $w_{n-1|n}$, every exchange is accepted with probability at least $c_{n-1,n}$, whatever the state of the ladder. The proof is a Doeblin minorisation of the whole ladder over a window of iterations long enough for every replica to be discarded and replaced.

\begin{proof}
Realise every acceptance as $\mathbf 1\{U\le\boldsymbol\alpha_{n-1,n}\}$ with independent uniforms $U$, and every draw from $\eta$, from a proposal $\Lb$ and from a local move as a measurable function of the current state and of independent noise. Write $\bX=(\bx_N,\ldots,\bx_0)$ for the ladder and $\mathcal K$ for the one-iteration kernel of \cref{alg:encore}. A minorisation of a single level does not suffice, because the future of one replica depends on the others, so we minorise the whole chain.

\emph{Step 1: what happens when every exchange is accepted.}
Under the even--odd schedule the pair $(n-1,n)$ is attempted at the iterations $m\equiv n\pmod 2$, so a replica at a level $0<n<N$ takes part in exactly one attempted exchange per iteration. Suppose every attempted exchange in a window of iterations is accepted. A replica that moves down from level $n$ to $n-1$ at iteration $m$ is in the pair $(n-2,n-1)$ at iteration $m+1$, since $n-1\equiv m+1\pmod 2$, and moves down again. A replica that moves up keeps moving up in the same way. So each replica descends monotonically to level $0$, idles there for one iteration, ascends monotonically to level $N$, and is discarded at the start of the next iteration, when the reference level is redrawn. The longest such journey, for a replica that has just entered at level $N$, is
\begin{equation*}
    \underbrace{N}_{\text{descend}}+\underbrace{1}_{\text{idle at }0}+\underbrace{N}_{\text{ascend}}+\underbrace{1}_{\text{discard}}=2N+2\ \text{iterations},
\end{equation*}
so a replica present at the start of the window is discarded within $W:=2N+2$ iterations. After $W$ iterations every level holds a replica that entered at the reference level inside the window, extended by proposal draws and truncated along the way.

\emph{Step 2: the all-accept event is independent of the states.}
Let $E$ be the event that every uniform used in the window satisfies $U\le c_{n-1,n}$ for its pair. Since $\{U\le c_{n-1,n}\}\subseteq\{U\le\boldsymbol\alpha_{n-1,n}\}$ whatever the states, every attempted exchange is accepted on $E$. Since $E$ is an event in the uniforms alone, it is independent of the initial ladder $\bX$ and of the reference, proposal and local-move noise. Each pair is attempted every second iteration, hence at most $N+1$ times in the window, so
\begin{equation}
    \Pr(E)\ \ge\ \epsilon:=\prod_{n=1}^{N}c_{n-1,n}^{\,N+1}>0 .
    \label{eq:app-epsilon}
\end{equation}

\emph{Step 3: minorisation.}
On $E$ the evolution over the window is the all-accept dynamics of Step 1. By Step 1 the ladder at the end of the window is a function $F(\text{noise})$ of the noise drawn inside the window alone: every state that depended on $\bX$ has been discarded, and the local moves only alter states that are themselves functions of that noise. Let $\nu$ be the law of $F(\text{noise})$. It does not depend on $\bX$. For every measurable $B$,
\begin{equation}
    \mathcal K^{W}(\bX,B)\ \ge\ \Pr\big(E\cap\{F(\text{noise})\in B\}\big)=\Pr(E)\,\nu(B)\ \ge\ \epsilon\,\nu(B),
    \label{eq:app-doeblin}
\end{equation}
using the independence of $E$ from the noise. This is a Doeblin condition for $\mathcal K^{W}$, uniform in $\bX$.

\emph{Step 4: conclusion.}
By the coupling argument for Doeblin chains \citep[Thm.~16.2.4]{meyn_markov_2009}, for every $\bX$ and $k$,
\begin{equation}
    \big\|\mathcal K^{k}(\bX,\cdot)-\prodpi\big\|_{\TV}\ \le\ (1-\epsilon)^{\lfloor k/W\rfloor},
    \label{eq:app-rate}
\end{equation}
so the chain is uniformly ergodic, and the law of the recorded state $x_0$ converges to $\pi_0$ at this rate.
\end{proof}

\begin{remark}
The hypothesis holds for bounded rewards, $|r|\le C$ with $G_n=\exp\{\kappa_nr\}$, $0\le\kappa_n\le\kappa$, and $\Lb=\Kb$, since then $e^{-2\kappa C}\le w_{n-1|n}\le e^{2\kappa C}$. It fails for unbounded tilts such as the Boltzmann reweighting, for which a drift condition on $\log G_n$ would be required.
\end{remark}

\subsubsection{Part (c): the maximal restart rate and the Doob relations}
\label{app:doob-max}

\begin{proof}
By \cref{eq:roundtrip}, $\tau\le\tfrac12$ with equality if and only if $r_{n-1,n}=0$ for every $n$, since $s\mapsto s/(1-s)$ is non-negative on $[0,1)$ and vanishes only at $s=0$. By \cref{lem:swap-tv}, $r_{n-1,n}=0$ if and only if the two laws in \cref{eq:rejection-rate} coincide,
\begin{equation}
    \pi_{n-1}\otimes\Kf_{n|n-1}=\pi_n\otimes\Lb_{n-1|n}\qquad\text{as laws of }(x_n,x_{n-1}).
    \label{eq:app-doob-equality}
\end{equation}
Two laws on a product agree if and only if their $x_n$-marginals agree and their conditionals of $x_{n-1}$ given $x_n$ agree, so we compute both sides.

\emph{The two laws.} The right side has density proportional to $G_n(x_n)\,p_n(x_n)\,\Lb_{n-1|n}(x_{n-1}\mid x_n)$. The left side, by the time reversal \cref{eq:app-reversal}, has density proportional to
\begin{equation*}
    G_{n-1}(x_{n-1})\,p_{n-1}(x_{n-1})\,\Kf_{n|n-1}(x_n\mid x_{n-1})=G_{n-1}(x_{n-1})\,p_n(x_n)\,\Kb_{n-1|n}(x_{n-1}\mid x_n).
\end{equation*}

\emph{Marginals.} Integrating out $x_{n-1}$,
\begin{align*}
    \text{right side:}&\qquad \pi_n(\dd x_n)\propto G_n(x_n)\,p_n(\dd x_n),\\
    \text{left side:}&\qquad \propto p_n(\dd x_n)\int G_{n-1}(y)\,\Kb_{n-1|n}(\dd y\mid x_n)=\big(\Kb_{n-1|n}G_{n-1}\big)(x_n)\,p_n(\dd x_n).
\end{align*}
The marginals agree if and only if $G_n\propto\Kb_{n-1|n}G_{n-1}$, the potential relation of \cref{eq:app-doob-relations}.

\emph{Conditionals.} Given $x_n$, the right side draws $x_{n-1}\sim\Lb_{n-1|n}(\cdot\mid x_n)$, and the left side draws $x_{n-1}$ from
\begin{equation*}
    \frac{G_{n-1}(y)\,\Kb_{n-1|n}(\dd y\mid x_n)}{(\Kb_{n-1|n}G_{n-1})(x_n)}.
\end{equation*}
They agree if and only if $\Lb_{n-1|n}\propto G_{n-1}\Kb_{n-1|n}$, the proposal relation of \cref{eq:app-doob-relations}.

Hence $\tau=\tfrac12$ if and only if the two Doob relations hold at every level $n$. By \cref{app:bg-doob} this is $G_n\propto G^*_n$ and $\Lb_{n-1|n}=\Lb^*_{n-1|n}$ for every $n$, with the top-level caveat noted there. In that case every acceptance ratio equals one, each reference draw reaches level $0$ through accepted exchanges, and its trajectory is a draw from $\eta\otimes\Lb^*_{N-1|N}\otimes\cdots\otimes\Lb^*_{0|1}$, the Doob-transformed chain of \cref{eq:app-doob-path}, whose terminal marginal is $\pi_0$.
\end{proof}

\subsection{Cost and Memory}
\label{app:cost}

\paragraph{Kernel simulations.} Each iteration of \cref{alg:encore} simulates one proposal $\Lb_{n-1|n}$ for each of the $N/2$ pairs active under the even--odd schedule, in parallel, and two consecutive iterations extend every level once. When the levels coincide with the model steps, two iterations therefore cost the same $N$ kernel simulations as one pass of the sampler, and per sample the cost matches sequential control (\cref{tab:landscape}). When a level spans several model steps, $\Lb_{n-1|n}$ composes them. The intermediate states are used once in the acceptance ratio and discarded. The optional local move costs one further proposal per level.

\paragraph{Memory.} The upward move replaces the lower replica by its truncation $\trunc_{n|n-1}(\bx_{n-1})$, so the initial segment must be stored. The replica at level $n$ holds its states at the levels $N,\ldots,n$ it has passed through, $N-n+1$ states, and the ladder holds
\begin{equation}
    \sum_{n=0}^{N}(N-n+1)=\frac{(N+1)(N+2)}{2}=O(N^2)
\end{equation}
states in all, $O(N^2d)$ against $O(Nd)$ for state-space replica exchange. Only states at the levels are stored, not at model steps in between, so a coarse ladder is cheap in memory as well as in communication. For a deterministic sampler with a tractable inverse the stored prefix can be recomputed from the last state by inverting the transitions, which trades the $O(N^2d)$ memory for extra model evaluations.

\section{Diffusion Models: Communication Barriers}
\label{app:barriers}

This appendix proves \cref{thm:barrier} and develops its consequences in the order of \cref{sec:diffusion}: the guided proposal against the Doob proposal, schedule tuning, the optimal spacing of an unguided ladder, the corresponding SMC barrier, and a closed-form example.

\subsection{Setting and Assumptions}
\label{app:barrier-setting}

We look at the exchange between the levels at times $t=t_n$ and $t'=t_{n-1}=t-\Delta t_n$, and index the objects attached to it by $n$. The same computation applies to every level. Let $X_t\sim\pi_t$ be the last state of the upper replica and $Y_{t'}$ the proposed extension produced by the step \cref{eq:euler-step} with noise $\xi\sim N(0,I)$ independent of $X_t$. The log-weight of \cref{eq:logweight} at level $n$ is
\begin{equation}
    \ell_n:=\log w_{t'|t}(X_t,\xi)
    =\log\frac{G_{t'}(Y_{t'})}{G_t(X_t)}+\mathbf 1_{\mathrm{guided}}\log\frac{\Kb_{t'|t}(Y_{t'}\mid X_t)}{\Lb_{t'|t}(Y_{t'}\mid X_t)},
    \label{eq:app-logweight}
\end{equation}
$\ell'_n$ is an independent copy, and by \cref{eq:app-rej-traj} the rejection rate of the exchange is
\begin{equation}
    r_{n-1,n}=\frac{\E|e^{\ell_n}-e^{\ell'_n}|}{2\,\E e^{\ell_n}}.
    \label{eq:app-rej-weight}
\end{equation}
We write $m_{1,n}:=\E|\ell_n|$ and $m_{4,n}:=\E\ell_n^4$ for the moments that control the expansion of \cref{eq:app-rej-weight}. Recall $g_t:=\sigma_t^{\top}\nabla\log G_t$ and $M_t:=\sigma_t^{\top}\nabla^2\log G_t\,\sigma_t$ from \cref{sec:diffusion}, and define the Doob residual
\begin{equation}
    \psi_t:=\frac{\gen_tG_t}{G_t}
    =-\partial_t\log G_t+b_t\cdot\nabla\log G_t+\tfrac12\operatorname{tr}M_t+\tfrac12\|g_t\|^2,
    \label{eq:app-psi}
\end{equation}
the second equality by the identity $G^{-1}\gen_tG=\gen_t\log G+\tfrac12\|\sigma_t^{\top}\nabla\log G\|^2$, which follows from $\nabla^2G/G=\nabla^2\log G+\nabla\log G\,\nabla\log G^{\top}$. With this notation \cref{eq:Psi-main} reads
\begin{equation}
    \Psi_t(x,\xi)=\psi_t(x)+\tfrac12\big(\xi^{\top}M_t(x)\xi-\operatorname{tr}M_t(x)\big),
    \qquad
    \E\big[\Psi_t(X_t,\xi)\mid X_t\big]=\psi_t(X_t),
    \label{eq:app-Psi}
\end{equation}
the conditional mean because $\E[\xi^{\top}M\xi]=\operatorname{tr}M$ for $\xi\sim N(0,I)$.

\paragraph{Assumptions.}
\begin{enumerate}
    \item[(A1)] $\sigma_t$ is invertible for every $t$, and $t\mapsto\sigma_t$ and $t\mapsto\sigma_t^{-1}$ are continuous on $[0,1]$.
    \item[(A2)] $(t,x)\mapsto b_t(x)$ is $C^{0,3}$ and $(t,x)\mapsto\log G_t(x)$ is $C^{1,4}$, with all these derivatives bounded on $[0,1]\times\Xsp$.
    \item[(A3)] The sampler and guided steps are the two cases of \cref{eq:euler-step}, with $\xi\sim N(0,I)$ independent of $X_t$.
\end{enumerate}
(A2) makes $\nabla\log G_t$, $\nabla^2\log G_t$ and $\partial_t\log G_t$, hence $g_t$, $M_t$ and $\psi_t$, bounded, and every remainder below uniform in $x$. Under (A1)--(A2) the diffusion is uniformly elliptic with bounded smooth drift, so its transition densities are everywhere positive, and hence so are $p_t$ and $\pi_t\propto G_tp_t$ for every $t<1$. Replacing the Euler--Maruyama steps by exact transitions of the two SDEs changes only the $O(\Delta t_n^{3/2})$ remainders, by the same expansions applied through the Itô formula.

\begin{remark}[On the assumptions]
(A2) is a convenience, not a description of practice: it is what makes the Taylor expansions uniform, and it excludes settings of interest. A Boltzmann tilt $\log G_0=-\beta U$ with $U$ of unbounded gradient violates it, as does constrained generation with $G_0=\mathbf 1_{\mathcal C}$, for which $\log G_0$ is not finite. The expansions remain valid in the first case on any region where the stated derivatives are bounded, and localisation arguments extend them under polynomial growth with Gaussian tails for $\pi_t$, at the price of a longer proof. The second case is outside the scope of an expansion in $\Delta t_n$ altogether, since the weight is then an indicator. We state \cref{thm:barrier} under (A2) and read it as a guide to the behaviour of the sampler rather than as a covering theorem for every application.
\end{remark}

\subsection{Three Lemmas}
\label{app:barrier-lemmas}

The first lemma turns the rejection rate \cref{eq:app-rej-weight} into a mean absolute difference of log-weights, the second evaluates that mean for Gaussian log-weights, and the third is the pair of Taylor expansions behind \cref{eq:guided-logweight}.

\begin{lemma}[Small-weight expansion]
\label{lem:smallweight}
Suppose $\sup_n\E e^{2|\ell_n|}\le C_0$. Then there is a constant $C$ depending only on $C_0$ such that, for every $n$,
\begin{equation}
    \Big|\,r_{n-1,n}-\tfrac12\,\E|\ell_n-\ell'_n|\,\Big|\ \le\ C\big(m_{4,n}^{1/2}+m_{1,n}^2\big).
    \label{eq:app-smallweight}
\end{equation}
\Cref{lem:moments} below verifies the hypothesis and bounds $m_{1,n}$ and $m_{4,n}$ by powers of $\Delta t_n$ uniformly in $n$, so that the right side is $O(\Delta t_n)$ in the unguided case and $O(\Delta t_n^2)$ in the guided case.
\end{lemma}

\begin{proof}
Fix $n$ and drop the index. The proof compares the numerator and the denominator of \cref{eq:app-rej-weight} with their first-order expansions.

\emph{Step 1: expand the exponential.} Write $e^{u}=1+u+\rho(u)$ with $|\rho(u)|\le\tfrac12u^2e^{|u|}$. By Cauchy--Schwarz and the hypothesis,
\begin{equation*}
    \E|\rho(\ell)|\ \le\ \tfrac12\big(\E\ell^4\big)^{1/2}\big(\E e^{2|\ell|}\big)^{1/2}\ \le\ \tfrac12C_0^{1/2}m_4^{1/2}.
\end{equation*}

\emph{Step 2: the numerator.} With $\mathrm N:=\E|e^{\ell}-e^{\ell'}|$ and $e^{\ell}-e^{\ell'}=(\ell-\ell')+(\rho(\ell)-\rho(\ell'))$, the triangle inequality gives
\begin{equation*}
    \mathrm N=\E|\ell-\ell'|+\rho_{\mathrm N},\qquad |\rho_{\mathrm N}|\le\E|\rho(\ell)|+\E|\rho(\ell')|\le C_0^{1/2}m_4^{1/2}.
\end{equation*}

\emph{Step 3: the denominator.} With $\mathrm D:=\E e^{\ell}=1+\E\ell+\E\rho(\ell)$,
\begin{equation*}
    \mathrm D=1+\rho_{\mathrm D},\qquad |\rho_{\mathrm D}|\le|\E\ell|+\E|\rho(\ell)|\le m_1+\tfrac12C_0^{1/2}m_4^{1/2}.
\end{equation*}

\emph{Step 4: combine.} Since $r=\mathrm N/(2\mathrm D)$,
\begin{equation*}
    r-\tfrac12\E|\ell-\ell'|
    =\frac{\mathrm N-\E|\ell-\ell'|}{2\mathrm D}+\tfrac12\E|\ell-\ell'|\Big(\frac1{\mathrm D}-1\Big)
    =\frac{\rho_{\mathrm N}}{2\mathrm D}-\frac{\E|\ell-\ell'|}{2}\cdot\frac{\rho_{\mathrm D}}{\mathrm D}.
\end{equation*}
Suppose first that $|\rho_{\mathrm D}|\le\tfrac12$, so that $\mathrm D\ge\tfrac12$. Using $\E|\ell-\ell'|\le2m_1$ and the bounds of Steps 2 and 3,
\begin{equation*}
    \Big|r-\tfrac12\E|\ell-\ell'|\Big|
    \ \le\ |\rho_{\mathrm N}|+2m_1|\rho_{\mathrm D}|
    \ \le\ C_0^{1/2}m_4^{1/2}+2m_1^2+C_0^{1/2}m_1m_4^{1/2}.
\end{equation*}
The last term is at most $\tfrac12C_0^{1/2}(m_1^2+m_4)$ by the inequality $ab\le\tfrac12(a^2+b^2)$, and $m_4\le m_4^{1/2}$ when $m_4\le1$, which gives \cref{eq:app-smallweight}. If $m_4>1$ or $|\rho_{\mathrm D}|>\tfrac12$, then $m_4^{1/2}+m_1^2$ is bounded below by a constant depending only on $C_0$, while the left side is at most $1+m_1\le2+m_1^2$, so \cref{eq:app-smallweight} holds with a larger $C$. The constant depends on $C_0$ only, hence not on $n$.
\end{proof}

\begin{lemma}[Gaussian mean absolute difference]
\label{lem:gaussdiff}
Let $g:\Xsp\to\mathbb R^d$ be bounded, $X,X'$ i.i.d., and $\xi,\xi'\sim N(0,I)$ independent of $(X,X')$ and of each other. Then
\begin{equation}
    \E\big|g(X)^{\top}\xi-g(X')^{\top}\xi'\big|=\sqrt{\tfrac{2}{\pi}}\;\E\sqrt{\|g(X)\|^2+\|g(X')\|^2}.
    \label{eq:app-gaussdiff}
\end{equation}
\end{lemma}

\begin{proof}
Conditionally on $(X,X')$, $g(X)^{\top}\xi-g(X')^{\top}\xi'$ is Gaussian with mean zero and variance $\|g(X)\|^2+\|g(X')\|^2$, and $\E|N(0,s^2)|=s\sqrt{2/\pi}$. Take expectations over $(X,X')$.
\end{proof}

\begin{lemma}[One-step expansions]
\label{lem:onestep}
Under (A1)--(A3), with all functions evaluated at $X_t$, the following hold uniformly in $X_t$.
\begin{enumerate}
    \item[(a)] For the sampler step ($\mathbf 1_{\mathrm{guided}}=0$),
    \begin{equation}
        \log\frac{G_{t'}(Y_{t'})}{G_t(X_t)}
        =\sqrt{\Delta t_n}\,g_t^{\top}\xi+\Delta t_n\Big[\Psi_t(X_t,\xi)-\tfrac12\|g_t\|^2\Big]+O\big(\Delta t_n^{3/2}(1+\|\xi\|^3)\big).
        \label{eq:app-onestep-a}
    \end{equation}
    \item[(b)] For the guided step ($\mathbf 1_{\mathrm{guided}}=1$) the same holds with the bracket $\Psi_t(X_t,\xi)+\tfrac12\|g_t\|^2$.
    \item[(c)] For any $f$ that is $C^{1,4}$ with bounded derivatives, the sampler step satisfies
    \begin{equation}
        \E\big[f_{t'}(Y_{t'})\mid X_t=x\big]=f_t(x)+\Delta t_n\,(\gen_tf_t)(x)+O(\Delta t_n^2).
        \label{eq:app-onestep-c}
    \end{equation}
    \item[(d)] At the point $Y_{t'}$ produced by the guided step, the proposal correction is exactly
    \begin{equation}
        \log\frac{\Kb_{t'|t}(Y_{t'}\mid X_t)}{\Lb_{t'|t}(Y_{t'}\mid X_t)}=-\sqrt{\Delta t_n}\,g_t^{\top}\xi-\tfrac{\Delta t_n}{2}\|g_t\|^2 .
        \label{eq:app-onestep-d}
    \end{equation}
\end{enumerate}
\end{lemma}

\begin{proof}
(a) Let $D_n:=Y_{t'}-X_t=b_t\Delta t_n+\sigma_t\sqrt{\Delta t_n}\,\xi$, so $\|D_n\|=O(\sqrt{\Delta t_n}(1+\|\xi\|))$. Taylor expand $\log G_{t'}(X_t+D)$ around $(t,X_t)$ to first order in $t$ and second order in $x$, with all derivatives evaluated at $(t,X_t)$:
\begin{align*}
    \log\frac{G_{t'}(Y_{t'})}{G_t(X_t)}
    &=-\Delta t_n\,\partial_t\log G_t
      +\nabla\log G_t\cdot D_n
      +\tfrac12D_n^{\top}\nabla^2\log G_t\,D_n
      +O\big(\Delta t_n\|D_n\|+\|D_n\|^3+\Delta t_n^2\big).
\end{align*}
Substitute $D_n$ into the linear term,
\begin{align*}
    \nabla\log G_t\cdot D_n
    &=\sqrt{\Delta t_n}\,\nabla\log G_t^{\top}\sigma_t\xi+\Delta t_n\,b_t\cdot\nabla\log G_t\\
    &=\sqrt{\Delta t_n}\,g_t^{\top}\xi+\Delta t_n\,b_t\cdot\nabla\log G_t,
\end{align*}
and into the quadratic term,
\begin{align*}
    \tfrac12D_n^{\top}\nabla^2\log G_t\,D_n
    &=\tfrac12\Delta t_n\,\xi^{\top}\sigma_t^{\top}\nabla^2\log G_t\,\sigma_t\xi+O\big(\Delta t_n^{3/2}(1+\|\xi\|)\big)\\
    &=\tfrac12\Delta t_n\,\xi^{\top}M_t\xi+O\big(\Delta t_n^{3/2}(1+\|\xi\|)\big).
\end{align*}
Collecting the terms of order $\Delta t_n$ and using \cref{eq:app-psi} to rewrite $-\partial_t\log G_t+b_t\cdot\nabla\log G_t$ as $\psi_t-\tfrac12\operatorname{tr}M_t-\tfrac12\|g_t\|^2$, the bracket is
\begin{align*}
    -\partial_t\log G_t+b_t\cdot\nabla\log G_t+\tfrac12\xi^{\top}M_t\xi
    &=\psi_t+\tfrac12\big(\xi^{\top}M_t\xi-\operatorname{tr}M_t\big)-\tfrac12\|g_t\|^2\\
    &=\Psi_t(X_t,\xi)-\tfrac12\|g_t\|^2
\end{align*}
by \cref{eq:app-Psi}. The remainder is $O(\Delta t_n^{3/2}(1+\|\xi\|^3))$ by (A2), since $\|D_n\|=O(\sqrt{\Delta t_n}(1+\|\xi\|))$.

(b) Replace $b_t$ by $b_t+\sigma_t\sigma_t^{\top}\nabla\log G_t$ in $D_n$. The extra drift contributes, through the linear term,
\begin{equation*}
    \Delta t_n\,\nabla\log G_t^{\top}\sigma_t\sigma_t^{\top}\nabla\log G_t=\Delta t_n\,\|g_t\|^2,
\end{equation*}
and only $O(\Delta t_n^{3/2})$ terms elsewhere, so the bracket becomes $\Psi_t(X_t,\xi)-\tfrac12\|g_t\|^2+\|g_t\|^2=\Psi_t(X_t,\xi)+\tfrac12\|g_t\|^2$.

(c) Apply the expansion of (a) to $f$ in place of $\log G$ and take expectations, using $\E\xi=0$, $\E[\xi^{\top}M\xi]=\operatorname{tr}M$ and $\E\|\xi\|^3<\infty$. The $\Delta t_n$ term is $-\partial_tf_t+b_t\cdot\nabla f_t+\tfrac12\operatorname{tr}(\sigma_t\sigma_t^{\top}\nabla^2f_t)=\gen_tf_t$ by \cref{eq:app-generator}.

(d) Both kernels are Gaussian with covariance $\Sigma:=\sigma_t\sigma_t^{\top}\Delta t_n$. Their means are $X_t+b_t\Delta t_n$ and $X_t+b_t\Delta t_n+s$ with $s:=\Sigma\nabla\log G_t$. The guided step produces $Y_{t'}=X_t+b_t\Delta t_n+s+e$ with $e:=\sigma_t\sqrt{\Delta t_n}\,\xi$, so the log-ratio of the two densities at $Y_{t'}$ is
\begin{align*}
    \log\frac{\Kb_{t'|t}(Y_{t'}\mid X_t)}{\Lb_{t'|t}(Y_{t'}\mid X_t)}
    &=-\tfrac12(s+e)^{\top}\Sigma^{-1}(s+e)+\tfrac12e^{\top}\Sigma^{-1}e\\
    &=-\tfrac12s^{\top}\Sigma^{-1}s-s^{\top}\Sigma^{-1}e,
\end{align*}
and the two quadratic forms are
\begin{align*}
    s^{\top}\Sigma^{-1}s&=\Delta t_n\,\nabla\log G_t^{\top}\sigma_t\sigma_t^{\top}\nabla\log G_t=\Delta t_n\,\|g_t\|^2,\\
    s^{\top}\Sigma^{-1}e&=\nabla\log G_t^{\top}\sigma_t\sqrt{\Delta t_n}\,\xi=\sqrt{\Delta t_n}\,g_t^{\top}\xi.
\end{align*}
There is no remainder.
\end{proof}

The constants in \cref{lem:onestep} depend only on the bounds in (A1)--(A2), which hold on all of $[0,1]\times\Xsp$, so the expansions are uniform in $n$. The next lemma turns them into the moment bounds that \cref{lem:smallweight} needs, again uniformly in $n$.

\begin{lemma}[Moment bounds]
\label{lem:moments}
Under (A1)--(A3) there are constants $C,C_0$ depending only on the bounds in (A1)--(A2) such that, for every schedule with $\max_n\Delta t_n\le1$ and every $n$,
\begin{align}
    \sup_n\E e^{2|\ell_n|}&\le C_0,\nonumber\\
    m_{1,n}\le C\sqrt{\Delta t_n},\qquad m_{4,n}&\le C\Delta t_n^{2}\qquad\text{(unguided)},\nonumber\\
    m_{1,n}\le C\Delta t_n,\qquad m_{4,n}&\le C\Delta t_n^{4}\qquad\text{(guided)}.
    \label{eq:app-uniform-moment}
\end{align}
\end{lemma}

\begin{proof}
\emph{Exponential moment.} $\log G_t$ is Lipschitz in $x$ and in $t$ by (A2), so
\begin{equation*}
    \Big|\log\frac{G_{t'}(Y_{t'})}{G_t(X_t)}\Big|\le C\big(\Delta t_n+\|Y_{t'}-X_t\|\big)\le C\sqrt{\Delta t_n}\,(1+\|\xi\|),
\end{equation*}
using $\|Y_{t'}-X_t\|=\|D_n\|\le C\sqrt{\Delta t_n}(1+\|\xi\|)$, and the proposal correction \cref{eq:app-onestep-d} is affine in $\xi$ with coefficients of the same order. Hence $|\ell_n|\le C\sqrt{\Delta t_n}(1+\|\xi\|)$ and, for $\Delta t_n\le1$,
\begin{equation*}
    \E e^{2|\ell_n|}\le\E e^{2C(1+\|\xi\|)}=:C_0<\infty,
\end{equation*}
with $C_0$ independent of $n$ and of the schedule.

\emph{Unguided moments.} By \cref{eq:app-onestep-a} without the proposal correction, $\ell_n=\sqrt{\Delta t_n}\,g_t^{\top}\xi+O(\Delta t_n(1+\|\xi\|^2))$, so $|\ell_n|\le C\sqrt{\Delta t_n}(1+\|\xi\|^2)$ for $\Delta t_n\le1$. Taking expectations, $m_{1,n}\le C\sqrt{\Delta t_n}\,\E(1+\|\xi\|^2)$ and $m_{4,n}\le C\Delta t_n^2\,\E(1+\|\xi\|^2)^4$, and the Gaussian moments are finite constants.

\emph{Guided moments.} By the proof of \cref{thm:barrier}(b) below, which uses only \cref{lem:onestep}(b) and (d), $\ell_n=\Delta t_n\Psi_t(X_t,\xi)+O(\Delta t_n^{3/2}(1+\|\xi\|^3))$ with $\Psi_t$ bounded by a quadratic in $\|\xi\|$, so $|\ell_n|\le C\Delta t_n(1+\|\xi\|^3)$ for $\Delta t_n\le1$, and $m_{1,n}\le C\Delta t_n$, $m_{4,n}\le C\Delta t_n^4$ follow as before.

All remainders in \cref{lem:onestep} are of the form $O(\Delta t_n^{a}(1+\|\xi\|^{b}))$ with $\xi$ Gaussian, so they may be integrated term by term and contribute $O(\Delta t_n^{a})$ to $m_{1,n}$ and to $\E|\ell_n-\ell'_n|$, and $O(\Delta t_n^{4a})$ to $m_{4,n}$, with constants independent of $n$.
\end{proof}

\begin{lemma}[Continuity of the local barriers]
\label{lem:continuity}
Under (A1)--(A2), $t\mapsto\lambda^{\mathrm u}(t)$ and $t\mapsto\lambda(t)$ are continuous on $[0,1]$. If moreover $t\mapsto p_t$ has a density that is continuously differentiable in $t$, as holds for the uniformly elliptic diffusion \cref{eq:app-sde} with the smooth bounded coefficients of (A1)--(A2) by parabolic regularity, then both are Lipschitz.
\end{lemma}

\begin{proof}
We treat $\lambda$. The argument for $\lambda^{\mathrm u}$ is the same with a simpler integrand.

\emph{Step 1: write $\lambda$ as an expectation of a fixed function.} Let $\nu_t:=\pi_t\otimes\pi_t\otimes N(0,I)\otimes N(0,I)$ be the joint law of $(X_t,X'_t,\xi,\xi')$ and
\begin{equation*}
    F(t,x,x',\xi,\xi'):=\tfrac12\big|\Psi_t(x,\xi)-\Psi_t(x',\xi')\big|,
    \qquad\text{so that}\qquad
    \lambda(t)=\int F(t,\cdot)\,\dd\nu_t .
\end{equation*}
By \cref{eq:app-Psi} and the bounds of (A1)--(A2) on $\psi_t$ and $M_t$,
\begin{equation}
    0\le F(t,x,x',\xi,\xi')\le C\big(1+\|\xi\|^2+\|\xi'\|^2\big)
    \label{eq:app-F-bound}
\end{equation}
uniformly in $(t,x,x')$, and $F$ is continuous in $(t,x,x')$ for each $(\xi,\xi')$, because $\psi_t(x)$ and $M_t(x)$ are continuous in $(t,x)$.

\emph{Step 2: continuity in the law.} Fix $t_0$ and let $t\to t_0$. Split
\begin{equation*}
    \lambda(t)-\lambda(t_0)
    =\underbrace{\int\big(F(t,\cdot)-F(t_0,\cdot)\big)\,\dd\nu_t}_{=:I_1(t)}
    +\underbrace{\int F(t_0,\cdot)\,\dd\nu_t-\int F(t_0,\cdot)\,\dd\nu_{t_0}}_{=:I_2(t)} .
\end{equation*}
For $I_1$, the integrand converges to zero pointwise as $t\to t_0$ by continuity of $F$ in $t$, and it is dominated by $2C(1+\|\xi\|^2+\|\xi'\|^2)$ by \cref{eq:app-F-bound}, whose $\nu_t$-integral is a constant independent of $t$. Hence $I_1(t)\to0$ by dominated convergence.

For $I_2$, the function $F(t_0,\cdot)$ is fixed, continuous, and bounded by a function of $(\xi,\xi')$ alone. Integrate out $(\xi,\xi')$ first: $\bar F(x,x'):=\int F(t_0,x,x',\xi,\xi')\,N(\dd\xi)N(\dd\xi')$ is bounded and continuous in $(x,x')$ by dominated convergence, and $I_2(t)=\int\bar F\,\dd(\pi_t\otimes\pi_t)-\int\bar F\,\dd(\pi_{t_0}\otimes\pi_{t_0})$. So it suffices that $t\mapsto\pi_t$ is weakly continuous. The law $p_t$ is the marginal at time $t$ of a diffusion with continuous paths, so $X_t\to X_{t_0}$ almost surely as $t\to t_0$ and $p_t\to p_{t_0}$ weakly. Since $G_t(x)$ is bounded, positive and continuous in $(t,x)$, for every bounded continuous $h$,
\begin{equation*}
    \int h\,\dd\pi_t=\frac{\int h\,G_t\,\dd p_t}{\int G_t\,\dd p_t}\ \longrightarrow\ \frac{\int h\,G_{t_0}\,\dd p_{t_0}}{\int G_{t_0}\,\dd p_{t_0}}=\int h\,\dd\pi_{t_0},
\end{equation*}
the numerator and denominator converging because $hG_t\to hG_{t_0}$ uniformly on compacts and $p_t\to p_{t_0}$ weakly, with tightness of $(p_t)$ supplied by path continuity. Hence $\pi_t\to\pi_{t_0}$ weakly, $I_2(t)\to0$, and $\lambda$ is continuous.

\emph{Step 3: Lipschitz regularity.} Suppose $p_t$ has a density $p_t(x)$ that is $C^1$ in $t$, with $\partial_tp_t$ integrable uniformly in $t$. Then $\pi_t(x)=G_t(x)p_t(x)/\mathcal Z_t$ is $C^1$ in $t$, and for the bounded function $\bar F$ of Step 2 the map $t\mapsto\int\bar F\,\dd(\pi_t\otimes\pi_t)$ is $C^1$ by differentiation under the integral. Moreover $F$ is Lipschitz in $t$ uniformly in $(x,x')$ and with a constant proportional to $1+\|\xi\|^2+\|\xi'\|^2$, since $\partial_t\psi_t$ and $\partial_tM_t$ are bounded by (A2). Writing $\lambda(t)-\lambda(s)=I_1+I_2$ as in Step 2 with $t_0$ replaced by $s$, both terms are bounded by $C|t-s|$, so $\lambda$ is Lipschitz.
\end{proof}

\subsection{Proof of \texorpdfstring{\cref{thm:barrier}}{Theorem 3.2}}
\label{app:barrier-proof}

The continuity of $\lambda^{\mathrm u}$ and $\lambda$ asserted in the theorem is \cref{lem:continuity}. The two expansions are proved next, with all constants uniform in $n$ by \cref{lem:moments}.

\begin{proof}[Proof of (a): the unguided proposal]
With $\mathbf 1_{\mathrm{guided}}=0$ the proposal correction is absent, so by \cref{eq:app-onestep-a} at $t=t_n$,
\begin{equation*}
    \ell_n=\sqrt{\Delta t_n}\,g_{t_n}(X_{t_n})^{\top}\xi+O\big(\Delta t_n(1+\|\xi\|^2)\big).
\end{equation*}
The leading term is a mean-zero Gaussian fluctuation of order $\sqrt{\Delta t_n}$, and it survives the mean absolute difference: by \cref{lem:moments}, $m_{1,n}=O(\sqrt{\Delta t_n})$ and $m_{4,n}^{1/2}=O(\Delta t_n)$ uniformly in $n$, so \cref{lem:smallweight} gives $r_{n-1,n}=\tfrac12\E|\ell_n-\ell'_n|+O(\Delta t_n)$. The $O(\Delta t_n)$ terms of $\ell_n-\ell'_n$ contribute $O(\Delta t_n)$ to the mean absolute difference by the triangle inequality, and \cref{lem:gaussdiff} evaluates the leading term:
\begin{align*}
    r_{n-1,n}&=\tfrac{\sqrt{\Delta t_n}}{2}\,\E\big|g_{t_n}(X_{t_n})^{\top}\xi-g_{t_n}(X'_{t_n})^{\top}\xi'\big|+O(\Delta t_n)\\
     &=\tfrac{\sqrt{\Delta t_n}}{2}\sqrt{\tfrac2\pi}\;\E\sqrt{\|g_{t_n}(X_{t_n})\|^2+\|g_{t_n}(X'_{t_n})\|^2}+O(\Delta t_n)\\
     &=\lambda^{\mathrm u}(t_n)\sqrt{\Delta t_n}+O(\Delta t_n),
\end{align*}
with $\lambda^{\mathrm u}$ as in \cref{eq:local-barriers}, and the $O(\Delta t_n)$ term is bounded by $C\Delta t_n$ with $C$ independent of $n$.

For the sum, $\lambda^{\mathrm u}$ is continuous in $t$ by \cref{lem:continuity}, and it is not identically zero unless $\nabla\log G_t\equiv0$ for every $t$, in which case there is no control. If $\lambda^{\mathrm u}\ge c>0$ on a set of times of length $L>0$, then on any schedule with $\max_n\Delta t_n\to0$,
\begin{align*}
    \sum_{n}r_{n-1,n}\ &\ge\ c\sum_{t_n\in\text{that set}}\sqrt{\Delta t_n}-O(1)\\
    \ &\ge\ c\sum_{t_n\in\text{that set}}\frac{\Delta t_n}{\sqrt{\max_n\Delta t_n}}-O(1)\\
    \ &\ge\ \frac{cL}{\sqrt{\max_n\Delta t_n}}-O(1)\\\ 
    &\to\ \infty .
\end{align*}
On a uniform schedule, $\Delta t_n=1/N$, the Riemann sum gives
\begin{align*}
    \sum_{n=1}^Nr_{n-1,n}
    &=\sqrt N\sum_{n=1}^N\lambda^{\mathrm u}(t_n)\,\Delta t_n+O(1)\\
    &=\sqrt N\int_0^1\lambda^{\mathrm u}(t)\,\dd t+O(1)\\
    &=\Theta(\sqrt N).
\end{align*}
Since $r/(1-r)\ge r$, the restart rate \cref{eq:roundtrip} tends to zero.
\end{proof}

\begin{proof}[Proof of (b): the guided proposal]
By \cref{eq:app-onestep-a} in the guided form (b) and by \cref{eq:app-onestep-d}, both at $t=t_n$, the two terms of \cref{eq:app-logweight} are
\begin{align*}
    \log\frac{G_{t'}(Y_{t'})}{G_t(X_t)}&=\sqrt{\Delta t_n}\,g_t^{\top}\xi+\Delta t_n\Big[\Psi_t(X_t,\xi)+\tfrac12\|g_t\|^2\Big]+O\big(\Delta t_n^{3/2}(1+\|\xi\|^3)\big),\\
    \log\frac{\Kb_{t'|t}(Y_{t'}\mid X_t)}{\Lb_{t'|t}(Y_{t'}\mid X_t)}&=-\sqrt{\Delta t_n}\,g_t^{\top}\xi-\tfrac{\Delta t_n}{2}\|g_t\|^2 .
\end{align*}
Adding them, the $\sqrt{\Delta t_n}$ terms and the $\|g_t\|^2$ terms cancel exactly, leaving
\begin{equation*}
    \ell_n=\Delta t_n\,\Psi_{t_n}(X_{t_n},\xi)+O\big(\Delta t_n^{3/2}(1+\|\xi\|^3)\big),
\end{equation*}
which is the guided case of \cref{eq:guided-logweight}: guidance moves the proposal to where the potential is steep, and the correction it pays for doing so removes exactly the fluctuation that the steepness would otherwise cause. Now $m_{1,n}=O(\Delta t_n)$ and $m_{4,n}^{1/2}=O(\Delta t_n^2)$ uniformly in $n$ by \cref{lem:moments}, so \cref{lem:smallweight} with the triangle inequality on the remainder gives
\begin{equation}
    r_{n-1,n}=\tfrac{\Delta t_n}{2}\,\E\big|\Psi_{t_n}(X_{t_n},\xi)-\Psi_{t_n}(X'_{t_n},\xi')\big|+R_n=\lambda(t_n)\,\Delta t_n+R_n,
    \qquad |R_n|\le C\,\Delta t_n^{3/2},
    \label{eq:app-guided-rate}
\end{equation}
with $C$ independent of $n$, by \cref{lem:moments} and the uniformity of the remainders in \cref{lem:onestep}. The error is therefore uniform in $n$ along any schedule.

\emph{Integral form.} By \cref{lem:continuity}, $\lambda$ is continuous, so $\lambda(t_n)\Delta t_n=\int_{t_{n-1}}^{t_n}\lambda(t)\,\dd t+o(\Delta t_n)$ uniformly in $n$ by uniform continuity on $[0,1]$, and $O(\Delta t_n^2)$ when $\lambda$ is Lipschitz. Hence \cref{eq:app-guided-rate} can be written as
\begin{equation}
    r_{n-1,n}=\int_{t_{n-1}}^{t_n}\lambda(t)\,\dd t+o(\Delta t_n),
    \label{eq:app-guided-rate-integral}
\end{equation}
uniformly in $n$, which is the statement of \cref{thm:barrier}(b) and the form used to construct the equi-rejection schedule in \cref{app:schedule}. The $o(\Delta t_n)$ needs only continuity of $\lambda$, hence only (A1)--(A3). It improves to $O(\Delta t_n^{3/2})$ under the Lipschitz regularity of \cref{lem:continuity}, and to $O(\Delta t_n^2)$ under the additional non-degeneracy condition of \cref{prop:second-order} below.

\emph{The sum.} Summing \cref{eq:app-guided-rate-integral} over the schedule,
\begin{equation*}
    \sum_{n=1}^Nr_{n-1,n}=\int_0^1\lambda(t)\,\dd t+\sum_{n=1}^No(\Delta t_n)\ \to\ \Lambda<\infty,
\end{equation*}
since the error is $o(1)$ by uniformity of the $o(\Delta t_n)$ terms in $n$ and $\sum_n\Delta t_n=1$. Since $r_{n-1,n}\to0$ uniformly in $n$, the odds satisfy $r/(1-r)=r+O(r^2)$ and $\sum_nr_{n-1,n}^2=O(\max_n\Delta t_n)\to0$, so the summed odds have the same limit and the restart rate \cref{eq:roundtrip} converges to $(2+2\Lambda)^{-1}$.

It remains to characterise $\Lambda=0$, which is \cref{lem:zero-barrier} below.
\end{proof}

\begin{proposition}[Second-order accuracy]
\label{prop:second-order}
Assume (A1)--(A3), that $\lambda$ is Lipschitz (\cref{lem:continuity}), and that for every $t$ the difference $Y_t:=\Psi_t(X_t,\xi)-\Psi_t(X'_t,\xi')$ of two independent copies has a density that is bounded near zero, uniformly in $t$. Then, for the guided proposal,
\begin{equation}
    r_{n-1,n}=\int_{t_{n-1}}^{t_n}\lambda(t)\,\dd t+O(\Delta t_n^2)\qquad\text{uniformly in }n .
    \label{eq:app-second-order}
\end{equation}
\end{proposition}

\begin{proof}
The Lipschitz regularity of $\lambda$ gives $\lambda(t_n)\Delta t_n=\int_{t_{n-1}}^{t_n}\lambda+O(\Delta t_n^2)$, so it suffices to show $r_{n-1,n}=\lambda(t_n)\Delta t_n+O(\Delta t_n^2)$. By \cref{lem:smallweight} and \cref{lem:moments} the error in $r_{n-1,n}=\tfrac12\E|\ell_n-\ell'_n|$ is already $O(\Delta t_n^2)$ in the guided case, so the question is the remainder in $\ell_n$ itself.

\emph{The remainder is odd in the noise.} Carrying the expansion of \cref{lem:onestep} one order further, with all derivatives at $(t,X_t)$,
\begin{equation*}
    \ell_n=\Delta t_n\,\Psi_t(X_t,\xi)+\Delta t_n^{3/2}\,R_t(X_t,\xi)+O\big(\Delta t_n^{2}(1+\|\xi\|^4)\big),
\end{equation*}
where $R_t$ collects the terms of order $\Delta t_n^{3/2}$: the cross term $b_t^{\top}\nabla^2\log G_t\,\sigma_t\xi$ from the quadratic part, the term $-\partial_t\nabla\log G_t^{\top}\sigma_t\xi$ from the mixed derivative, and the cubic term $\tfrac16\nabla^3\log G_t[\sigma_t\xi,\sigma_t\xi,\sigma_t\xi]$, together with the corresponding terms from the guided drift. Every term in $R_t$ has an odd power of $\xi$, so $R_t(x,-\xi)=-R_t(x,\xi)$, whereas $\Psi_t(x,-\xi)=\Psi_t(x,\xi)$ by \cref{eq:app-Psi} and the $O(\Delta t_n^2)$ terms are even in $\xi$ as well.

\emph{Parity.} Let $Y:=\Psi_t(X_t,\xi)-\Psi_t(X'_t,\xi')$ and $Z:=R_t(X_t,\xi)-R_t(X'_t,\xi')$, and consider $f(\varepsilon):=\E|Y+\varepsilon Z|$. The map $(\xi,\xi')\mapsto(-\xi,-\xi')$ preserves the law of $(X_t,X'_t,\xi,\xi')$, fixes $Y$ and negates $Z$, so $f(\varepsilon)=f(-\varepsilon)$. Under the density assumption on $Y$, $f$ is twice continuously differentiable near $0$ with $f'(0)=\E[\operatorname{sign}(Y)Z]$ and $f''(0)=\E[Z^2\mid Y=0]\,p_Y(0)$, both finite since $Z$ has all moments by (A2). Evenness forces $f'(0)=0$, hence $f(\varepsilon)=f(0)+O(\varepsilon^2)$ with a constant uniform in $t$. With $\varepsilon=\sqrt{\Delta t_n}$,
\begin{equation*}
    \tfrac12\E|\ell_n-\ell'_n|
    =\tfrac{\Delta t_n}{2}\,f\big(\sqrt{\Delta t_n}\big)+O(\Delta t_n^2)
    =\tfrac{\Delta t_n}{2}\,\E|Y|+O(\Delta t_n^2)
    =\lambda(t_n)\,\Delta t_n+O(\Delta t_n^2),
\end{equation*}
the first equality absorbing the even $O(\Delta t_n^2(1+\|\xi\|^4))$ remainder by the triangle inequality and the Gaussian moments. The constants depend only on the bounds in (A1)--(A2) and on the density bound, hence not on $n$.
\end{proof}

Without the density assumption the parity argument still gives $f(\varepsilon)-f(0)\le\varepsilon\E|Z|$, which is the $O(\Delta t_n^{3/2})$ of \cref{eq:app-guided-rate}. The assumption holds, for instance, whenever $v(t)$ is bounded away from zero and $\Psi_t(X_t,\xi)$ has a bounded density, as it does when $M_t\not\equiv0$ or $\psi_t(X_t)$ is non-degenerate.

\begin{lemma}[The Doob residual]
\label{lem:doob-residual}
Under (A1)--(A2), the following are equivalent for every $t$: (i) $\psi_t=\gen_tG_t/G_t$ is constant in $x$, (ii) $G_t$ is an eigenfunction of $\gen_t$, $\gen_tG_t=c(t)G_t$ for some $c(t)$, and (iii) $G_t\propto G^*_t$, where $G^*_t(x)=\E[G_0(X_0)\mid X_t=x]$ is the eigenfunction with eigenvalue zero.
\end{lemma}

\begin{proof}
The first two statements are the same statement, since $\psi_t=\gen_tG_t/G_t$. Suppose $\gen_tG_t=c(t)G_t$. Set $\phi(t):=\int_0^tc(s)\,\dd s$ and $\widetilde G_t:=e^{\phi(t)}G_t$. Since $\gen_t$ acts on the time dependence through $-\partial_t$ only,
\begin{equation*}
    \gen_t\widetilde G_t=-\partial_t\big(e^{\phi(t)}\big)G_t+e^{\phi(t)}\gen_tG_t=-c(t)e^{\phi(t)}G_t+e^{\phi(t)}c(t)G_t=0,
\end{equation*}
and $\widetilde G_0=G_0$. So $\widetilde G$ solves the backward equation $\gen_t\widetilde G_t=0$ with terminal condition $G_0$, whose unique solution with bounded logarithm is $G^*_t$: existence is the Feynman--Kac representation, and uniqueness holds because the difference of two bounded solutions is a bounded space-time harmonic function of the diffusion, hence constant along it and zero at $t=0$ \citep[Ch.~1]{delmoral_feynmankac_2004}. Hence $G_t=e^{-\phi(t)}G^*_t\propto G^*_t$.

Conversely, if $G_t=e^{-\phi(t)}G^*_t$ for a differentiable $\phi$, then
\begin{equation*}
    \gen_tG_t=-\partial_t\big(e^{-\phi(t)}\big)G^*_t+e^{-\phi(t)}\gen_tG^*_t=\phi'(t)\,e^{-\phi(t)}G^*_t+0=\phi'(t)\,G_t,
\end{equation*}
so $\psi_t=\phi'(t)$ is constant in $x$.
\end{proof}

\begin{lemma}[Zero barrier]
\label{lem:zero-barrier}
Under (A1)--(A2), $\lambda(t)=0$ if and only if $\log G_t$ is affine in $x$ and $G_t\propto G^*_t$. Hence $\Lambda=0$ if and only if both hold for every $t$, i.e.\ if and only if the guided proposal is the Doob proposal.
\end{lemma}

\begin{proof}
$\lambda(t)=\tfrac12\E|\Psi_t(X_t,\xi)-\Psi_t(X'_t,\xi')|$ vanishes if and only if two independent copies of $\Psi_t(X_t,\xi)$ agree almost surely, i.e.\ if and only if $\Psi_t(X_t,\xi)$ is almost surely constant.

\emph{The curvature term must vanish.} Condition on $X_t=x$. By \cref{eq:app-Psi},
\begin{equation*}
    \Psi_t(x,\xi)=\psi_t(x)+\tfrac12\big(\xi^{\top}M_t(x)\xi-\operatorname{tr}M_t(x)\big)
\end{equation*}
is almost surely constant in $\xi\sim N(0,I)$, so the quadratic form $\xi\mapsto\xi^{\top}M_t(x)\xi$ is almost surely constant. A polynomial that is constant on a set of positive Lebesgue measure is constant everywhere, and a quadratic form is constant only if its matrix is zero, so $M_t(x)=0$ for $\pi_t$-almost every $x$. By continuity of $M_t$ in $x$ and positivity of $\pi_t$, $M_t\equiv0$. By (A1) $\sigma_t$ is invertible, so $\nabla^2\log G_t\equiv0$ and $\log G_t$ is affine in $x$.

\emph{The residual must be constant.} With $M_t\equiv0$, $\Psi_t(x,\xi)=\psi_t(x)$, which is almost surely constant under $\pi_t$, hence everywhere by continuity and positivity. By \cref{lem:doob-residual}, $G_t\propto G^*_t$.

\emph{Conversely}, if $\log G_t$ is affine then $M_t\equiv0$, and if $G_t\propto G^*_t$ then $\psi_t$ is constant in $x$ by \cref{lem:doob-residual}. So $\Psi_t$ is constant and $\lambda(t)=0$. Since $\lambda\ge0$ is continuous, $\Lambda=0$ if and only if $\lambda\equiv0$. The final equivalence is the affine case of \cref{app:bg-doob}: for an affine log-potential the guided step is the Doob proposal of $G_t$ by \cref{eq:app-complete-square}, and $G_t\propto G^*_t$ makes that the Doob proposal of the target.
\end{proof}

\subsection{The Guided Proposal Against the Doob Proposal}
\label{app:affine}

For the chosen potentials $G_t$, the Doob proposal is the twisted kernel
\begin{equation}
    \Lb^{\mathrm{tw}}_{t'|t}(\dd y\mid x):=\frac{G_{t'}(y)\,\Kb_{t'|t}(\dd y\mid x)}{(\Kb_{t'|t}G_{t'})(x)},
    \label{eq:app-twisted}
\end{equation}
which satisfies the proposal relation of \cref{eq:app-doob-relations} exactly and, by \cref{thm:exact}(c), minimises the rejection over all proposals for those potentials. The guided step approximates it. The next proposition computes its barrier and compares it with the guided one.

\begin{proposition}[Twisted proposal]
\label{prop:twisted}
Under the assumptions of \cref{thm:barrier}, the twisted proposal has $r_{n-1,n}=\lambda^{\mathrm{tw}}(t_n)\Delta t_n+O(\Delta t_n^2)$ with
\begin{equation}
    \lambda^{\mathrm{tw}}(t):=\tfrac12\,\E\big|\psi_t(X_t)-\psi_t(X'_t)\big|\ \le\ \lambda(t),
    \label{eq:app-lambda-tw}
\end{equation}
with equality if and only if $\log G_t$ is affine in $x$, and $\Lambda^{\mathrm{tw}}:=\int_0^1\lambda^{\mathrm{tw}}(t)\,\dd t=0$ if and only if $G_t\propto G^*_t$ for all $t$.
\end{proposition}

\begin{proof}
The proposal correction of the twisted kernel cancels the potential ratio in $y$, and the weight depends on $X_t$ alone:
\begin{equation*}
    w_{t'|t}(X_t,\xi)=\frac{G_{t'}(Y_{t'})}{G_t(X_t)}\cdot\frac{\Kb_{t'|t}(Y_{t'}\mid X_t)}{\Lb^{\mathrm{tw}}_{t'|t}(Y_{t'}\mid X_t)}
    =\frac{G_{t'}(Y_{t'})}{G_t(X_t)}\cdot\frac{(\Kb_{t'|t}G_{t'})(X_t)}{G_{t'}(Y_{t'})}
    =\frac{(\Kb_{t'|t}G_{t'})(X_t)}{G_t(X_t)}.
\end{equation*}
The numerator is the one-step expectation of $G_{t'}$, so by \cref{eq:app-onestep-c} applied to $f=G$,
\begin{align*}
    \ell_n&=\log\frac{G_t(X_t)+\Delta t_n\,(\gen_tG_t)(X_t)+O(\Delta t_n^2)}{G_t(X_t)}\\
    &=\log\big(1+\Delta t_n\,\psi_t(X_t)+O(\Delta t_n^2)\big)
    =\Delta t_n\,\psi_t(X_t)+O(\Delta t_n^2),
\end{align*}
uniformly in $X_t$ and in $n$. Hence $m_{1,n}=O(\Delta t_n)$, $m_{4,n}^{1/2}=O(\Delta t_n^2)$, and \cref{lem:smallweight} gives $r_{n-1,n}=\lambda^{\mathrm{tw}}(t_n)\Delta t_n+O(\Delta t_n^2)$.

For the comparison, $\E[\Psi_t(X_t,\xi)\mid X_t]=\psi_t(X_t)$ by \cref{eq:app-Psi}, and conditional Jensen applied to the convex function $|\cdot|$ gives
\begin{align*}
    \lambda(t)&=\tfrac12\,\E\big|\Psi_t(X_t,\xi)-\Psi_t(X'_t,\xi')\big|\\
    &\ge\tfrac12\,\E\Big|\,\E\big[\Psi_t(X_t,\xi)-\Psi_t(X'_t,\xi')\ \big|\ X_t,X'_t\big]\Big|\\
    &=\tfrac12\,\E\big|\psi_t(X_t)-\psi_t(X'_t)\big|=\lambda^{\mathrm{tw}}(t),
\end{align*}
with equality if and only if the curvature term $\tfrac12(\xi^{\top}M_t\xi-\operatorname{tr}M_t)$ vanishes almost surely, i.e.\ $M_t\equiv0$, i.e.\ $\log G_t$ affine, as in \cref{lem:zero-barrier}. Finally $\lambda^{\mathrm{tw}}(t)=0$ if and only if $\psi_t$ is constant in $x$, which is $G_t\propto G^*_t$ by \cref{lem:doob-residual}.
\end{proof}

\paragraph{The curvature term.} When $\log G_{t'}$ is affine in $x$ the guided step is the twisted kernel exactly, by \cref{eq:app-complete-square}. When it is not, the two kernels differ at second order. Write the sampler step as $\Kb_{t'|t}(\cdot\mid x)=N(\mu_n(x),\Sigma_n)$ with $\mu_n(x)=x+b_t(x)\Delta t_n$ and $\Sigma_n=\sigma_t\sigma_t^{\top}\Delta t_n$, and expand $\log G_{t'}(y)\approx c_n^{\top}y+\tfrac12y^{\top}H_ny$ with $H_n=\nabla^2\log G_{t'}$. Completing the square gives
\begin{equation*}
    \Lb^{\mathrm{tw}}_{t'|t}(\cdot\mid x)\approx N\Big((\Sigma_n^{-1}-H_n)^{-1}\big(\Sigma_n^{-1}\mu_n(x)+c_n\big),\ (\Sigma_n^{-1}-H_n)^{-1}\Big),
\end{equation*}
so the twisted kernel has covariance $(\Sigma_n^{-1}-H_n)^{-1}$ rather than $\Sigma_n$: narrower where $\log G_{t'}$ is concave ($H_n\prec0$), wider where it is convex, with the mean shifted accordingly. The guided step keeps the covariance $\Sigma_n$ and shifts only the mean, and the difference between the two is what produces the curvature term $\tfrac12(\xi^{\top}M_t\xi-\operatorname{tr}M_t)$ in $\Psi_t$ of \cref{eq:Psi-main}. That term has mean zero given $X_t$, so no choice of the potentials removes it. It is the cost of a Gaussian proposal with the sampler's covariance, and it is what separates $\Lambda$ from $\Lambda^{\mathrm{tw}}$.

\subsection{Tuning the Schedule}
\label{app:schedule}

Under the guided proposal, \cref{eq:app-guided-rate-integral} gives $r_{n-1,n}=\int_{t_{n-1}}^{t_n}\lambda(t)\,\dd t+o(\Delta t_n)$ uniformly in $n$, so the summed rejection rates converge to $\Lambda=\int_0^1\lambda$ and a schedule controls only how $\Lambda$ is distributed over the levels. By \cref{eq:roundtrip} the restart rate is decreasing in $\sum_nr_{n-1,n}/(1-r_{n-1,n})$, and $s\mapsto s/(1-s)$ is convex on $[0,1)$, so by Jensen's inequality, among schedules with a fixed number of levels and a fixed sum $\sum_nr_{n-1,n}\approx\Lambda$, the sum of odds is minimised when the per-level rejection rates are equal:
\begin{equation}
    r_{n-1,n}=\int_{t_{n-1}}^{t_n}\lambda(t)\,\dd t=\frac{\Lambda}{N}\qquad\text{for every }n .
    \label{eq:app-equi}
\end{equation}
This is the equi-rejection schedule of \citet{syed_nonreversible_2022} applied to the local barrier of \cref{thm:barrier}, and their analysis applies: $N\approx2\Lambda$ levels give stable communication, after which a fresh reference draw reaches the target every $2+2\Lambda$ iterations whatever $N$, and a longer run rather than a finer ladder increases the sample size. In practice $\lambda$ is estimated from the rejection rates observed along a pilot run and the schedule is refined by inverting the resulting empirical barrier, as there. Under the unguided proposal no such limit exists and the spacing is instead set by \cref{prop:spacing}.

\subsection{Optimal Spacing for the Unguided Ladder}
\label{app:spacing}

\begin{lemma}[Log-normal weights]
\label{lem:erf}
If $\log W_{n-1|n}\sim N(\mu_n,s_n^2)$ under $\pi_n\otimes\Lb_{n-1|n}$, then
\begin{equation}
    r_{n-1,n}=\erf(s_n/2)=\frac{s_n}{\sqrt\pi}+O(s_n^3).
    \label{eq:app-erf}
\end{equation}
\end{lemma}

\begin{proof}
Write $W_{n-1|n}=e^{A}$, $W'_{n-1|n}=e^{B}$ with $A,B$ i.i.d.\ $N(\mu_n,s_n^2)$. By symmetry,
\begin{equation*}
    \E|W_{n-1|n}-W'_{n-1|n}|=2\,\E\big[(e^{A}-e^{B})\mathbf 1_{A>B}\big]=2\big(\E[e^{A}\mathbf 1_{A>B}]-\E[e^{B}\mathbf 1_{A>B}]\big).
\end{equation*}
Under the tilted measure $\dd\mathbb Q=e^{A-\mu_n-s_n^2/2}\dd\mathbb P$ one has $A\sim N(\mu_n+s_n^2,s_n^2)$ and $B$ unchanged, so $A-B\sim N(s_n^2,2s_n^2)$ under $\mathbb Q$ and
\begin{equation*}
    \E\big[e^{A}\mathbf 1_{A>B}\big]=e^{\mu_n+s_n^2/2}\,\mathbb Q(A>B)=e^{\mu_n+s_n^2/2}\,\Phi\big(s_n/\sqrt2\big),
\end{equation*}
with $\Phi$ the standard normal distribution function. The same tilt applied to $B$ gives $\E[e^{B}\mathbf 1_{A>B}]=e^{\mu_n+s_n^2/2}\,\Phi(-s_n/\sqrt2)$. Hence
\begin{equation*}
    \E|W_{n-1|n}-W'_{n-1|n}|=2e^{\mu_n+s_n^2/2}\big(2\Phi(s_n/\sqrt2)-1\big),
    \qquad
    \E[W_{n-1|n}]=e^{\mu_n+s_n^2/2},
\end{equation*}
and $2\Phi(x)-1=\erf(x/\sqrt2)$ gives $r_{n-1,n}=\E|W_{n-1|n}-W'_{n-1|n}|/(2\E W_{n-1|n})=\erf(s_n/2)$ by \cref{eq:app-rej-weight}.
\end{proof}

The proof of \cref{thm:barrier}(a) shows that with the unguided proposal, refining a time stretch always increases the summed rejection rate: if the log-weight increments over disjoint intervals are approximately Gaussian and independent with variances $v_i$, then
\begin{equation*}
    \sum_i\erf\big(\sqrt{v_i}/2\big)\ \ge\ \erf\Big(\sqrt{\textstyle\sum_iv_i}\,/2\Big),
\end{equation*}
because $v\mapsto\erf(\sqrt v/2)$ is concave with value $0$ at $0$, hence subadditive. Merging levels is limited by the cost $\sum_ir_i/(1-r_i)$ in \cref{eq:roundtrip}, which diverges as $r_i\to1$. The balance is the following.

\begin{proposition}[Optimal unguided spacing]
\label{prop:spacing}
Let $\rho(u):=\erf(u)/\erfc(u)$ and $f(v):=\rho(\sqrt v/2)$. For any partition of a time stretch into steps with Gaussian log-weight variances $v_i$ summing to $S$,
\begin{equation}
    \sum_i\frac{r_i}{1-r_i}=\sum_if(v_i)\ \ge\ S\inf_{v>0}\frac{f(v)}{v}\ \approx\ 1.06\,S,
    \label{eq:app-spacing}
\end{equation}
with equality up to integrality when all $v_i=v^*\approx1.5$, i.e.\ log-weight standard deviation $\approx1.2$ per step and per-exchange rejection $r^*=\erf(0.61)\approx0.61$.
\end{proposition}

\begin{proof}
By \cref{lem:erf}, $r_i=\erf(\sqrt{v_i}/2)$, so $r_i/(1-r_i)=f(v_i)$. Then
\begin{equation*}
    \sum_if(v_i)=\sum_iv_i\,\frac{f(v_i)}{v_i}\ \ge\ \Big(\sum_iv_i\Big)\inf_{v>0}\frac{f(v)}{v}=S\inf_{v>0}\frac{f(v)}{v},
\end{equation*}
with equality if and only if every $v_i$ attains the infimum. Substituting $v=4u^2$, $f(v)/v=\rho(u)/(4u^2)$ tends to $+\infty$ as $u\to0$ (since $\rho(u)\sim2u/\sqrt\pi$) and as $u\to\infty$ (since $\erfc(u)\to0$), and a numerical minimisation gives $u^*\approx0.61$ with value $\approx1.06$.
\end{proof}

Under \cref{eq:roundtrip} this is the spacing that maximises the restart rate across an unguided time stretch, and the effective cost of the stretch is its total log-weight variance $S$: consecutive levels should differ by about $1.2$ in log-weight standard deviation, and the ladder should be coarsened wherever $\nabla\log G_t$ is small. Where $G_t\equiv1$ the log-weight is identically zero and a single level suffices.

\subsection{The SMC Barrier: FKD and TDS}
\label{app:smc-barrier}

The backward weights of \cref{eq:weight-factor} also drive the sequential control samplers FKD and TDS, and the expansions above determine how those samplers behave under refinement. Write $\ell_n:=\log w_{t_{n-1}|t_n}$ for the log-weight of a particle at step $n$ of the schedule, with variance taken under the proposal law $\pi_n\otimes\Lb_{n-1|n}$, and recall from \cref{eq:app-smc-variance} that the relative variance of the normalising-constant estimator is governed by the total discrepancy $\mathcal D_{0,N}=\sum_n\mathcal D_{n-1,n}$ of \cref{eq:app-discrepancy}, which in the small-weight regime is the summed log-weight variance $\sum_n\Var(\ell_n)$ up to $O(\sum_n\E|\ell_n|^3)$, a term that vanishes with refinement in both cases below.

\paragraph{Unguided proposal (FKD).} By the proof of \cref{thm:barrier}(a), $\ell_n=\sqrt{\Delta t_n}\,g_{t_n}(X_{t_n})^{\top}\xi+O(\Delta t_n)$ with $\xi$ independent of $X_{t_n}$. Conditionally on $X_{t_n}$ the leading term has variance $\Delta t_n\|g_{t_n}(X_{t_n})\|^2$ and mean zero, so
\begin{equation*}
    \Var(\ell_n)=\E_{\pi_{t_n}}\big[\|g_{t_n}(X_{t_n})\|^2\big]\,\Delta t_n+O(\Delta t_n^{3/2}),
\end{equation*}
and summing the Riemann sum,
\begin{equation}
    \mathcal D_{0,N}\ \xrightarrow[N\to\infty]{}\ \int_0^1\E_{\pi_t}\big[\|\sigma_t^{\top}\nabla\log G_t(X_t)\|^2\big]\,\dd t .
    \label{eq:app-smc-unguided}
\end{equation}
The limit is positive whenever the potential path is non-trivial and does not depend on the schedule: refining the sampler does not reduce the weight variance of the unguided sampler, and a stronger tilt increases it quadratically. This is why FKD collapses at small batch sizes irrespective of the number of steps.

\paragraph{Guided proposal (TDS).} By the proof of \cref{thm:barrier}(b), $\ell_n=\Delta t_n\,\Psi_{t_n}(X_{t_n},\xi)+O(\Delta t_n^{3/2})$, so $\Var(\ell_n)=\Delta t_n^2\,v(t_n)+O(\Delta t_n^{5/2})$ with
\begin{equation}
    v(t):=\Var\big(\Psi_t(X_t,\xi)\big)=\Var_{\pi_t}\big(\psi_t(X_t)\big)+\tfrac12\,\E_{\pi_t}\|M_t(X_t)\|_F^2,
    \label{eq:app-v}
\end{equation}
by the law of total variance: the first term is the variance of the conditional mean $\psi_t$, the second is the mean of the conditional variance of the curvature term, which is $\tfrac12\|M_t\|_F^2$ since $\Var(\xi^{\top}M\xi)=2\|M\|_F^2$ for symmetric $M$ and $\xi\sim N(0,I)$. The twisted proposal drops the second term. Hence
\begin{equation}
    \mathcal D_{0,N}=\sum_{n=1}^N\Delta t_n^2\,v(t_n)+o\big(\textstyle\max_n\Delta t_n\big),
    \label{eq:app-smc-guided}
\end{equation}
which is of order $1/N$ for any schedule with $\max_n\Delta t_n=O(1/N)$. The schedule that minimises it for given $N$ equalises the per-step standard deviations $\Delta t_n\sqrt{v(t_n)}$: by Cauchy--Schwarz, with $\sum_n\Delta t_n=1$,
\begin{equation*}
    \Big(\sum_{n=1}^N\Delta t_n\sqrt{v(t_n)}\Big)^2\ \le\ N\sum_{n=1}^N\Delta t_n^2\,v(t_n),
\end{equation*}
with equality if and only if all $\Delta t_n\sqrt{v(t_n)}$ are equal, and the left side is a Riemann sum for $\Lambda_{\mathrm{SMC}}^2$. Therefore
\begin{equation}
    \min_{\text{schedules}}\mathcal D_{0,N}=\frac{\Lambda_{\mathrm{SMC}}^2}{N}+o(1/N),
    \qquad
    \Lambda_{\mathrm{SMC}}:=\int_0^1\sqrt{v(t)}\,\dd t,
    \label{eq:app-smc-barrier}
\end{equation}
attained when $\Delta t_n\sqrt{v(t_n)}=\Lambda_{\mathrm{SMC}}/N$ for every $n$. $\Lambda_{\mathrm{SMC}}$ is the global barrier of annealed SMC in the sense of \citet{syed_optimized_2026}, with $\Psi_t$ as the rate of change of the log incremental weight, and it is finite under the assumptions of \cref{thm:barrier}. By \cref{eq:app-smc-variance} the relative variance of $\widehat Z^{\mathrm{SMC}}$ under the optimal schedule is $\approx\Lambda_{\mathrm{SMC}}^2/(NM)$ for $N$ large, which vanishes with refinement. This is why TDS is accurate with few particles, and its optimal schedule places steps where $\Psi_t$ fluctuates most.

\paragraph{Relation between the two barriers.} The two barriers are equivalent up to constants. Let $v(t)=\Var(\Psi_t(X_t,\xi))$ as in \cref{eq:app-v} and, where $v(t)>0$, let
\begin{equation}
    \kappa_t:=\frac{\E\big[(\Psi_t(X_t,\xi)-\E\Psi_t(X_t,\xi))^4\big]}{v(t)^2}
    \label{eq:app-kurtosis}
\end{equation}
be the kurtosis of $\Psi_t(X_t,\xi)$. Under (A1)--(A2) all moments of $\Psi_t$ are bounded uniformly in $t$, so $\kappa_t$ is finite wherever $v(t)>0$.

\begin{proposition}[Equivalence of the barriers]
\label{prop:barrier-equivalence}
For every $t$ with $v(t)>0$,
\begin{equation}
    \frac{\sqrt{v(t)}}{\sqrt{\kappa_t+3}}\ \le\ \lambda(t)\ \le\ \frac{\sqrt{v(t)}}{\sqrt2},
    \label{eq:app-local-equivalence}
\end{equation}
and both sides vanish with $\lambda(t)$ when $v(t)=0$. Consequently, if $\bar\kappa:=\sup\{\kappa_t:v(t)>0\}<\infty$,
\begin{equation}
    \frac{\Lambda_{\mathrm{SMC}}}{\sqrt{\bar\kappa+3}}\ \le\ \Lambda\ \le\ \frac{\Lambda_{\mathrm{SMC}}}{\sqrt2}.
    \label{eq:app-barrier-comparison}
\end{equation}
\end{proposition}

\begin{proof}
Let $Y:=\Psi_t(X_t,\xi)-\Psi_t(X'_t,\xi')$, a difference of two i.i.d.\ copies, so that $\lambda(t)=\tfrac12\E|Y|$, $\E Y=0$, $\E Y^2=2v(t)$, and
\begin{equation*}
    \E Y^4=2\,\E\big[(\Psi_t-\E\Psi_t)^4\big]+6\,v(t)^2=(2\kappa_t+6)\,v(t)^2 .
\end{equation*}

\emph{Upper bound.} By Cauchy--Schwarz, $\E|Y|\le\sqrt{\E Y^2}=\sqrt{2v(t)}$, so $\lambda(t)\le\sqrt{v(t)/2}$.

\emph{Lower bound.} Write $Y^2=|Y|^{2/3}\,|Y|^{4/3}$ and apply H\"older's inequality with exponents $3/2$ and $3$:
\begin{equation*}
    \E Y^2\ \le\ \big(\E|Y|\big)^{2/3}\big(\E Y^4\big)^{1/3}.
\end{equation*}
Rearranging and substituting the moments,
\begin{equation*}
    \E|Y|\ \ge\ \frac{(\E Y^2)^{3/2}}{(\E Y^4)^{1/2}}
    =\frac{(2v(t))^{3/2}}{\sqrt{2\kappa_t+6}\;v(t)}
    =\frac{2\sqrt{v(t)}}{\sqrt{\kappa_t+3}},
\end{equation*}
so $\lambda(t)\ge\sqrt{v(t)}/\sqrt{\kappa_t+3}$. When $v(t)=0$, $\Psi_t(X_t,\xi)$ is almost surely constant and $\lambda(t)=0=\sqrt{v(t)}$.

Integrating \cref{eq:app-local-equivalence} over $t$ gives \cref{eq:app-barrier-comparison}, with $\Lambda_{\mathrm{SMC}}=\int_0^1\sqrt{v(t)}\,\dd t$.
\end{proof}

Both barriers vanish on the Doob path with the twisted proposal, where $\psi_t\equiv0$, and away from it they are comparable: the two samplers pay for the same mismatch in different currencies, parallel control through rejected exchanges that slow restarts, sequential control through weight variance that must be absorbed by particles. The kurtosis bound $\bar\kappa<\infty$ excludes only the degenerate case in which $v(t)$ vanishes at some times but not identically while $\Psi_t$ has heavy tails there. It holds, for instance, whenever $v$ is bounded away from zero on $[0,1]$, the analogue of the non-degeneracy assumption of \citet{syed_optimized_2026}.

\subsection{A Closed-Form Example}
\label{app:example}

Take a variance-preserving diffusion with forward rate $\gamma_t$ and signal coefficient $m_t^2=e^{-\int_0^t\gamma}$, with $p_0=N(0,1)$ so that $p_t=N(0,1)$ for all $t$, exact Gaussian reverse kernels with $\sigma_t^2=\gamma_t$, and the linear reward $r(x)=\kappa x$. Take the Doob path scaled by a fixed $\beta\in(0,1]$, $G_t=(G^*_t)^{\beta}$. Since $X_0\mid X_t=x\sim N(m_tx,\,1-m_t^2)$, with $s_t^2:=1-m_t^2$,
\begin{equation*}
    G^*_t(x)=\E\big[e^{\kappa X_0}\mid X_t=x\big]=\exp\{\kappa m_tx+\tfrac12\kappa^2s_t^2\},
    \qquad
    \log G_t(x)=\beta\kappa m_tx+\text{const},
    \qquad
    \pi_t=N(\beta\kappa m_t,\,1).
\end{equation*}

\paragraph{Unguided.} For $t'<t$ the reverse kernel is
\begin{equation*}
    X_{t'}\mid X_t\ \sim\ N\Big(\frac{m_t}{m_{t'}}X_t,\ 1-\frac{m_t^2}{m_{t'}^2}\Big),
\end{equation*}
so $m_{t'}X_{t'}-m_tX_t$ is $N(0,\,m_{t'}^2-m_t^2)$ given $X_t$, whatever its tilt. The log-weight of the exchange from $t$ to $t'$ is $\beta\kappa(m_{t'}X_{t'}-m_tX_t)$ up to a constant, hence exactly $N\big(0,\beta^2\kappa^2(m_{t'}^2-m_t^2)\big)$, and by \cref{lem:erf}
\begin{equation}
    r(t\to t')=\erf\Big(\frac{\beta\kappa\sqrt{m_{t'}^2-m_t^2}}{2}\Big).
\end{equation}
Variances add over sub-intervals, with total $\beta^2\kappa^2$ from $t=1$ ($m_1\approx0$) to $t=0$ ($m_0=1$), so \cref{prop:spacing} applies with $S=\beta^2\kappa^2(m_{t'}^2-m_t^2)$ on a stretch $[t',t]$. For small $\Delta t$, $m_{t-\Delta t}^2-m_t^2\approx\gamma_tm_t^2\Delta t$, and the rejection rate is $\approx\beta\kappa m_t\sqrt{\gamma_t\Delta t}/\sqrt\pi$, which matches $\lambda^{\mathrm u}(t)\sqrt{\Delta t}$: here $\|g_t\|=\beta\kappa m_t\sqrt{\gamma_t}$ is constant in $x$, so \cref{eq:local-barriers} gives $\lambda^{\mathrm u}(t)=\sqrt{2}\,\|g_t\|/\sqrt{2\pi}=\|g_t\|/\sqrt\pi$.

\paragraph{Guided.} $\log G_t$ is affine in $x$, so $M_t=0$, the guided step coincides with the twisted proposal, and $\psi_t$ is constant in $x$ because $G_t\propto G^*_t$. Hence $\lambda=\lambda^{\mathrm{tw}}\equiv0$ and every exchange is accepted, whatever the schedule. For $\beta=1$ the guided step is the Euler--Maruyama discretisation of the exact Doob transform \cref{eq:app-htransform}. The gap between zero and infinity here is a degenerate case of \cref{thm:barrier}, but it isolates the mechanism: with the gradient of the potential constant in $x$, there is nothing for a guided exchange to pay for, and everything for an unguided one.

%% file: appendices/experiments.tex
\section{Experimental Details}
\label{app:experiments}
Experiments use NVIDIA GeForce RTX 2080 Ti and NVIDIA A100 GPUs, and Intel Core i9-10900X CPUs at 3.70 GHz. Synthetic sampling uses CPUs; molecular and image models use GPUs.

\subsection{Synthetic Reward Tilting with Sampler Perturbations}
\label{app:gmm}

\paragraph{Target.}
We use a six-component 1D Gaussian mixture and the 40-component 2D FAB mixture of \citet{midgley_flow_2023}. The 1D means, standard deviations and weights are
\begin{align*}
\mu&=(-1.25,-0.58,-0.16,0.38,0.74,1.12),\\
s&=(0.070,0.110,0.080,0.065,0.055,0.075),\\\
\pi&=(0.04,0.09,0.13,0.12,0.28,0.34).
\end{align*}
The 2D benchmark has uniform weights and isotropic component standard deviation \(1.3132616\). For mixture responsibilities \(\gamma_k(x)\), the smooth reward is
\(r(x)=\log\sum_k\gamma_k(x)e^{u_k}\), where
\(u_k=a\,v^\top\mu_k/\max_j|v^\top\mu_j|\).
We use \((v,a,\lambda)=(1,-1,3)\) in 1D and \(((1,-1),1,4)\) in 2D, with \(g(x)=e^{\lambda r(x)}\).
The discontinuous 2D reward is the indicator that the most responsible component belongs to \(\{9,11,20,24,26,35,37,38\}\), with \(\lambda=4\).

\paragraph{Samplers.}
Exact mixture scores drive 200 stochastic reverse transitions on a Karras schedule with \((\sigma_{\min},\sigma_{\max},\rho)=(0.002,80,7)\) \citep{karras_elucidating_2022}:
\begin{equation}
x_{\mathrm{lo}}=x_{\mathrm{hi}}+\Delta\nabla\log p_{\sigma_{\mathrm{hi}}}(x_{\mathrm{hi}})
 +\sqrt{c(\sigma_{\mathrm{hi}})\Delta}\,\epsilon,
\quad \Delta=\sigma_{\mathrm{hi}}^2-\sigma_{\mathrm{lo}}^2,\quad\epsilon\sim\mathcal N(0,I).
\label{eq:gmm-executed-kernel}
\end{equation}
Initialization uses the exact noised mixture distribution, and the endpoint map is the mixture posterior mean.
Nominal sampling uses \(c=1\).
For the perturbed sampler, we set \(c=0.01\) above noise thresholds \(0.3\) in 1D and \(5.4341856\) in 2D, and \(c=1\) otherwise, while leaving the drift unchanged.
Each sampler kernel \(K\) induces its own endpoint law \(p_K\) and hence its own target,
\begin{equation}
    q_K(\dd x)\propto g(x)\,p_K(\dd x).
    \label{eq:gmm-kernel-target}
\end{equation}
The nominal and perturbed samplers have different endpoint targets despite sharing the same reward.

\paragraph{Methods.}
ENCORE and CREPE \citep{he_crepe_2026} use 51 uniformly spaced replica levels, four independent ladders per seed, alternating exchanges, reference refreshes, and no within-level moves.
The potentials are \(G_i=g(\widehat x_0(x_i))^{(1-i/200)^5}\).
Each ladder discards 4,000 half-sweeps and retains the subsequent 20,480 states, yielding 81,920 outputs per seed; rejected and inactive states are retained.
ENCORE uses the reverse transition kernels as proposals, \(L_n=K_n\).
CREPE resolves every transition between adjacent replicas and applies its acceptance test at each exchange.
TDS \citep{wu_practical_2023} and FKD-Diff \citep{singhal_general_2025} use 4, 32, or 256 particles.
TDS retains the terminal importance weights, while FKD-Diff performs adaptive resampling at steps 40, 80, 120, 160, and 200 whenever the ESS falls below half the particle population.
SVDD-MC \citep{li_derivativefree_2025} uses ten candidates, a width-64 value network, 4,096 training rollouts, 2,048 validation rollouts, and 1,000 Adam steps with learning rate \(10^{-3}\).
Uncontrolled sampling and SNIS use the same base draws, with SNIS assigning weights proportional to \(g\).
All methods use the same endpoint map.

\paragraph{Evaluation.}
We use eight seeds and compare methods either at 81,920 outputs per method or at forward-evaluation budgets of 5.12, 6.4, 12.8, and 25.6 million, including initialization, warm-up, endpoint decoding, and value-network training rollouts.
Accuracy is measured by weighted empirical \(\mathrm W_2\) in 1D and sliced \(\mathrm W_2\) over 32 fixed directions in 2D.
Metrics are computed directly from the empirical samples, with weighted methods evaluated as weighted empirical measures.
We report means and sample standard deviations across seeds, pooling ladders within each seed.

\subsubsection{Additional Results}
\label{app:gmm-additional}

\Cref{fig:gmm-robustness-all} shows the 2D densities pooled across eight seeds.
\Cref{tab:gmm-baselines} compares methods at a common output count, while \cref{tab:gmm-model-budget-1d_smooth,tab:gmm-model-budget-fab_smooth,tab:gmm-model-budget-fab_discontinuous} report accuracy under fixed forward-evaluation budgets.
\Cref{tab:gmm-controls} varies the global variance multiplier and the number of uniformly spaced replicas while retaining the 200-step generator and all intervening transitions.
Changing the variance modifies the induced target, whereas changing only the number of replicas leaves it unchanged.
The tables report \(\mathrm W_2\) in 1D and sliced \(\mathrm W_2\) in 2D; parentheses indicate particle counts.
SNIS and TDS are evaluated using their retained importance weights.
For the zero-reward control, unguided ENCORE has exactly zero rejection, as predicted by \cref{app:rejection}.

\begin{figure}[t]
\centering
\includegraphics[width=\linewidth]{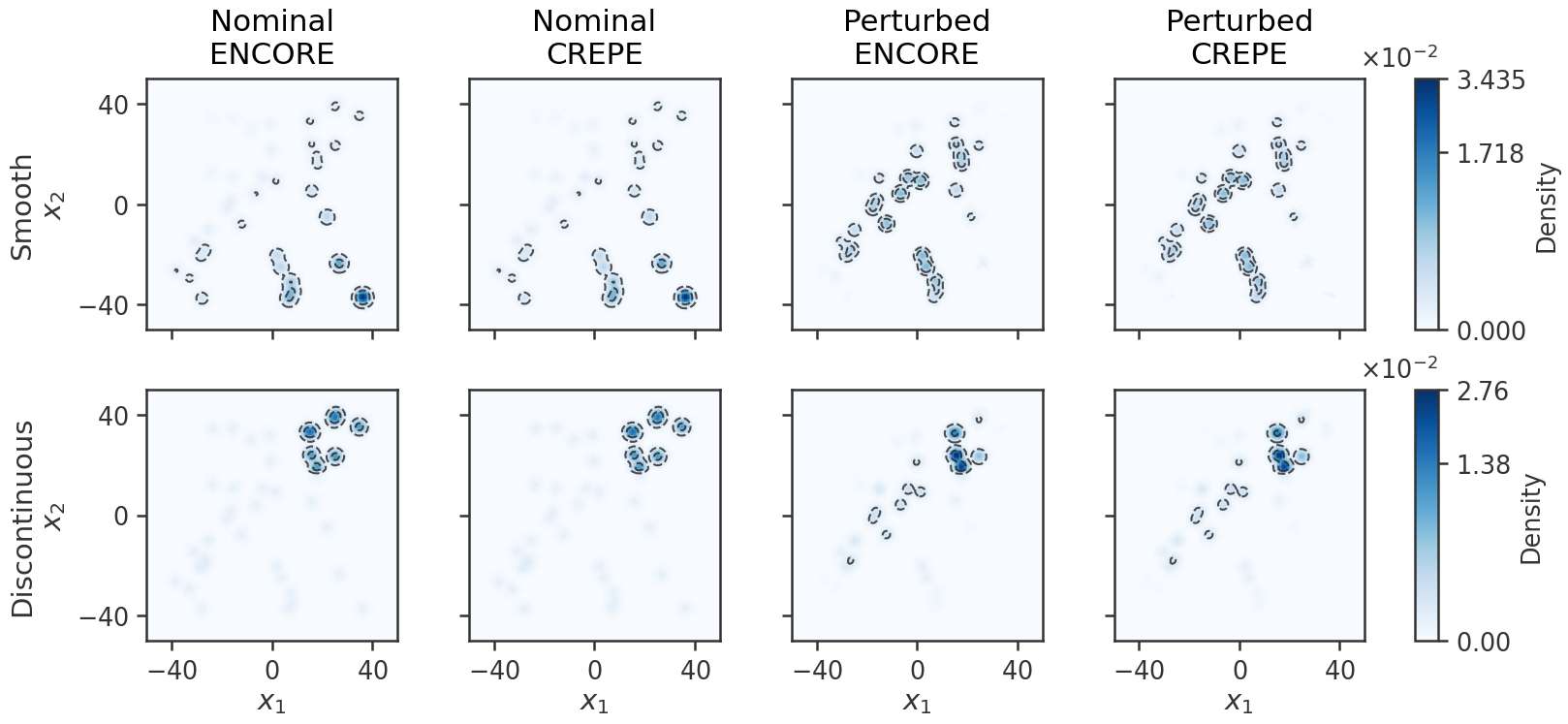}
\caption{2D densities for smooth (top) and discontinuous (bottom) rewards.}
\label{fig:gmm-robustness-all}
\end{figure}

\paragraph{Exchange rejection.}
\label{app:gmm-rejection-completed}
For post-warm-up attempted and rejected swaps \(A_n,D_n\) between adjacent replicas \(n-1,n\), we report the rejection percentage \(100\sum_nD_n/\sum_nA_n\) and summed rejection \(S_R=\sum_{n=1}^{R-1}D_n/A_n\).
The 51-replica ladder has 50 adjacent replica pairs, with 40,960 attempts per pair per seed pooled across four ladders; hence its rejection percentage is \(100S_R/50\). Warm-up and unconditional reference refreshes are excluded.
\Cref{tab:gmm-rejection-published} reports both quantities over eight seeds. The sum is dimensionless, can exceed one, and does not establish target accuracy.

\begin{table}[t]
\centering
\caption{Post-warm-up exchange rejection on the 51-replica ladder (mean \(\pm\) standard deviation over eight seeds).}
\label{tab:gmm-rejection-published}
\begingroup\small
\enctablesetup[4pt]
\enctablefit{\textwidth}{%
\begin{tabular}{@{}lcccc@{}}
\toprule
& \multicolumn{2}{c}{ENCORE} & \multicolumn{2}{c}{CREPE} \\
\cmidrule(lr){2-3}\cmidrule(lr){4-5}
Setting & Rejection (\%) & Summed rejection & Rejection (\%) & Summed rejection \\
\midrule
1D smooth / nominal & \enctablevalue{2.5481628418}{0.00658401271851}{2}{3} & \enctablevalue{1.2740814209}{0.00329200635925}{3}{4} & \enctablevalue{2.54375}{0.00788129759718}{2}{3} & \enctablevalue{1.271875}{0.00394064879859}{3}{4} \\
1D smooth / perturbed & \enctablevalue{0.999761962891}{0.00861589377931}{2}{3} & \enctablevalue{0.499880981445}{0.00430794688965}{3}{4} & \enctablevalue{1.81073608398}{0.0252124537966}{2}{3} & \enctablevalue{0.905368041992}{0.0126062268983}{3}{4} \\
\addlinespace[2pt]
2D smooth / nominal & \enctablevalue{1.87298583984}{0.00862199180381}{2}{3} & \enctablevalue{0.936492919922}{0.0043109959019}{3}{4} & \enctablevalue{1.87424316406}{0.00732282352782}{2}{3} & \enctablevalue{0.937121582031}{0.00366141176391}{3}{4} \\
2D smooth / perturbed & \enctablevalue{0.857525634766}{0.0156524077227}{2}{3} & \enctablevalue{0.428762817383}{0.00782620386133}{3}{4} & \enctablevalue{0.865454101562}{0.0139763509346}{2}{3} & \enctablevalue{0.432727050781}{0.00698817546728}{3}{4} \\
\addlinespace[2pt]
2D discontinuous / nominal & \enctablevalue{1.51207885742}{0.00500713231377}{2}{3} & \enctablevalue{0.756039428711}{0.00250356615689}{3}{4} & \enctablevalue{1.5127746582}{0.00493767570059}{2}{3} & \enctablevalue{0.756387329102}{0.00246883785029}{3}{4} \\
2D discontinuous / perturbed & \enctablevalue{1.58567504883}{0.00824106744737}{2}{3} & \enctablevalue{0.792837524414}{0.00412053372369}{3}{4} & \enctablevalue{1.58931884766}{0.00596490201548}{2}{3} & \enctablevalue{0.794659423828}{0.00298245100774}{3}{4} \\
\addlinespace[2pt]
\bottomrule
\end{tabular}}
\endgroup
\end{table}

\begin{table}[t]
\centering
\caption{Synthetic reward-tilting accuracy at 81,920 outputs per seed; parentheses indicate particle counts (mean \(\pm\) standard deviation over eight seeds).}
\label{tab:gmm-baselines}
\begingroup\small
\enctablesetup[4pt]
\enctablefit{\textwidth}{%
\begin{tabular}{@{}lcccccc@{}}
\toprule
& \multicolumn{2}{c}{1D smooth} & \multicolumn{2}{c}{2D smooth} & \multicolumn{2}{c}{2D discontinuous} \\
& \multicolumn{2}{c}{\(\mathrm W_2\downarrow\)} & \multicolumn{2}{c}{\(\mathrm{SW}_2\downarrow\)} & \multicolumn{2}{c}{\(\mathrm{SW}_2\downarrow\)} \\
\cmidrule(lr){2-3}\cmidrule(lr){4-5}\cmidrule(lr){6-7}
Method & Nominal & Perturbed & Nominal & Perturbed & Nominal & Perturbed \\
\midrule
ENCORE & \enctablevalue{0.0242212306106}{0.0113168302258}{3}{4} & \enctablevalue{0.0146396278041}{0.00685338485698}{3}{4} & \enctablevalue{0.647732623485}{0.119491754045}{3}{4} & \enctablevalue{0.458322208011}{0.185435340936}{3}{4} & \enctablevalue{0.341233515021}{0.101808168802}{3}{4} & \enctablevalue{0.402047540402}{0.128316559665}{3}{4} \\
CREPE & \enctablevalue{0.0293356631752}{0.00897762996816}{3}{4} & \enctablevalue{0.649713933221}{0.033133133403}{3}{4} & \enctablevalue{0.700024258223}{0.305105455833}{3}{4} & \enctablevalue{0.609243652933}{0.247348177649}{3}{4} & \enctablevalue{0.344645137453}{0.0924094038486}{3}{4} & \enctablevalue{0.513371518272}{0.115480243227}{3}{4} \\
TDS (4) & \enctablevalue{0.263207313934}{0.00766461032876}{3}{4} & \enctablevalue{0.0456583920283}{0.0014705677437}{3}{4} & \enctablevalue{11.5229153535}{0.16629277485}{3}{4} & \enctablevalue{5.20276382828}{0.0690064545341}{3}{4} & \enctablevalue{11.5425004112}{0.0898084641851}{3}{4} & \enctablevalue{14.4872633764}{0.0784188856766}{3}{4} \\
TDS (32) & \enctablevalue{0.0974864984693}{0.00624874000234}{3}{4} & \enctablevalue{0.0259626890598}{0.00404045917245}{3}{4} & \enctablevalue{6.9905145182}{0.127657046606}{3}{4} & \enctablevalue{1.74438143347}{0.106435165213}{3}{4} & \enctablevalue{7.980705126}{0.178795927807}{3}{4} & \enctablevalue{5.69461471906}{0.198446193359}{3}{4} \\
TDS (256) & \enctablevalue{0.0637247211138}{0.00851664701052}{3}{4} & \enctablevalue{0.015647703803}{0.00730369772694}{3}{4} & \enctablevalue{6.05783646983}{0.20628103266}{3}{4} & \enctablevalue{0.524998642516}{0.112438647186}{3}{4} & \enctablevalue{1.80416986042}{0.272984606633}{3}{4} & \enctablevalue{1.65351447057}{0.265066298357}{3}{4} \\
FKD (4) & \enctablevalue{0.973773936541}{0.00447782954023}{3}{4} & \enctablevalue{0.178290664659}{0.00124320402859}{3}{4} & \enctablevalue{21.4612569898}{0.104581134863}{3}{4} & \enctablevalue{7.92491563756}{0.024673100264}{3}{4} & \enctablevalue{18.2725946559}{0.086220127723}{3}{4} & \enctablevalue{15.17162585}{0.0816272170406}{3}{4} \\
FKD (32) & \enctablevalue{0.192207462356}{0.00679160714539}{3}{4} & \enctablevalue{0.15447435937}{0.00242969674069}{3}{4} & \enctablevalue{5.61107295627}{0.162412424591}{3}{4} & \enctablevalue{8.03554064406}{0.035197956063}{3}{4} & \enctablevalue{4.82624131837}{0.0901962629918}{3}{4} & \enctablevalue{2.09382976312}{0.0689838832104}{3}{4} \\
FKD (256) & \enctablevalue{0.0568028979028}{0.0059672127745}{3}{4} & \enctablevalue{0.178871027456}{0.00263091347871}{3}{4} & \enctablevalue{1.749809971}{0.150181044518}{3}{4} & \enctablevalue{8.29075470478}{0.028835583363}{3}{4} & \enctablevalue{3.2361716363}{0.110182607107}{3}{4} & \enctablevalue{3.02822114148}{0.101010042568}{3}{4} \\
SVDD-MC & \enctablevalue{0.3368531738}{0.0605936742026}{3}{4} & \enctablevalue{0.154085948281}{0.0148950058472}{3}{4} & \enctablevalue{10.961043415}{1.72392438502}{3}{4} & \enctablevalue{7.55424721233}{0.0439477564728}{3}{4} & \enctablevalue{15.0623827286}{0.732298301029}{3}{4} & \enctablevalue{19.9249291255}{0.0915884261109}{3}{4} \\
SNIS & \enctablevalue{0.0116852805596}{0.0054987982129}{3}{4} & \enctablevalue{0.00654807084818}{0.00330529677582}{3}{4} & \enctablevalue{0.415510970806}{0.202932654616}{3}{4} & \enctablevalue{0.324693871062}{0.162542378746}{3}{4} & \enctablevalue{0.154881752841}{0.0577338100942}{3}{4} & \enctablevalue{0.178858162139}{0.0692937598898}{3}{4} \\
Uncontrolled & \enctablevalue{1.38630514817}{0.00179777299422}{3}{4} & \enctablevalue{0.187167203242}{0.000904588063211}{3}{4} & \enctablevalue{28.8194822518}{0.0416643755546}{3}{4} & \enctablevalue{8.37766207295}{0.0336858301527}{3}{4} & \enctablevalue{28.782297058}{0.0349814974131}{3}{4} & \enctablevalue{20.5062342562}{0.0694362291351}{3}{4} \\
\bottomrule
\end{tabular}}
\endgroup
\end{table}

\begin{table}[t]
\centering
\caption{1D smooth reward-tilting accuracy at fixed forward-evaluation budgets (mean \(\pm\) standard deviation over eight seeds).}
\label{tab:gmm-model-budget-1d_smooth}
\begingroup\small
\enctablesetup[4pt]
\enctablefit{\textwidth}{%
\begin{tabular}{@{}lcccc@{}}
\toprule
& \multicolumn{4}{c}{\(\mathrm W_2\downarrow\)} \\
Forward evaluations & 5,120,000 & 6,400,000 & 12,800,000 & 25,600,000 \\
\midrule
\multicolumn{5}{l}{\textit{Nominal sampler}} \\
ENCORE & \enctablevalue{0.0470352759216}{0.0263122416943}{3}{4} & \enctablevalue{0.040736197569}{0.0224396939198}{3}{4} & \enctablevalue{0.0260534877148}{0.0180826988234}{3}{4} & \enctablevalue{0.0237828518129}{0.00853686919154}{3}{4} \\
CREPE & \enctablevalue{0.10280481496}{0.0550702895933}{3}{4} & \enctablevalue{0.0626353740801}{0.0160715677808}{3}{4} & \enctablevalue{0.0349699175992}{0.00965832508101}{3}{4} & \enctablevalue{0.0291353278681}{0.0104563570473}{3}{4} \\
TDS (4) & \enctablevalue{0.271553556205}{0.0116405129518}{3}{4} & \enctablevalue{0.271929921049}{0.0110152354266}{3}{4} & \enctablevalue{0.264412440927}{0.0134800795825}{3}{4} & \enctablevalue{0.264046546743}{0.010123644956}{3}{4} \\
TDS (32) & \enctablevalue{0.106655845178}{0.0258000367979}{3}{4} & \enctablevalue{0.105834064415}{0.0261048186003}{3}{4} & \enctablevalue{0.100253386524}{0.0159991714559}{3}{4} & \enctablevalue{0.0982193397558}{0.0087540563987}{3}{4} \\
TDS (256) & \enctablevalue{0.0621409558946}{0.0111809819613}{3}{4} & \enctablevalue{0.0653804959818}{0.0100520219586}{3}{4} & \enctablevalue{0.0597248232748}{0.00939875652851}{3}{4} & \enctablevalue{0.062627193335}{0.0110933759166}{3}{4} \\
FKD (4) & \enctablevalue{0.971120966481}{0.00835570374496}{3}{4} & \enctablevalue{0.971004301244}{0.00710782443496}{3}{4} & \enctablevalue{0.973243482328}{0.00561354013219}{3}{4} & \enctablevalue{0.974571513798}{0.00366334106589}{3}{4} \\
FKD (32) & \enctablevalue{0.196111054017}{0.0084054022242}{3}{4} & \enctablevalue{0.193974817875}{0.00788942037967}{3}{4} & \enctablevalue{0.193059388375}{0.00720120360973}{3}{4} & \enctablevalue{0.192788234451}{0.00625652067353}{3}{4} \\
FKD (256) & \enctablevalue{0.0596241922145}{0.0121417900963}{3}{4} & \enctablevalue{0.0572714096618}{0.0118522995285}{3}{4} & \enctablevalue{0.0564354231348}{0.00821982032964}{3}{4} & \enctablevalue{0.0575783212315}{0.00385426902056}{3}{4} \\
SVDD-MC & \enctablevalue{0.337002910479}{0.062169097318}{3}{4} & \enctablevalue{0.331563027461}{0.0602971407903}{3}{4} & \enctablevalue{0.334513850345}{0.0616977544082}{3}{4} & \enctablevalue{0.336918468634}{0.0616597509152}{3}{4} \\
SNIS & \enctablevalue{0.0146632932448}{0.0084688130092}{3}{4} & \enctablevalue{0.0206438771359}{0.00659818307784}{3}{4} & \enctablevalue{0.0143066221043}{0.0101653477724}{3}{4} & \enctablevalue{0.00943859144605}{0.00297944363487}{3}{4} \\
Uncontrolled & \enctablevalue{1.38640509446}{0.00455231937188}{3}{4} & \enctablevalue{1.38765516922}{0.00275311904329}{3}{4} & \enctablevalue{1.38692606824}{0.00221734772131}{3}{4} & \enctablevalue{1.38689779507}{0.00167231577569}{3}{4} \\
\midrule
\multicolumn{5}{l}{\textit{Perturbed sampler}} \\
ENCORE & \enctablevalue{0.0306746337461}{0.0192067918287}{3}{4} & \enctablevalue{0.0265599365252}{0.0130294816294}{3}{4} & \enctablevalue{0.0182917478744}{0.00804039574943}{3}{4} & \enctablevalue{0.0129887571469}{0.00608665361598}{3}{4} \\
CREPE & \enctablevalue{0.592613754228}{0.100727506869}{3}{4} & \enctablevalue{0.655788564403}{0.0893244321972}{3}{4} & \enctablevalue{0.661625277722}{0.0487770710815}{3}{4} & \enctablevalue{0.647471179367}{0.0369895894683}{3}{4} \\
TDS (4) & \enctablevalue{0.0476862378749}{0.00488237331387}{3}{4} & \enctablevalue{0.0465987826686}{0.00433697617984}{3}{4} & \enctablevalue{0.0463441954699}{0.00270279799593}{3}{4} & \enctablevalue{0.0459403793245}{0.00226140033465}{3}{4} \\
TDS (32) & \enctablevalue{0.0273110236833}{0.0055328944844}{3}{4} & \enctablevalue{0.0284668196965}{0.00577193959177}{3}{4} & \enctablevalue{0.0260307936197}{0.00600162236801}{3}{4} & \enctablevalue{0.0263707072321}{0.00378542491394}{3}{4} \\
TDS (256) & \enctablevalue{0.0240093748034}{0.0103965555063}{3}{4} & \enctablevalue{0.0214035447511}{0.0096138337878}{3}{4} & \enctablevalue{0.0174225688005}{0.00586840817534}{3}{4} & \enctablevalue{0.014740846516}{0.00757432974323}{3}{4} \\
FKD (4) & \enctablevalue{0.178119910189}{0.00300326735385}{3}{4} & \enctablevalue{0.177874893962}{0.00251230051383}{3}{4} & \enctablevalue{0.178215432026}{0.00157022914666}{3}{4} & \enctablevalue{0.178331185424}{0.00110792526947}{3}{4} \\
FKD (32) & \enctablevalue{0.155301515032}{0.00471052894243}{3}{4} & \enctablevalue{0.155198637655}{0.0038009878233}{3}{4} & \enctablevalue{0.15471169323}{0.00213742836372}{3}{4} & \enctablevalue{0.154339799026}{0.00200040738008}{3}{4} \\
FKD (256) & \enctablevalue{0.17780971654}{0.00500869699901}{3}{4} & \enctablevalue{0.178239434387}{0.00435927681588}{3}{4} & \enctablevalue{0.178337398562}{0.00288813356097}{3}{4} & \enctablevalue{0.179035108142}{0.0021015209379}{3}{4} \\
SVDD-MC & \enctablevalue{0.154804488203}{0.0133217178107}{3}{4} & \enctablevalue{0.152940424495}{0.0138066308665}{3}{4} & \enctablevalue{0.15486899178}{0.0149631413817}{3}{4} & \enctablevalue{0.154832470721}{0.0142757397797}{3}{4} \\
SNIS & \enctablevalue{0.0118218344045}{0.00516375927472}{3}{4} & \enctablevalue{0.0112802885378}{0.00587836206197}{3}{4} & \enctablevalue{0.0077916145005}{0.00414797831048}{3}{4} & \enctablevalue{0.00678997265121}{0.00201656303545}{3}{4} \\
Uncontrolled & \enctablevalue{0.187360640587}{0.00180718846644}{3}{4} & \enctablevalue{0.187406569069}{0.00148657337935}{3}{4} & \enctablevalue{0.187338126036}{0.00106158205193}{3}{4} & \enctablevalue{0.187351621909}{0.000751995224404}{3}{4} \\
\bottomrule
\end{tabular}}
\endgroup
\end{table}

\begin{table}[t]
\centering
\caption{2D smooth reward-tilting accuracy at fixed forward-evaluation budgets (mean \(\pm\) standard deviation over eight seeds).}
\label{tab:gmm-model-budget-fab_smooth}
\begingroup\small
\enctablesetup[4pt]
\enctablefit{\textwidth}{%
\begin{tabular}{@{}lcccc@{}}
\toprule
& \multicolumn{4}{c}{\(\mathrm{SW}_2\downarrow\)} \\
Forward evaluations & 5,120,000 & 6,400,000 & 12,800,000 & 25,600,000 \\
\midrule
\multicolumn{5}{l}{\textit{Nominal sampler}} \\
ENCORE & \enctablevalue{1.10473640983}{0.361132065462}{3}{4} & \enctablevalue{1.2355124518}{0.372078800656}{3}{4} & \enctablevalue{0.855437381414}{0.167376064818}{3}{4} & \enctablevalue{0.484774889023}{0.204200851742}{3}{4} \\
CREPE & \enctablevalue{3.31904917494}{0.935731404417}{3}{4} & \enctablevalue{1.98225129698}{0.674778645492}{3}{4} & \enctablevalue{1.21667459431}{0.413246949187}{3}{4} & \enctablevalue{0.81355575947}{0.344816689595}{3}{4} \\
TDS (4) & \enctablevalue{11.3668936728}{0.346370538811}{3}{4} & \enctablevalue{11.4240951672}{0.293009831218}{3}{4} & \enctablevalue{11.4358482827}{0.212532284782}{3}{4} & \enctablevalue{11.4927758996}{0.197089700585}{3}{4} \\
TDS (32) & \enctablevalue{7.00142391502}{0.195666390819}{3}{4} & \enctablevalue{7.07651036966}{0.312280826586}{3}{4} & \enctablevalue{7.03385255412}{0.197998180443}{3}{4} & \enctablevalue{6.9667478832}{0.190002625537}{3}{4} \\
TDS (256) & \enctablevalue{6.17877334731}{0.687525467334}{3}{4} & \enctablevalue{6.14707400331}{0.707739681314}{3}{4} & \enctablevalue{5.93238936614}{0.325060678482}{3}{4} & \enctablevalue{6.07696578599}{0.252079540532}{3}{4} \\
FKD (4) & \enctablevalue{21.5559298292}{0.137371187105}{3}{4} & \enctablevalue{21.5200666709}{0.130791924785}{3}{4} & \enctablevalue{21.4632501167}{0.0985643893725}{3}{4} & \enctablevalue{21.4248715292}{0.0873271583858}{3}{4} \\
FKD (32) & \enctablevalue{5.57434441795}{0.309076193848}{3}{4} & \enctablevalue{5.59858091239}{0.269794171184}{3}{4} & \enctablevalue{5.58133289279}{0.190648401275}{3}{4} & \enctablevalue{5.60974771318}{0.115888089451}{3}{4} \\
FKD (256) & \enctablevalue{1.9671141637}{0.303400415919}{3}{4} & \enctablevalue{1.96562531137}{0.227156317371}{3}{4} & \enctablevalue{1.78750883494}{0.212840121019}{3}{4} & \enctablevalue{1.70671128146}{0.115019381557}{3}{4} \\
SVDD-MC & \enctablevalue{10.9730820919}{1.88653952864}{3}{4} & \enctablevalue{10.9883729442}{1.83960817569}{3}{4} & \enctablevalue{10.8674622805}{1.793934413}{3}{4} & \enctablevalue{10.936938778}{1.78366022654}{3}{4} \\
SNIS & \enctablevalue{0.50338124547}{0.187397265865}{3}{4} & \enctablevalue{0.546062056498}{0.209704937347}{3}{4} & \enctablevalue{0.425778261293}{0.162440671022}{3}{4} & \enctablevalue{0.292030563833}{0.125241159084}{3}{4} \\
Uncontrolled & \enctablevalue{28.7996386819}{0.0522924732485}{3}{4} & \enctablevalue{28.8030508299}{0.0390497588129}{3}{4} & \enctablevalue{28.822425775}{0.0467113834275}{3}{4} & \enctablevalue{28.8206491637}{0.0176694981082}{3}{4} \\
\midrule
\multicolumn{5}{l}{\textit{Perturbed sampler}} \\
ENCORE & \enctablevalue{1.10433430638}{0.409055062286}{3}{4} & \enctablevalue{1.09009315702}{0.396501233108}{3}{4} & \enctablevalue{0.554839660881}{0.272556876809}{3}{4} & \enctablevalue{0.436359064673}{0.20498524834}{3}{4} \\
CREPE & \enctablevalue{2.54637704221}{0.504375767077}{3}{4} & \enctablevalue{1.28403083231}{0.548352272614}{3}{4} & \enctablevalue{0.789737036646}{0.373185715416}{3}{4} & \enctablevalue{0.629975563541}{0.301212446439}{3}{4} \\
TDS (4) & \enctablevalue{5.19500853375}{0.153973160147}{3}{4} & \enctablevalue{5.2254601827}{0.146932929629}{3}{4} & \enctablevalue{5.22302431927}{0.0823102224816}{3}{4} & \enctablevalue{5.19647942101}{0.0632958432643}{3}{4} \\
TDS (32) & \enctablevalue{1.86014103355}{0.321869610091}{3}{4} & \enctablevalue{1.92087473848}{0.242517546032}{3}{4} & \enctablevalue{1.82494296783}{0.1825961249}{3}{4} & \enctablevalue{1.70930065234}{0.121362175252}{3}{4} \\
TDS (256) & \enctablevalue{0.836835091011}{0.260703785621}{3}{4} & \enctablevalue{0.75227460363}{0.165882130282}{3}{4} & \enctablevalue{0.59063705302}{0.223025083595}{3}{4} & \enctablevalue{0.513957809347}{0.142145514053}{3}{4} \\
FKD (4) & \enctablevalue{7.94660793435}{0.0672219197485}{3}{4} & \enctablevalue{7.92954409132}{0.0568238942907}{3}{4} & \enctablevalue{7.91406092757}{0.0329267691564}{3}{4} & \enctablevalue{7.93219414175}{0.0287367135169}{3}{4} \\
FKD (32) & \enctablevalue{8.05208889328}{0.0614378050083}{3}{4} & \enctablevalue{8.04127700111}{0.0677685522276}{3}{4} & \enctablevalue{8.03452525888}{0.0321709578319}{3}{4} & \enctablevalue{8.0324227832}{0.0302434249435}{3}{4} \\
FKD (256) & \enctablevalue{8.29417874883}{0.0757365760395}{3}{4} & \enctablevalue{8.282620278}{0.0718769867601}{3}{4} & \enctablevalue{8.29267937317}{0.031904968961}{3}{4} & \enctablevalue{8.29420766234}{0.0250254003792}{3}{4} \\
SVDD-MC & \enctablevalue{7.63410349669}{0.19629690211}{3}{4} & \enctablevalue{7.62274570469}{0.199550786078}{3}{4} & \enctablevalue{7.55961621}{0.137392916519}{3}{4} & \enctablevalue{7.5544309728}{0.0959104354444}{3}{4} \\
SNIS & \enctablevalue{0.435468034119}{0.161136995316}{3}{4} & \enctablevalue{0.43467825469}{0.192536917338}{3}{4} & \enctablevalue{0.362775251992}{0.153391686765}{3}{4} & \enctablevalue{0.275463255798}{0.0919592995476}{3}{4} \\
Uncontrolled & \enctablevalue{8.37320907047}{0.0618391315872}{3}{4} & \enctablevalue{8.3798060175}{0.0708039782701}{3}{4} & \enctablevalue{8.37740948168}{0.0449793704032}{3}{4} & \enctablevalue{8.37842147155}{0.0255637805975}{3}{4} \\
\bottomrule
\end{tabular}}
\endgroup
\end{table}

\begin{table}[t]
\centering
\caption{2D discontinuous reward-tilting accuracy at fixed forward-evaluation budgets (mean \(\pm\) standard deviation over eight seeds).}
\label{tab:gmm-model-budget-fab_discontinuous}
\begingroup\small
\enctablesetup[4pt]
\enctablefit{\textwidth}{%
\begin{tabular}{@{}lcccc@{}}
\toprule
& \multicolumn{4}{c}{\(\mathrm{SW}_2\downarrow\)} \\
Forward evaluations & 5,120,000 & 6,400,000 & 12,800,000 & 25,600,000 \\
\midrule
\multicolumn{5}{l}{\textit{Nominal sampler}} \\
ENCORE & \enctablevalue{0.857713682025}{0.400003291759}{3}{4} & \enctablevalue{0.825131974376}{0.263109743549}{3}{4} & \enctablevalue{0.419385297743}{0.105215405341}{3}{4} & \enctablevalue{0.298088079071}{0.0887972616672}{3}{4} \\
CREPE & \enctablevalue{2.24715796507}{0.968577822392}{3}{4} & \enctablevalue{1.46117677967}{0.460981939714}{3}{4} & \enctablevalue{0.639990076905}{0.267890238712}{3}{4} & \enctablevalue{0.37004949117}{0.129218998951}{3}{4} \\
TDS (4) & \enctablevalue{11.5690990086}{0.33864356184}{3}{4} & \enctablevalue{11.5406974273}{0.240097188843}{3}{4} & \enctablevalue{11.5404003852}{0.132776642845}{3}{4} & \enctablevalue{11.5138680634}{0.0797393521771}{3}{4} \\
TDS (32) & \enctablevalue{7.88760890821}{0.563735472028}{3}{4} & \enctablevalue{7.90754070787}{0.549013006654}{3}{4} & \enctablevalue{7.90424675256}{0.426500712651}{3}{4} & \enctablevalue{7.96653315494}{0.219527925024}{3}{4} \\
TDS (256) & \enctablevalue{1.77360620513}{0.622161378451}{3}{4} & \enctablevalue{1.72548526907}{0.393598596018}{3}{4} & \enctablevalue{1.77495902433}{0.447309639364}{3}{4} & \enctablevalue{1.84273622804}{0.241738715119}{3}{4} \\
FKD (4) & \enctablevalue{18.2372961972}{0.111472062429}{3}{4} & \enctablevalue{18.2569354477}{0.088493343263}{3}{4} & \enctablevalue{18.2833982155}{0.113497391447}{3}{4} & \enctablevalue{18.2721408438}{0.025256918048}{3}{4} \\
FKD (32) & \enctablevalue{4.82902661199}{0.154583582891}{3}{4} & \enctablevalue{4.83335771838}{0.130956040031}{3}{4} & \enctablevalue{4.81625231818}{0.106400221021}{3}{4} & \enctablevalue{4.84052340354}{0.0577075771041}{3}{4} \\
FKD (256) & \enctablevalue{3.28321698327}{0.133958785548}{3}{4} & \enctablevalue{3.26790012087}{0.0950456942343}{3}{4} & \enctablevalue{3.24249666553}{0.119313873273}{3}{4} & \enctablevalue{3.1687106357}{0.0967054986296}{3}{4} \\
SVDD-MC & \enctablevalue{14.9149273493}{1.09400194323}{3}{4} & \enctablevalue{14.9699353269}{0.834550413537}{3}{4} & \enctablevalue{15.2160996155}{0.864210211143}{3}{4} & \enctablevalue{15.1027123243}{0.75666666479}{3}{4} \\
SNIS & \enctablevalue{0.242783779367}{0.0809737027497}{3}{4} & \enctablevalue{0.230637285395}{0.061899609439}{3}{4} & \enctablevalue{0.159972315647}{0.0705175073197}{3}{4} & \enctablevalue{0.120277345053}{0.0208933570996}{3}{4} \\
Uncontrolled & \enctablevalue{28.7875977499}{0.0976653462652}{3}{4} & \enctablevalue{28.7773290574}{0.079421011606}{3}{4} & \enctablevalue{28.7776077774}{0.0535508728584}{3}{4} & \enctablevalue{28.7693732829}{0.0340518831277}{3}{4} \\
\midrule
\multicolumn{5}{l}{\textit{Perturbed sampler}} \\
ENCORE & \enctablevalue{0.855541412679}{0.214726669867}{3}{4} & \enctablevalue{0.700938022303}{0.222819319123}{3}{4} & \enctablevalue{0.578927871465}{0.172386501696}{3}{4} & \enctablevalue{0.351443375845}{0.134336112044}{3}{4} \\
CREPE & \enctablevalue{1.75349037316}{0.62210898937}{3}{4} & \enctablevalue{1.22983419068}{0.533405563106}{3}{4} & \enctablevalue{0.70232332586}{0.146999589163}{3}{4} & \enctablevalue{0.529561896416}{0.124178043704}{3}{4} \\
TDS (4) & \enctablevalue{14.5806061284}{0.254218652655}{3}{4} & \enctablevalue{14.5622856645}{0.212494340246}{3}{4} & \enctablevalue{14.4941543853}{0.17620146913}{3}{4} & \enctablevalue{14.4927282545}{0.106377962828}{3}{4} \\
TDS (32) & \enctablevalue{5.82644839792}{0.272659061805}{3}{4} & \enctablevalue{5.88167015231}{0.25382841067}{3}{4} & \enctablevalue{5.7860459889}{0.19903491739}{3}{4} & \enctablevalue{5.73685758922}{0.234462221363}{3}{4} \\
TDS (256) & \enctablevalue{1.72691905935}{0.393114195909}{3}{4} & \enctablevalue{1.64079310651}{0.439064324607}{3}{4} & \enctablevalue{1.70977125956}{0.318716265881}{3}{4} & \enctablevalue{1.6787149736}{0.259424589483}{3}{4} \\
FKD (4) & \enctablevalue{15.1825780072}{0.165220708505}{3}{4} & \enctablevalue{15.1923991455}{0.113424222696}{3}{4} & \enctablevalue{15.1753227842}{0.0904173948111}{3}{4} & \enctablevalue{15.1521048977}{0.0557770740803}{3}{4} \\
FKD (32) & \enctablevalue{2.09551299904}{0.0962531779969}{3}{4} & \enctablevalue{2.10620735059}{0.0891486721342}{3}{4} & \enctablevalue{2.11044308801}{0.086974339802}{3}{4} & \enctablevalue{2.0868414245}{0.0438778500616}{3}{4} \\
FKD (256) & \enctablevalue{3.02776774163}{0.211797496656}{3}{4} & \enctablevalue{3.00362864444}{0.220817695781}{3}{4} & \enctablevalue{3.04972766773}{0.1115073375}{3}{4} & \enctablevalue{3.05080496163}{0.09032283149}{3}{4} \\
SVDD-MC & \enctablevalue{19.7845816281}{0.52259590191}{3}{4} & \enctablevalue{19.7869163536}{0.544341723204}{3}{4} & \enctablevalue{19.9035528193}{0.268977007743}{3}{4} & \enctablevalue{19.9020151383}{0.177966424716}{3}{4} \\
SNIS & \enctablevalue{0.303750582121}{0.125447205553}{3}{4} & \enctablevalue{0.27617198457}{0.132366039044}{3}{4} & \enctablevalue{0.219213428918}{0.0785466624596}{3}{4} & \enctablevalue{0.156951672466}{0.0377690673846}{3}{4} \\
Uncontrolled & \enctablevalue{20.4873916321}{0.1078293781}{3}{4} & \enctablevalue{20.4861436387}{0.0955904511352}{3}{4} & \enctablevalue{20.4951407519}{0.0798005001965}{3}{4} & \enctablevalue{20.4950893526}{0.0465989767972}{3}{4} \\
\bottomrule
\end{tabular}}
\endgroup
\end{table}

\begin{table}[t]
\centering
\caption{Synthetic reward tilting under variance and replica-schedule changes (mean \(\pm\) standard deviation over eight seeds).}
\label{tab:gmm-controls}
\begingroup\small
\enctablesetup[4pt]
\enctablefit{\textwidth}{%
\begin{tabular}{@{}llccc@{}}
\toprule
& & \multicolumn{3}{c}{\(\mathrm W_2\) (1D), \(\mathrm{SW}_2\) (2D) \(\downarrow\)} \\
\cmidrule(lr){3-5}
\multicolumn{2}{l}{Transition-variance multiplier (51 replicas)} & 0.25 & 1 & 4 \\
\midrule
1D smooth & ENCORE & \enctablevalue{0.0501067947435}{0.0175387644683}{3}{4} & \enctablevalue{0.0242212306106}{0.0113168302258}{3}{4} & \enctablevalue{0.00748962481157}{0.00196576321899}{3}{4} \\
 & CREPE & \enctablevalue{0.343995863323}{0.0171995456164}{3}{4} & \enctablevalue{0.0293356631752}{0.00897762996816}{3}{4} & \enctablevalue{0.0560749911826}{0.00587773371164}{3}{4} \\
2D smooth & ENCORE & \enctablevalue{0.984515350993}{0.281184561845}{3}{4} & \enctablevalue{0.647732623485}{0.119491754045}{3}{4} & \enctablevalue{0.473172281062}{0.0921582523462}{3}{4} \\
 & CREPE & \enctablevalue{0.865049288901}{0.237509662475}{3}{4} & \enctablevalue{0.700024258223}{0.305105455833}{3}{4} & \enctablevalue{0.527834341357}{0.181171912086}{3}{4} \\
2D discontinuous & ENCORE & \enctablevalue{0.451969141317}{0.0993036116201}{3}{4} & \enctablevalue{0.341233515021}{0.101808168802}{3}{4} & \enctablevalue{0.276895968835}{0.0671969831321}{3}{4} \\
 & CREPE & \enctablevalue{0.629741646167}{0.146490156194}{3}{4} & \enctablevalue{0.344645137453}{0.0924094038486}{3}{4} & \enctablevalue{0.405657382524}{0.152649931081}{3}{4} \\
\addlinespace[3pt]
\multicolumn{2}{l}{Uniform replica count (nominal sampler)} & 6 & 11 & 51 \\
\midrule
1D smooth & ENCORE & \enctablevalue{0.0308627883579}{0.0127710738795}{3}{4} & \enctablevalue{0.0290406244059}{0.0150464232551}{3}{4} & \enctablevalue{0.0242212306106}{0.0113168302258}{3}{4} \\
 & CREPE & \enctablevalue{0.0247640660041}{0.00905885319469}{3}{4} & \enctablevalue{0.0284155014476}{0.0146730765028}{3}{4} & \enctablevalue{0.0293356631752}{0.00897762996816}{3}{4} \\
2D smooth & ENCORE & \enctablevalue{1.06893545293}{0.357686077017}{3}{4} & \enctablevalue{0.812586513431}{0.275610224542}{3}{4} & \enctablevalue{0.647732623485}{0.119491754045}{3}{4} \\
 & CREPE & \enctablevalue{0.814896591773}{0.448632899739}{3}{4} & \enctablevalue{0.826067713548}{0.39585097127}{3}{4} & \enctablevalue{0.700024258223}{0.305105455833}{3}{4} \\
2D discontinuous & ENCORE & \enctablevalue{0.401110737328}{0.164336773166}{3}{4} & \enctablevalue{0.433746013481}{0.0933742039976}{3}{4} & \enctablevalue{0.341233515021}{0.101808168802}{3}{4} \\
 & CREPE & \enctablevalue{0.514142477743}{0.354298626626}{3}{4} & \enctablevalue{0.444550741698}{0.127920279509}{3}{4} & \enctablevalue{0.344645137453}{0.0924094038486}{3}{4} \\
\addlinespace[3pt]
\multicolumn{2}{l}{Uniform replica count (perturbed sampler)} & 6 & 11 & 51 \\
\midrule
1D smooth & ENCORE & \enctablevalue{0.0121263033065}{0.00466271812911}{3}{4} & \enctablevalue{0.0128132946046}{0.00454860166651}{3}{4} & \enctablevalue{0.0146396278041}{0.00685338485698}{3}{4} \\
 & CREPE & \enctablevalue{0.322698546986}{0.0391825567356}{3}{4} & \enctablevalue{0.473857560458}{0.0473255777236}{3}{4} & \enctablevalue{0.649713933221}{0.033133133403}{3}{4} \\
2D smooth & ENCORE & \enctablevalue{0.422668512569}{0.116938126028}{3}{4} & \enctablevalue{0.392290638843}{0.106232927463}{3}{4} & \enctablevalue{0.458322208011}{0.185435340936}{3}{4} \\
 & CREPE & \enctablevalue{0.467074504696}{0.227010443205}{3}{4} & \enctablevalue{0.516051042941}{0.114499981293}{3}{4} & \enctablevalue{0.609243652933}{0.247348177649}{3}{4} \\
2D discontinuous & ENCORE & \enctablevalue{0.588538672698}{0.253587795484}{3}{4} & \enctablevalue{0.419248108095}{0.165036572305}{3}{4} & \enctablevalue{0.402047540402}{0.128316559665}{3}{4} \\
 & CREPE & \enctablevalue{0.620550041458}{0.275394525964}{3}{4} & \enctablevalue{0.558312590697}{0.149353451407}{3}{4} & \enctablevalue{0.513371518272}{0.115480243227}{3}{4} \\
\addlinespace[3pt]
\bottomrule
\end{tabular}}
\endgroup
\end{table}

\subsection{Sampling Alanine Transition Regions}
\label{app:alanine-transition}

\paragraph{Target.}
We apply a single harmonic umbrella to the 22-atom ALDP system at \(800\,\mathrm K\):
\begin{equation}
G_0(x)=\exp\!\left[-\frac{\widetilde\phi(x)^2}{2(0.2)^2}\right],
\qquad q_{\mathrm{model}}(x)\propto p_0(x)G_0(x),
\label{eq:transition-target}
\end{equation}
where \(\widetilde\phi\) is the torsion wrapped to \([-\pi,\pi)\); \(\psi\) is unrestrained.
We report the band \(|\widetilde\phi|\leq0.4\) rad, split into upper (\(\psi\geq0\)) and lower (\(\psi<0\)) sectors, retaining outside-band samples. This is conformation sampling within a prescribed region, rather than a measurement of transition rates.

\paragraph{Samplers.}
All methods use the ALDP checkpoint from \cref{app:alanine}, 200 Karras reverse transitions, terminal denoising and the same bond-constraint projection.
The shared intermediate log potential is
\(\ell_i(z)=\gamma_i\log G_0(\mathcal R_i(z))\), with projected denoiser readout \(\mathcal R_i\) and
\(\gamma_i=[((1-i/200)-5/7)/(2/7)]_{[0,1]}\).
The model reference weights 262,144 independent model draws by \(G_0\); the physical reference weights 150,000 equilibrium MD frames at \(800\,\mathrm K\).
These estimate the tilted model and physical laws, respectively, and are kept distinct.

\paragraph{Methods.}
ENCORE and CREPE use eight ladders with levels \((200,196,\ldots,0)\) and run for 20,000 half-sweeps, discarding the first 5,000 and yielding 120,000 outputs per seed.
ENCORE uses unguided native reverse proposals, while CREPE uses its fixed-CV noising kernel without temperature scaling.
FKD and TDS use 256 particles per population.
FKD checks for resampling at the replica levels using an ESS threshold of half the population, whereas TDS uses a guidance coefficient of \(1\), systematic resampling at every native step, and the Gaussian proposal correction of \citet{wu_practical_2023}.
Both methods resample at the endpoint.
Their 315 and 337 populations yield 80,640 and 86,272 outputs, respectively.
ENCORE, FKD, and TDS each use approximately 17.34--17.36 million forward function evaluations (NFEs) per seed, including warm-up, while CREPE uses 20.76 million.
TDS additionally requires 4.92 million backward evaluations.

\paragraph{Evaluation.}
We use three seeds.
Regional TVD is the total variation distance between the probabilities assigned to the upper, lower, and out-of-band regions.
Within the band, energy error is measured by the supremum distance between the empirical energy CDF and the model-reference CDF.
Molecular potential energies are evaluated with OpenMM, excluding the umbrella potential and kinetic energy \citep{eastman_openmm_2023}.
All quantitative metrics use the complete set of retained samples.
Density plots show 5,000 samples from a single seed, selected before any conditioning, while reference samples retain their full importance weights.

\subsubsection{Additional Results}

\Cref{fig:transition-conditionals} shows the conditional conformational densities.
\Cref{tab:transition-controls} reports accuracy against the model + IS reference under two forward-NFE budgets, spanning 8.08--10.36 and 17.34--20.76 million evaluations per seed, with exact budgets reported for each method.
The lower- and higher-budget evaluations use nested portions of the same trajectories.

\begin{figure}[t]
\centering
\includegraphics[width=\linewidth]{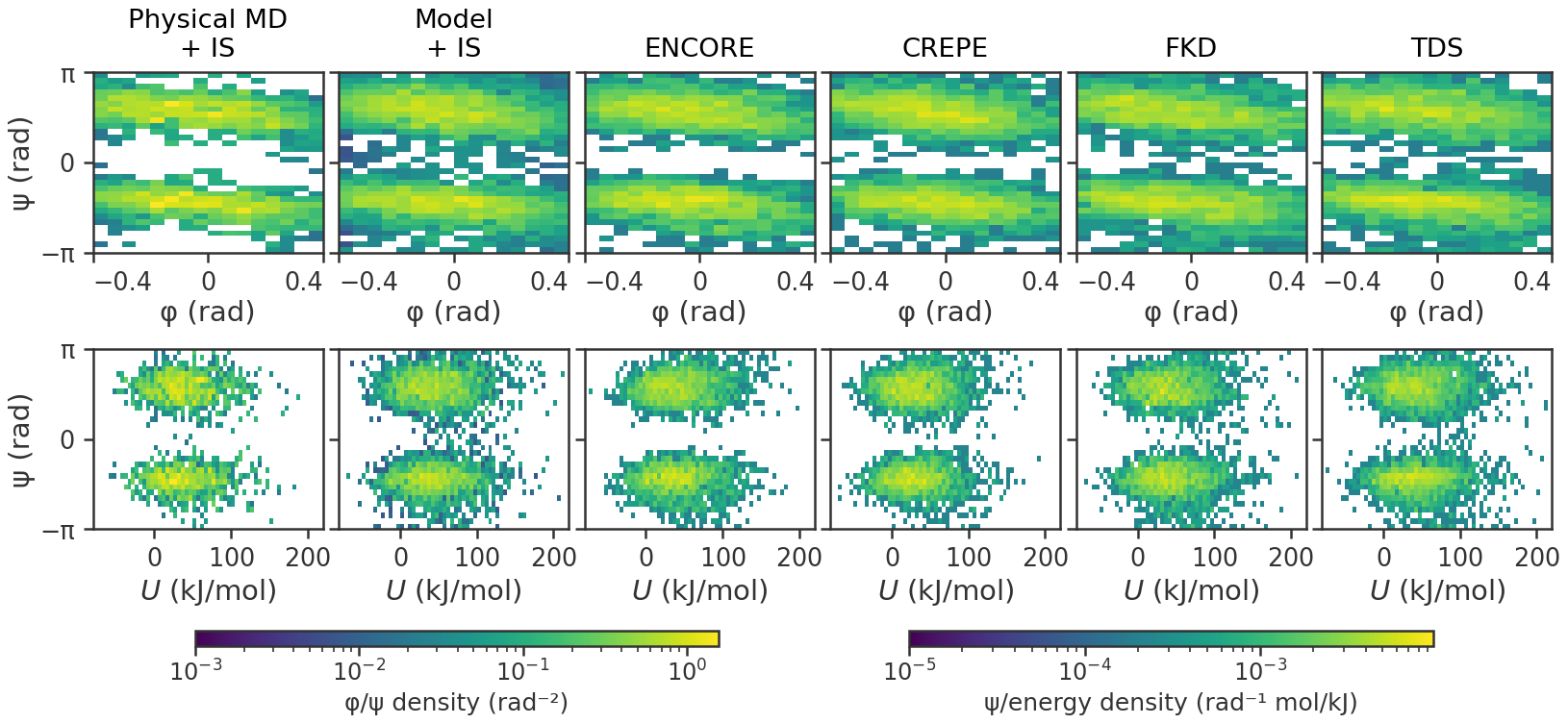}
\caption{Conditional \(\phi/\psi\) densities (top) and joint \(\psi\)/energy densities (bottom) within \(|\phi|\leq0.4\).}
\label{fig:transition-conditionals}
\end{figure}

\begin{table}[t]
\centering
\caption{Transition-region accuracy across forward-evaluation budgets (mean \(\pm\) standard deviation over three seeds).}
\label{tab:transition-controls}
\begingroup\small
\enctablesetup[4pt]
\enctablefit{\textwidth}{%
\begin{tabular}{@{}lrcc@{}}
\toprule
Method & Forward NFEs & Regional TVD\(\downarrow\) & Energy CDF\(\downarrow\) \\
\midrule
ENCORE (8 chains) & 8,681,760 & \enctablevalue{0.0241620711196}{0.0235376091082}{3}{4} & \enctablevalue{0.0311246026246}{0.0187960166378}{3}{4} \\
 & 17,361,760 & \enctablevalue{0.0117710714319}{0.00409859041417}{3}{4} & \enctablevalue{0.0278296523497}{0.0138161772122}{3}{4} \\
\addlinespace[2pt]
CREPE (8 chains) & 10,361,600 & \enctablevalue{0.0333364224671}{0.020932075178}{3}{4} & \enctablevalue{0.14838135215}{0.0123702189401}{3}{4} \\
 & 20,761,600 & \enctablevalue{0.0267119246048}{0.0141609791852}{3}{4} & \enctablevalue{0.150950071131}{0.00789535488985}{3}{4} \\
\addlinespace[2pt]
FKD (256 particles) & 8,641,280 & \enctablevalue{0.0506070360282}{0.0330000993255}{3}{4} & \enctablevalue{0.0206814873594}{0.00420672749506}{3}{4} \\
 & 17,337,600 & \enctablevalue{0.0367915868005}{0.012018602141}{3}{4} & \enctablevalue{0.0164682541881}{0.00712385197103}{3}{4} \\
\addlinespace[2pt]
TDS (256 particles) & 8,078,592 & \enctablevalue{0.0299148672238}{0.00164163049895}{3}{4} & \enctablevalue{0.0249660397095}{0.0107301429335}{3}{4} \\
 & 17,340,672 & \enctablevalue{0.0370256477017}{0.017033505027}{3}{4} & \enctablevalue{0.0203719705137}{0.0149714364816}{3}{4} \\
\addlinespace[2pt]
\bottomrule
\end{tabular}}
\endgroup
\end{table}

\subsection{Alanine Tempering}
\label{app:alanine}

\paragraph{Target.}
We temper alanine dipeptide (ALDP; 22 atoms), tetrapeptide (A4; 43 atoms), and hexapeptide (A6; 63 atoms) from \(800\,\mathrm K\) to \(300\), \(500\), and \(600\,\mathrm K\), respectively.
For molar potential energy \(U\), the endpoint factor is
\(g(x)=\exp[-(\beta_{T_t}-\beta_{T_s})U(x)]\), with \(\beta_T=1/(RT)\).
OpenMM evaluates \(U\) \citep{eastman_openmm_2023}; held-out target-temperature MD supplies the physical reference.

\paragraph{Samplers.}
We use the CREPE molecular diffusion implementation \citep{he_crepe_2026}: an EDM-preconditioned equivariant network with five message-passing layers and width 256, with 200 Karras steps and \((\sigma_{\min},\sigma_{\max},\rho)=(0.001,10,7)\).
ALDP uses the \(800\,\mathrm K\) RNE checkpoint \citep{he_rne_2026}; A4 uses an EMA checkpoint at step 190,000, selected on source-temperature validation data after FEAT-style training \citep{du_feat_2025}; A6 uses the width-256 checkpoint supplied by the CREPE authors.

\paragraph{Methods.}
ENCORE compresses the high-noise zero-potential region while retaining native transitions, then increases the reward to the full endpoint tilt.
CREPE uses 51 levels, one local ULA update per level, reference refreshes and alternating exchanges. CREPE retains 49,000 outputs after 1,000 warm-up half-sweeps.
RNE and FKC-Score/Noise use 50 populations of 1,000 particles with resampling every four transitions \citep{he_rne_2026,skreta_feynmankac_2025}.
FKC-Score multiplies the score by \(T_s/T_t\); FKC-Noise uses score coefficient \((T_s/T_t+1)/2\) and noise multiplier \((T_s/T_t)^{-1/2}\).
FKD-Diff uses 32 particles, reward checks at steps \(124,132,\ldots,164,168\), and resampling below half-population ESS \citep{singhal_general_2025}. FKD and uncontrolled sampling each return 50,000 outputs.
For the separate acceptance/diversity comparison, ENCORE and CREPE use dense \((200,196,\ldots,0)\) and coarse-to-fine \((200,80,76,\ldots,0)\) ladders, one chain per seed and three seeds.
ENCORE uses \(L_n=K_n\) and an additional reward-only endpoint replica; both use 50,000 half-sweeps, including 1,000 warm-up.

\paragraph{Evaluation.}
Boltzmann-tempering errors use 25,000 target-temperature MD frames and three evaluations drawing 5,000 configurations without replacement from both outputs and reference. Thus their standard deviations describe subsampling of one sampler run.
Energy TVD uses 100-bin histograms of the negative regularised reduced potential; distance TVD pools unordered interatomic distances. Aligned \(\mathrm W_2\) uses centred, Kabsch-aligned coordinates in nm \citep{kabsch_solution_1976}; torsion MMD\(^2\) uses a Gaussian kernel on sine--cosine backbone torsions \citep{gretton_kernel_2012}.
A4/A6 ENCORE and FKD outputs are projected onto OpenMM bond constraints.
For acceptance/diversity, RMSD \(\mathrm W_1\) compares distributions from 100,000 backbone-heavy-atom pairs per sampler and MD reference, without energy filtering. TICA JS uses \(72\times72\) target-fitted histograms and natural logarithms; Energy TVD averages three 100-bin evaluations per run. ENCORE is evaluated after constraint projection and before terminal denoising; CREPE uses its level-0 state directly. Histogram metrics retain out-of-range mass; acceptance excludes warm-up.

\subsubsection{Additional Results}
\label{app:acceptance-diversity}
\Cref{tab:encore-boltzmann-results} reports ALDP/A4 tempering errors; A6 appears in \cref{tab:encore-a6-results}.
\Cref{tab:acceptance-alanine} gives the full dense and coarse-to-fine comparison, with means and standard deviations across three seeds.

\begin{table}[t]
  \centering
  \caption{ALDP and A4 Boltzmann-tempering accuracy with 5,000 samples per evaluation (mean \(\pm\) standard deviation over three subsampling evaluations from one sampler run).}
  \label{tab:encore-boltzmann-results}
  \begingroup
  \enctablesetup
  \enctablefit{\textwidth}{%
  \begin{tabular}{@{}clcccc@{}}
    \toprule
    System & Method
      & Energy TVD\(\downarrow\)
      & Distance TVD\(\downarrow\)
      & Aligned \(\mathrm W_2\downarrow\)
      & Torsion MMD\(^2\)\(\downarrow\) \\
    \midrule

    \multirow[c]{7}{50mm}{\encbtsystem{Alanine Dipeptide}{ALDP}{800}{300}{aldp}}
         & ENCORE
         & \encbtsecond{0.214133}{0.003974}
         & \encbtbest{0.015267}{0.000800}
         & \encbtbest{0.242800}{0.002474}
         & \encbtbest{0.003544}{0.000449} \\
         & CREPE
         & \encbtbest{0.200800}{0.001987}
         & \encbtvalue{0.019072}{0.000229}
         & \encbtsecond{0.256750}{0.000966}
         & \encbtvalue{0.016761}{0.001243} \\
         & RNE
         & \encbtvalue{0.363200}{0.006951}
         & \encbtsecond{0.017278}{0.000618}
         & \encbtvalue{0.277039}{0.001448}
         & \encbtsecond{0.004614}{0.000409} \\
         & FKC-Score
         & \encbtvalue{0.369400}{0.009162}
         & \encbtvalue{0.019014}{0.000660}
         & \encbtvalue{0.283154}{0.000472}
         & \encbtvalue{0.005906}{0.000489} \\
         & FKC-Noise
         & \encbtvalue{0.883867}{0.000618}
         & \encbtvalue{0.032452}{0.000130}
         & \encbtvalue{0.266461}{0.000797}
         & \encbtvalue{0.009596}{0.000391} \\
         & FKD
         & \encbtvalue{0.564467}{0.005548}
         & \encbtvalue{0.037170}{0.000859}
         & \encbtvalue{0.367174}{0.001173}
         & \encbtvalue{0.083950}{0.001917} \\
    \cmidrule(lr){2-6}
         & Uncontrolled
         & \encbtvalue{0.993200}{0.000589}
         & \encbtvalue{0.054273}{0.000369}
         & \encbtvalue{0.300370}{0.002062}
         & \encbtvalue{0.068217}{0.003234} \\

    \midrule

    \multirow[c]{7}{50mm}{\encbtsystem{Alanine Tetrapeptide}{A4}{800}{500}{a4}}
         & ENCORE
         & \encbtsecond{0.136267}{0.005705}
         & \encbtbest{0.008373}{0.000287}
         & \encbtbest{0.734337}{0.001632}
         & \encbtbest{0.002495}{0.000329} \\
         & CREPE
         & \encbtvalue{0.502267}{0.007611}
         & \encbtvalue{0.013637}{0.000373}
         & \encbtsecond{0.748194}{0.002856}
         & \encbtsecond{0.007295}{0.000278} \\
         & RNE
         & \encbtvalue{0.245867}{0.014537}
         & \encbtsecond{0.012596}{0.000361}
         & \encbtvalue{0.905042}{0.003959}
         & \encbtvalue{0.011949}{0.000194} \\
         & FKC-Score
         & \encbtbest{0.117467}{0.014313}
         & \encbtvalue{0.024978}{0.000127}
         & \encbtvalue{0.994609}{0.000994}
         & \encbtvalue{0.024405}{0.000216} \\
         & FKC-Noise
         & \encbtvalue{0.395667}{0.007876}
         & \encbtvalue{0.020019}{0.000058}
         & \encbtvalue{0.894074}{0.004875}
         & \encbtsecond{0.006671}{0.000294} \\
         & FKD
         & \encbtvalue{0.517467}{0.008069}
         & \encbtvalue{0.024310}{0.000309}
         & \encbtvalue{0.878605}{0.001473}
         & \encbtvalue{0.031488}{0.000573} \\
    \cmidrule(lr){2-6}
         & Uncontrolled
         & \encbtvalue{0.928400}{0.004121}
         & \encbtvalue{0.023367}{0.000523}
         & \encbtvalue{0.787948}{0.006282}
         & \encbtvalue{0.019895}{0.001646} \\
    \bottomrule
  \end{tabular}}
  \endgroup
\end{table}

\begin{table}[t]
  \centering
  \caption{Alanine conformational diversity and energy agreement under dense and coarse-to-fine replica ladders (mean \(\pm\) standard deviation over three seeds).}%
  \label{tab:acceptance-alanine}
    \begingroup\footnotesize
    \enctablesetup[2pt]
    \enctablefit{\textwidth}{%
    \begin{tabular}{@{}llcccccc@{}}
    \toprule
     & & \multicolumn{3}{c}{Dense} & \multicolumn{3}{c}{Coarse-to-fine} \\
    \cmidrule(lr){3-5}\cmidrule(lr){6-8}
    System & Method & RMSD $\mathrm W_1$ (nm)$\downarrow$ & TICA JS$\downarrow$ & Energy TVD$\downarrow$ & RMSD $\mathrm W_1$ (nm)$\downarrow$ & TICA JS$\downarrow$ & Energy TVD$\downarrow$ \\
    \midrule
    ALDP & ENCORE & \enctablebest{0.00366289737554}{0.000823291617654}{4}{4} & \encgmmsecond{0.204927395665}{0.00883067916251} & \encgmmbest{0.0965555555556}{0.0162907381094} & \enctablebest{0.00359804222631}{0.00149062258672}{4}{4} & \encgmmsecond{0.221773719197}{0.00685196825787} & \encgmmbest{0.0785111111111}{0.0120303936083} \\
     & CREPE & \enctablesecond{0.00911792383566}{0.0067569378126}{4}{4} & \encgmmbest{0.068268001474}{0.00942162855314} & \encgmmsecond{0.225444444444}{0.0174408630435} & \enctablesecond{0.0253662672225}{0.00860462394581}{4}{4} & \encgmmbest{0.134978495427}{0.0136287857657} & \encgmmsecond{0.402888888889}{0.366120868526} \\
    \addlinespace[2pt]
    A4 & ENCORE & \enctablesecond{0.00697275567678}{0.0019655700027}{4}{4} & \encgmmsecond{0.143536825638}{0.00342206488744} & \encgmmbest{0.0801555555556}{0.0113412390727} & \enctablebest{0.00651341162749}{0.00227946356516}{4}{4} & \encgmmbest{0.1395274552}{0.00500594285884} & \encgmmbest{0.0877333333333}{0.0121711863751} \\
     & CREPE & \enctablebest{0.00515305326518}{0.00269121629906}{4}{4} & \encgmmbest{0.0989598232525}{0.00975509674192} & \encgmmsecond{0.515644444444}{0.0150720246127} & \enctablesecond{0.0190034166793}{0.00507635022419}{4}{4} & \encgmmsecond{0.17408743025}{0.0727559278242} & \encgmmsecond{0.485955555556}{0.0339784463273} \\
    \addlinespace[2pt]
    A6 & ENCORE & \enctablesecond{0.00568029124641}{0.00366230182975}{4}{4} & \encgmmsecond{0.0657962695244}{0.00191712116972} & \encgmmbest{0.143511111111}{0.0122073444802} & \enctablebest{0.00631895627605}{0.00351321185839}{4}{4} & \encgmmbest{0.0633949249389}{0.00524464162593} & \encgmmbest{0.139422222222}{0.0161278686797} \\
     & CREPE & \enctablebest{0.00235704035524}{0.00209766598436}{4}{4} & \encgmmbest{0.0428929557791}{0.000583292323158} & \encgmmsecond{0.329133333333}{0.00656996533054} & \enctablesecond{0.0295947620469}{0.0157452795041}{4}{4} & \encgmmsecond{0.187931652361}{0.100888668111} & \encgmmsecond{0.291355555556}{0.0503828160006} \\
    \bottomrule
    \end{tabular}}
    \endgroup
\end{table}

\subsection{Text-to-Image Reward Tilting}
\label{app:t2i}

\paragraph{Target.}
For prompt \(c\), we target
\(q_0(\mathrm dx\mid c)\propto\exp\{100\,r_{\mathrm{IR}}(x,c)\}\,p_0(\mathrm dx\mid c)\),
where \(r_{\mathrm{IR}}\) is ImageReward \citep{xu_imagereward_2023}, following \citet{he_crepe_2026}.

\paragraph{Samplers.}
Four-step SDXL-Turbo uses Euler--Ancestral transitions at \(512\times512\) \citep{sauer_adversarial_2024}; four-step DMD2-UNet uses native LCM transitions at \(1024\times1024\) \citep{yin_improved_2024}.
Both use timesteps \((999,749,499,249)\), without classifier-free guidance or watermarking.

\paragraph{Methods.}
ENCORE uses five path levels with reward strengths \((0,25,50,75,100)\) and independent reference refreshes. With 16 or 8 chains it discards 64 or 128 iterations, then collects every 32 iterations to obtain 32 outputs.
CREPE uses the same chain counts and warm-up, collecting every 50 iterations; its SDXL-Turbo adaptation excludes the deterministic final transition from exchange. We do not evaluate it on DMD2, whose native LCM sampler lacks the paired forward kernel required by its correction.
FKD-Diff uses 2,048 particles with multinomial resampling after every transition and returns 32 particles \citep{singhal_general_2025}.
LiDAR uses a 2,048-sample bank for its score correction and generates 32 trajectories \citep{kim_lookahead_2026}. Both use reward strength 100.
Best-of-64 returns the highest-reward output from each of 32 independent groups of 64 draws; uncontrolled sampling returns 32 independent base outputs.

\paragraph{Evaluation.}
We use 32 prompts from the ImageReward test split, with 32 outputs per prompt, decoded in a common evaluation environment.
Metrics are ImageReward v1.0, HPSv2.1 (ViT-H/14), CLIP-T and CLIP diversity (ViT-L/14), DINO-Vendi (normalized DINOv2-Large embeddings), LPIPS (AlexNet), and the fraction of distinct decoded RGB images \citep{xu_imagereward_2023,wu_human_2023,radford_learning_2021,friedman_vendi_2023,oquab_dinov2_2024,zhang_unreasonable_2018}.
Reward, preference and alignment are averaged within each prompt; DINO-Vendi and uniqueness use the complete output set, and pairwise diversity averages all 496 unordered pairs. Tables report means and standard deviations across prompts.

\subsubsection{Additional Results}
\Cref{tab:encore-image-turbo,tab:encore-image-dmd2} report reward, preference, alignment and diversity for both samplers.

\begin{table}[t]
  \centering
  \caption{Reward, preference, prompt alignment, and diversity on
  SDXL-Turbo-4 (mean \(\pm\) standard deviation over 32 prompts).}
  \label{tab:encore-image-turbo}
  \begingroup
  \footnotesize
  \enctablesetup
  \enctablefit{\textwidth}{%
    \begin{tabular}{@{}l*{7}{c}@{}}
      \toprule
      & \multicolumn{1}{c}{\textit{Reward}}
      & \multicolumn{2}{c}{\textit{Preference / Alignment}}
      & \multicolumn{4}{c}{\textit{Diversity}} \\
      \cmidrule(lr){2-2}\cmidrule(lr){3-4}\cmidrule(lr){5-8}
      Method
        & IR \(\uparrow\)
        & HPS \(\uparrow\)
        & CLIP-T \(\uparrow\)
        & DINO-Vendi \(\uparrow\)
        & CLIP diversity \(\uparrow\)
        & LPIPS \(\uparrow\)
        & Unique (\%) \(\uparrow\) \\
      \midrule
      ENCORE (16 chains)
        & \encimgresult{1.7114234553}{0.2784384968}
        & \encimgsecond{0.3235161304}{0.0248921533}
        & \encimgsecond{0.2901362841}{0.0296879913}
        & \encimgresult{3.0201178041}{0.9521507647}
        & \encimgresult{0.1703995376}{0.0574651019}
        & \encimgresult{0.5770453503}{0.0534803142}
        & \encimgpercentsecond{84.08203125}{11.1226940672} \\
      ENCORE (8 chains)
        & \encimgbest{1.7452704262}{0.2577992609}
        & \encimgsecond{0.3242468834}{0.0245946101}
        & \encimgsecond{0.2898515749}{0.0289986420}
        & \encimgresult{2.7775099105}{0.7928292876}
        & \encimgresult{0.1701421910}{0.0665273727}
        & \encimgresult{0.5593948583}{0.0530738986}
        & \encimgpercent{64.94140625}{22.5868392511} \\
      CREPE (16 chains)
        & \encimgresult{1.7035764496}{0.2831972717}
        & \encimgresult{0.3200669289}{0.0254819250}
        & \encimgsecond{0.2903242631}{0.0291304360}
        & \encimgsecond{3.0946569343}{0.9597243263}
        & \encimgsecond{0.1825966169}{0.0647603038}
        & \encimgsecond{0.5834499075}{0.0525486416}
        & \encimgpercent{78.80859375}{10.4923638090} \\
      CREPE (8 chains)
        & \encimgsecond{1.7376574219}{0.2409553467}
        & \encimgresult{0.3197839260}{0.0257036330}
        & \encimgsecond{0.2897389025}{0.0292248505}
        & \encimgresult{2.8194768861}{0.8245221592}
        & \encimgresult{0.1758086520}{0.0635232680}
        & \encimgresult{0.5634750504}{0.0594485364}
        & \encimgpercent{62.01171875}{19.1178525785} \\
      FKD
        & \encimgresult{1.7034516587}{0.3254523194}
        & \encimgbest{0.3263484240}{0.0252718125}
        & \encimgbest{0.2911328385}{0.0315804489}
        & \encimgresult{1.7203517591}{0.6493074741}
        & \encimgresult{0.0813054877}{0.0633711139}
        & \encimgresult{0.2990188346}{0.2131767360}
        & \encimgpercent{44.921875}{30.0867152827} \\
      LiDAR
        & \encimgresult{1.1251515571}{0.5411863976}
        & \encimgresult{0.3076999187}{0.0244937128}
        & \encimgresult{0.2833418079}{0.0316117461}
        & \encimgbest{4.3902119329}{1.7832814607}
        & \encimgbest{0.2671480806}{0.0952369121}
        & \encimgbest{0.6137040161}{0.0472963350}
        & \encimgpercentbest{100.0}{0.0} \\
      \cmidrule(lr){1-8}
      Best-of-64
        & \encimgresult{1.6512273313}{0.3464447454}
        & \encimgresult{0.3226350546}{0.0242616262}
        & \encimgresult{0.2888637738}{0.0301455669}
        & \encimgresult{3.1844174730}{1.0612121693}
        & \encimgresult{0.1723026500}{0.0540068932}
        & \encimgresult{0.5866820938}{0.0504167384}
        & \encimgpercent{100.0}{0.0} \\
      Uncontrolled
        & \encimgresult{1.0483377455}{0.5491698221}
        & \encimgresult{0.3110949993}{0.0254519878}
        & \encimgresult{0.2845110282}{0.0299891905}
        & \encimgresult{3.9052523436}{1.6762374503}
        & \encimgresult{0.2178120973}{0.0853625546}
        & \encimgresult{0.6086764456}{0.0472007647}
        & \encimgpercent{100.0}{0.0} \\
      \bottomrule
    \end{tabular}}
  \endgroup
\end{table}

\begin{table}[t]
  \centering
  \caption{Reward, preference, prompt alignment, and diversity on
  DMD2-UNet-4 (mean \(\pm\) standard deviation over 32 prompts).}
  \label{tab:encore-image-dmd2}
  \begingroup
  \footnotesize
  \enctablesetup
  \enctablefit{\textwidth}{%
    \begin{tabular}{@{}l*{7}{c}@{}}
      \toprule
      & \multicolumn{1}{c}{\textit{Reward}}
      & \multicolumn{2}{c}{\textit{Preference / Alignment}}
      & \multicolumn{4}{c}{\textit{Diversity}} \\
      \cmidrule(lr){2-2}\cmidrule(lr){3-4}\cmidrule(lr){5-8}
      Method
        & IR \(\uparrow\)
        & HPS \(\uparrow\)
        & CLIP-T \(\uparrow\)
        & DINO-Vendi \(\uparrow\)
        & CLIP diversity \(\uparrow\)
        & LPIPS \(\uparrow\)
        & Unique (\%) \(\uparrow\) \\
      \midrule
      ENCORE (16 chains)
        & \encimgbest{1.6613961245}{0.3085505078}
        & \encimgbest{0.3241206408}{0.0253754925}
        & \encimgbest{0.2988004773}{0.0334627926}
        & \encimgbest{3.2759559096}{1.1355599756}
        & \encimgbest{0.1800746481}{0.0810995996}
        & \encimgbest{0.5952657047}{0.0671469488}
        & \encimgpercentbest{78.3203125}{10.7358280276} \\
      FKD
        & \encimgsecond{1.6585585496}{0.3476475703}
        & \encimgbest{0.3242847919}{0.0237642213}
        & \encimgsecond{0.2976931207}{0.0374670541}
        & \encimgsecond{1.4729391365}{0.6032042262}
        & \encimgsecond{0.0536976453}{0.0556878919}
        & \encimgsecond{0.1936422564}{0.1744311231}
        & \encimgpercentsecond{38.4765625}{31.5852731076} \\
      \cmidrule(lr){1-8}
      Best-of-64
        & \encimgresult{1.6089698287}{0.3661400537}
        & \encimgresult{0.3233835697}{0.0249755840}
        & \encimgresult{0.2981934625}{0.0335600334}
        & \encimgresult{3.6334785339}{1.4605234794}
        & \encimgresult{0.1799046218}{0.0878779970}
        & \encimgresult{0.6057399734}{0.0622595958}
        & \encimgpercent{100.0}{0.0} \\
      Uncontrolled
        & \encimgresult{0.9617134248}{0.6205643895}
        & \encimgresult{0.3137189150}{0.0284390264}
        & \encimgresult{0.2947623738}{0.0361685995}
        & \encimgresult{3.9762307894}{1.4719025801}
        & \encimgresult{0.1968198875}{0.0703228448}
        & \encimgresult{0.6187615420}{0.0682831274}
        & \encimgpercent{100.0}{0.0} \\
      \bottomrule
    \end{tabular}}
  \endgroup
\end{table}